\ifdefined\XeTeXversion\else\pdfoutput=1\fi

\documentclass[11pt,twoside]{article}

\usepackage[utf8]{inputenc}
\usepackage{setspace}

\usepackage{amsmath,amsthm,amssymb,amsfonts}
\usepackage{mathtools}
\usepackage{mathrsfs}
\usepackage{dsfont}
\usepackage{esvect}

\usepackage[top=1in,bottom=1.2in,left=1in,right=1in]{geometry}

\usepackage{graphicx}
\usepackage{float}
\usepackage{textcomp}

\usepackage{longtable}
\usepackage{booktabs}
\usepackage{array}
\usepackage{multirow}
\usepackage{threeparttable}

\usepackage{algorithm}
\usepackage{algorithmic}
\usepackage{enumitem}
\usepackage{subcaption}
\usepackage{caption}
\usepackage{rotating}
\usepackage[hyphens]{url}
\usepackage{footnote}
\usepackage{optidef}
\usepackage{thmtools}
\usepackage{etoolbox}
\usepackage{lastpage}
\usepackage{natbib}
\usepackage{hyperref}
\usepackage{cleveref}

\def\argmax{\mathop{\rm argmax}}

\def\bfX{\boldsymbol{X}}

\newcommand{\RNum}[1]{\uppercase\expandafter{\romannumeral #1\relax}}

\newcommand{\normmm}{{\vert\kern-0.25ex\vert\kern-0.25ex\vert}}
\newcommand{\bignormmm}{{\big\vert\kern-0.25ex\big\vert\kern-0.25ex\big\vert}}
\newcommand{\Bignormmm}{{\Big\vert\kern-0.25ex\Big\vert\kern-0.25ex\Big\vert}}

\newtheorem{theorem}{Theorem}
\newtheorem{lemma}{Lemma}
\newtheorem{proposition}{Proposition}
\newtheorem{remark}{Remark}
\newtheorem{corollary}{Corollary}
\newtheorem{definition}{Definition}

\newtheorem{assumption}{Assumption}

\newcommand{\BlackBox}{\rule{1.5ex}{1.5ex}}
\ifdefined\proof
    \renewenvironment{proof}{\par\noindent{\bf Proof\ }}{\hfill\BlackBox\\[2mm]}
\else
    \newenvironment{proof}{\par\noindent{\bf Proof\ }}{\hfill\BlackBox\\[2mm]}
\fi

\makeatletter
\long\def\@makecaption#1#2{
\vskip 0.8ex
\setbox\@tempboxa\hbox{\small {\bf #1:} #2}
\parindent 1.5em
\dimen0=\hsize
\advance\dimen0 by -3em
\ifdim \wd\@tempboxa >\dimen0
\hbox to \hsize{
\parindent 0em
\hfil
\parbox{\dimen0}{\def\baselinestretch{0.96}\small
    {\bf #1.} {#2}
  }
\hfil}
\else \hbox to \hsize{\hfil \box\@tempboxa \hfil}
\fi
}
\makeatother

\begin{document}
\onehalfspacing

\begin{center}
  {\bf \LARGE High-Dimensional Differentially Private \\
   Quantile Regression: \\
  \vspace{0.25em}
   Distributed Estimation and Statistical Inference} \\
  \vspace{0.5em}
    {\large
        Ziliang Shen$^{\dagger*}$,
        Caixing Wang$^{\diamond*}$,
        Shaoli Wang$^{\dagger*}$, and
        Yibo Yan$^{\flat*}$\par
        \medskip
        \normalsize
        $^{\dagger}$School of Statistics and Data Science,
        Shanghai University of Finance and Economics\par
        $^\diamond$School of Statistics and Data Science,
        Southeast University\par
        $^\flat$School of Statistics and Data Science,
        Shanghai University of International Business and Economics\par
    }
  \vspace{.6em}

  \today
\end{center}

\begin{center}
{\bf Abstract} \\ \vspace{.6em}
\begin{minipage}{0.9\linewidth}
{\small ~~~~ With the development of big data and machine learning, privacy concerns have become increasingly critical, especially when handling heterogeneous datasets containing sensitive personal information. Differential privacy provides a rigorous framework for safeguarding individual privacy while enabling meaningful statistical analysis. In this paper, we propose a differentially private quantile regression method for high-dimensional data in a distributed setting. Quantile regression is a powerful and robust tool for modeling the relationships between the covariates and responses in the presence of outliers or heavy-tailed distributions. To address the computational challenges due to the non-smoothness of the quantile loss function, we introduce a Newton-type transformation that reformulates the quantile regression task into an ordinary least squares problem. Building on this, we develop a differentially private estimation algorithm with iterative updates, ensuring both near-optimal statistical accuracy and formal privacy guarantees. For inference, we further propose a differentially private debiased estimator, which enables valid confidence interval construction and hypothesis testing. Additionally, we propose a communication-efficient and differentially private bootstrap for simultaneous hypothesis testing in high-dimensional quantile regression, suitable for distributed settings with both small and abundant local data. Extensive simulations and a real data application demonstrate the robustness and effectiveness of our methods in practical scenarios.}
\\ \textbf{Keywords}: Differential privacy, quantile regression, Newton-type transformation, debiased method, Bahadur representation, multiplier bootstrap.
\end{minipage}
\end{center}

\vspace*{\fill}
\noindent
\begin{minipage}{\textwidth}
\footnotesize *: All authors contributed equally to this paper and their names are listed in alphabetical order by last name.
\end{minipage}

\newpage

\section{Introduction}
In the era of big data, privacy protection has become a central concern in statistical learning and machine learning, especially when sensitive personal or confidential information is involved. Major regulations such as the European Union's General Data Protection Regulation (GDPR), the California Consumer Privacy Act (CCPA), and China's Personal Information Protection Law (PIPL) reflect the growing global emphasis on data privacy. Differential privacy (DP), introduced by \cite{dwork2006calibrating}, provides a rigorous mathematical framework for quantifying privacy loss. It ensures that the inclusion or exclusion of a single individual's data has only a limited influence on the output of a statistical procedure, thereby protecting individual privacy. This property is particularly important in data-driven processes, where large datasets are used for statistical modeling, machine learning, and predictive analytics.  A common way to achieve DP is to inject carefully calibrated random noise, such as Laplace, Gaussian, or binomial noise, according to the sensitivity of the target procedure \cite{dwork2006our,dwork2014algorithmic}. Information-theoretic properties of DP mechanisms have also been studied, including bounds on min-entropy leakage \cite{barthe2011information}. Due to its strong theoretical guarantees and broad applicability, DP has become an important tool in privacy-sensitive domains such as healthcare, finance, and the social sciences.

In addition to privacy concerns, the distributed nature of modern datasets brings further challenges to statistical learning. Data are often collected and stored across multiple nodes or devices, making centralized analysis computationally expensive or infeasible. Distributed learning methods are therefore needed to handle large-scale data while accounting for communication constraints, heterogeneous data distributions, and privacy requirements. At the same time, real-world data frequently exhibit skewness, heavy-tailed noise, outliers, and heteroscedasticity. Quantile regression (QR) provides a powerful tool in statistical modeling for enhancing robustness. Unlike traditional mean regression, which focuses on estimating the conditional mean, quantile regression provides a more comprehensive view of the relationship between variables by estimating different quantiles, making it particularly useful in the presence of outliers or heavy-tailed error distributions \cite{koenker1978regression,koenker2005quantile}. However, the non-smooth nature of the quantile loss function creates substantial computational and theoretical difficulties, especially in high-dimensional settings where the number of covariates can be very large. These difficulties become even more pronounced in distributed and privacy-sensitive settings, where communication efficiency, statistical accuracy, sparsity, robustness, and formal privacy guarantees must be addressed simultaneously.

Existing studies have made substantial progress in differentially private statistical learning and distributed quantile regression, but these two lines of research remain largely separate. Differentially private methods have been developed, including mean estimation \cite{cai2021cost, zhang2024differentially}, linear regression \cite{cai2021cost, wang2020sparse}, and distributed learning \cite{agarwal2018cpsgd,erlingsson2020encode,bittau2017prochlo,girgis2021shuffled}. Another line of research has focused on developing efficient distributed algorithms for quantile regression, particularly in high-dimensional settings \cite{chen2020distributed, wang2024distributed}. However, these methods either do not incorporate differential privacy or do not address the challenges posed by heavy-tailed noise and heteroscedastic data. The investigation of robust estimation and inference techniques in a distributed and differentially private manner is particularly pressing.

In this paper, we first develop a privacy-preserving framework for distributed estimation and inference in high-dimensional sparse quantile regression, addressing three foundational questions:
\begin{itemize}
    \item \textit{How to perform DP high-dimensional quantile regression with heavy-tailed noise and heteroscedastic data in the distributed setting?}
    \item \textit{How does DP affect the statistical accuracy of high-dimensional distributed quantile regression in both estimation and inference?}
    \item \textit{How should efficient DP estimation and inference be achieved with communication constraints under distributed learning?}
\end{itemize}
\subsection{  {Literature Review}}
In recent years, privacy-preserving techniques have been increasingly integrated into statistical learning and machine learning, driven by the rapid growth of data-driven applications and rising concerns about the protection of sensitive information \cite{dwork2006calibrating,dwork2014algorithmic}. Differential privacy, first introduced by \cite{dwork2006calibrating}, provides a fundamental mathematical framework for quantifying privacy loss. Since then, several related privacy mechanisms have been developed. For example, local differential privacy (LDP) adds noise directly on the client side and therefore does not require a trusted curator \cite{duchi2013local}, while Gaussian differential privacy (GDP) characterizes privacy loss through a single parameter based on a unit-variance Gaussian trade-off function \cite{dong2022gaussian}.  These privacy frameworks have been incorporated into a wide range of statistical learning problems. The specific form and cost of privacy protection depend heavily on the underlying learning task, including whether the model is low- or high-dimensional and whether the data are stored centrally or distributed across multiple machines. We therefore review four closely related lines of work.  

\noindent \textbf{Low-dimensional estimation with differential privacy}.
Balancing privacy protection and statistical accuracy is a central challenge in differentially private statistical learning. In low-dimensional settings, this privacy-utility trade-off has been extensively studied since the early works of \cite{dinur2003revealing,dwork2007price,dwork2008new}. Recent advances have developed differentially private procedures for a variety of classical statistical tasks, including mean estimation \cite{yu2024gaussian,liu2021robust,agarwal2025private}, covariance estimation \cite{amin2019differentially}, and linear regression \cite{sheffet2017differentially,alabi2023differentially}. Parallel progress has also been made under local differential privacy, including minimax-optimal estimation and lower bounds \cite{duchi2018minimax,asi2022optimal}. Several recent studies further connect robustness and privacy, showing that robust estimators can often be privatized through sensitivity-calibrated noise or contamination-based interpretations \cite{avella2023differentially,li2023robustness,kamath2024bias}. These works provide important foundations for understanding the statistical cost of privacy in classical low-dimensional models. 

\noindent \textbf{High-dimensional estimation and inference with differential privacy}.
With the growing need for privacy-preserving analysis in high-dimensional sparse models, differentially private sparse estimation has received considerable attention. A key algorithmic tool is the ``peeling'' or Noisy Hard Thresholding procedure introduced by \cite{Dwork_Su_Zhang_2021}, which has been widely used in private high-dimensional estimation \cite{cai2021cost,xia2023adaptive,cai2023private}. These studies showed an inherent cost of privacy: the estimation error typically increases with the model dimension and sparsity level \cite{cai2021cost}. This privacy--utility trade-off has been investigated in several high-dimensional problems, including top-$k$ feature selection \cite{steinke2017tight}, sparse mean estimation \cite{cai2021cost}, covariance estimation \cite{dong2022differentially}, sparse linear regression \cite{talwar2015nearly,cai2021cost,wang2020sparse}, and least absolute deviation regression \cite{liu2024efficient}. More recently, \cite{dwork2024differentially} explored private learning in high-dimensional regimes where the dimension grows proportionally to the sample size.

Compared with estimation, differentially private inference in high dimensions remains much less developed. Existing advances mainly focus on linear or generalized linear models, such as differentially private inference and false discovery rate control for high-dimensional linear regression \cite{cai2023private}, and DP nonparametric testing for generalized linear models and the Bradley--Terry--Luce model \cite{cai2023score}.  {However, valid inference for non-smooth robust models, especially high-dimensional quantile regression under distributed data and differential privacy constraints, remains largely unexplored.}

\noindent \textbf{Distributed statistical learning with differential privacy}.
Differential privacy has also been incorporated into distributed statistical learning, where privacy protection must be balanced with communication efficiency and statistical accuracy. In low-dimensional settings, existing works have developed communication-efficient DP stochastic gradient descent algorithms \cite{agarwal2018cpsgd}, one-shot DP logistic regression methods \cite{bao2023communication}, and private protocols without a trusted central machine \cite{gao2024private}. For high-dimensional models,  \cite{wang2017distributed} investigated decentralized data ownership, while \cite{zhang2024differentially} introduced a multi-round DP gradient method for high-dimensional linear regression under heterogeneous privacy budgets. Recent progress has also been made in distributed nonparametric and federated learning, including minimax-optimal estimation under machine-specific privacy constraints \cite{cai2024optimal,cai2024federated}, DP goodness-of-fit testing \cite{vuursteen2024optimal}, adaptive federated nonparametric classification \cite{auddy2024minimax}, and federated DP learning under site heterogeneity \cite{li2024federated}.  {Despite these advances, existing distributed DP methods mainly focus on smooth or mean-type models, and do not directly address the challenges of high-dimensional quantile regression, where non-smooth loss functions, sparsity, heavy-tailed or heteroscedastic errors, and valid statistical inference must be handled simultaneously.}

 {\noindent \textbf{Distributed linear quantile regression.}
Distributed quantile regression has been widely studied in recent years. Traditional divide-and-conquer methods have been developed for low-dimensional quantile regression \cite{zhao2014general,xu2020block,chen2020quantile}. However, in high-dimensional settings, where sparsity plays a crucial role, such methods may lose sparsity due to debiasing and averaging, leading to unsatisfactory support recovery \cite{bradic2017uniform,chen2020distributed}. They also typically require restrictions on the number of local machines to maintain the global convergence rate. To address these limitations, \cite{chen2020distributed} proposed a Newton-type distributed estimator by transforming the non-smooth check loss into a weighted least squares problem via kernel smoothing, and established guarantees for estimation and support recovery. However, their theory requires the error term to be independent of the covariates, which can be restrictive under heteroscedastic models. \cite{tan2022communication} further developed a communication-efficient convolution-smoothed quantile regression approach with a twice differentiable loss function, covering both low- and high-dimensional regimes. More recently, \cite{wang2024distributed} extended the framework of \cite{chen2020distributed} to allow noise terms that depend on covariates. Distributed quantile inference has also been studied in \cite{volgushev2019distributed,he2024debiased}. These works provide important non-private foundations for distributed quantile regression under high-dimensional sparsity.}

 {Taken together, the above literature shows that differential privacy, distributed learning, and high-dimensional quantile regression have each been substantially developed, but their intersection remains largely unexplored. Existing DP methods rarely address non-smooth quantile losses with high-dimensional sparsity and robust inference, while existing distributed quantile regression methods are mostly non-private. This paper fills this gap by developing a differentially private framework for distributed high-dimensional quantile regression, covering sparse estimation, debiasing, confidence interval construction, and simultaneous inference. Our theory explicitly quantifies how differential privacy affects both estimation accuracy and inferential validity.}

\subsection{Our Contributions}
In this paper, we develop a differentially private framework for distributed high-dimensional quantile regression under the trusted-machine setting of \cite{zhang2024differentially}. The proposed framework is designed for robust estimation and inference in the presence of heavy-tailed noise, heteroscedastic errors, high-dimensional sparsity, and distributed data storage. Our methodology combines a Newton-type transformation for the non-smooth quantile loss with privacy-preserving sparse updates, thereby enabling both computationally feasible optimization and formal ($\epsilon,\delta$)-differential privacy guarantees. We further develop differentially private debiasing and bootstrap procedures for statistical inference. Theoretically, we establish finite-sample estimation error bounds, non-asymptotic precision matrix estimation guarantees, Bahadur representations, asymptotic normality, valid confidence intervals, and simultaneous inference guarantees, explicitly quantifying the additional statistical cost induced by privacy protection. The main contributions are summarized as follows.
\begin{enumerate}
    \item \textbf{Differentially private sparse estimation.}
We propose a distributed estimation procedure for high-dimensional quantile regression under differential privacy. To overcome the non-smoothness of the quantile loss, we employ a Newton-type transformation that converts each local quantile regression update into a least-squares-type problem. We then combine distributed gradient aggregation with Noisy Hard Thresholding to enforce sparsity and privacy simultaneously at each iterative update. For the proposed estimator in Algorithm~\ref{alg:DP:SparseEstimation}, we establish finite-sample error bounds in Theorems~\ref{thm1} and~\ref{thm2}. The resulting rate explicitly decomposes into the oracle statistical error, the privacy-induced error, and the iterative approximation error, thereby characterizing how differential privacy affects high-dimensional distributed quantile regression.
   \item \textbf{Differentially private debiasing and coordinate-wise inference.} Since sparse private estimators are generally biased due to thresholding and regularization, we develop a differentially private debiasing procedure for valid statistical inference. We first construct a differentially private sparse pseudo precision matrix estimator using a noisy covariance matrix and a CLIME-type optimization problem in Algorithm~\ref{alg:DP:PrecisionEstimation}. Theorem~\ref{thm3} provides non-asymptotic error bounds for this estimator under several matrix norms. Based on this precision matrix estimator, we propose a differentially private debiased estimator in Algorithm~\ref{alg:DPconfidentinterval}. We derive its Bahadur representation in Theorem~\ref{thm:bahadur}, establish asymptotic normality in Corollary~\ref{cor:normality}, and construct valid differentially private confidence intervals in Theorem~\ref{thm:interval}.
    \item \textbf{Differentially private simultaneous inference.} Beyond coordinate-wise inference, we study simultaneous testing for high-dimensional quantile regression under distributed privacy constraints. Leveraging the Bahadur representation, we reduce the simultaneous inference problem to the approximation of a high-dimensional Gaussian maximum. We then propose a communication-efficient differentially private multiplier bootstrap procedure in Algorithm~\ref{alg:DPbootstrap}. Theorem~\ref{thm:bootstrap} establishes its theoretical validity, allowing simultaneous confidence interval construction and multiple testing while preserving differential privacy. \item  {\textbf{Numerical evaluation.} We conduct extensive simulation studies and real-world applications to examine the finite-sample performance of the proposed methods under different privacy budgets, error distributions, sparsity levels, and distributed data configurations. The results illustrate the robustness of the proposed procedure under heavy-tailed and heterogeneous errors, and demonstrate the expected privacy--utility trade-off in distributed high-dimensional quantile regression.}
\end{enumerate}

\subsection{Paper Organization and Notation}
The remainder of the paper is organized as follows. 
Section~\ref{sec:Preliminaries} reviews the fundamentals of differential privacy, including key definitions and properties, and introduces the Noisy Hard Thresholding algorithm for high-dimensional sparse estimation (Section~\ref{subsec:DP}), the basics of linear quantile regression (Section~\ref{subsec:qr}), and the Newton-type transformation that reformulates quantile regression as iterative least squares (Section~\ref{subsec:newtontype}).
Section~\ref{sec:DP-estimation} proposes our differentially private estimation algorithm under distributed setting and establishes its statistical guarantees. Section~\ref{sec:DP-inference} develops a debiased inference procedure under differential privacy, together with its theoretical analysis. Section~\ref{sec:DP-bootstrap} presents a bootstrap-based framework for simultaneous inference and multiple testing in the distributed setting, along with supporting theory.  {Sections~\ref{sec:simulation} and ~\ref{realdata} report comprehensive simulation studies and a real data application under various scenarios to evaluate the empirical performance of the proposed methods.} We conclude in Section~\ref{sec:Conclusion} with a summary of contributions, a discussion of open challenges, and directions for future research. Additional technical details, algorithmic discussions, and experimental results are reported separately.

\medskip

\noindent
{\sc Notation}:
We use $\mathbb{R}^p$ to denote the $p$-dimensional Euclidean space. For every $\boldsymbol{v} = \left(v_1, v_2, \ldots, v_p\right)^{\top} \in \mathbb{R}^p$, define $\|\boldsymbol{v}\|_q = (\sum_{i=1}^p v_i^q)^{1/q}$ for $1 \leq q < \infty$, $\|\boldsymbol{v}\|_\infty = \max_{1 \leq i \leq p} |v_i|$ and $\|\boldsymbol{v}\|_{0} = \sum_{i=1}^{p} \mathbb{I}(v_i \neq 0)$ with $\mathbb{I}(\cdot)$ being the indicator function. Denote $\mathbb{S}^p = \{ \boldsymbol{v} \in \mathbb{R}^{p+1} : \|\boldsymbol{v}\|_2 = 1 \}$ the unit sphere in $\mathbb{R}^{p+1}$ and $\mathbb{B}_1(r)=\{\boldsymbol{v} \in\mathbb{R}^{p+1}:\|\boldsymbol{v}\|_1 \leq r \}$ the $\ell_1$-ball in $\mathbb{R}^{p+1}$ with radius $r$. For a set of indices $\mathcal{S}$, the subvector $\boldsymbol{v}_{\mathcal{S}}$ consists of the components of $\boldsymbol{v}$ indexed by $\mathcal{S}$. For any $p \times q$ matrix $\mathbf{A}=\left(a_{i j}\right) \in \mathbb{R}^{p \times q}$, we define the elementwise $\ell_{\infty}$-norm $\|\mathbf{A}\|_{\infty}=\max _{1 \leq i \leq p, 1 \leq j \leq q}\left|a_{i j}\right|$, the elementwise $\ell_1$-norm $\|\mathbf{A}\|_1=$ $\sum_{i=1}^p \sum_{j=1}^q\left|a_{i j}\right|$, the matrix $L_1$-norm $\|\mathbf{A}\|_{L_1}=\max _{1 \leq i \leq p} \sum_{j=1}^q\left|a_{i j}\right|$, matrix $L_2$-norm $\|\mathbf{A}\|_{L_2}=\max _{1 \leq i \leq p} (\sum_{j=1}^q a^2_{i j})^{1/2}$, Frobenius-norm $\|\mathbf{A}\|_{\mathsf{F}} = (\sum_{i=1}^p \sum_{j=1}^q a^2_{i j})^{1/2}$, and the matrix operator norm $\|\mathbf{A}\|_{op} = \sup_{\|\boldsymbol{v}\|_2 = 1}\|\mathbf{A}\boldsymbol{v}\|$. For two subsets $\mathcal{S}_1 \in \{1,\ldots,p\} $ and $ \mathcal{S}_2 \in \{1,\ldots,q\}$, define the submatrix $\mathbf{A}_{\mathcal{S}_1 \times \mathcal{S}_2}=(a_{ij})_{i \in \mathcal{S}_1, j \in \mathcal{S}_2}$. If $\mathbf{A}$ is a square matrix, then we use $\Lambda_{\max }(\mathbf{A})$ and $\Lambda_{\min }(\mathbf{A})$ to denote the largest and smallest eigenvalues of $\mathbf{A}$, respectively.
Throughout this work, we use $\mathbf{I}$ to represent the identity matrix and $\mathbf{e}_j$ to denote the unit vector with $j$-th element being $1$. Additionally, we define the truncation function $\Pi_r: \mathbb{R}^p \rightarrow \mathbb{R}^p$, $\Pi_r(\boldsymbol{v}) = (\operatorname{sign}(v_i) \cdot \min(|v_i|,r))_{i=1}^{p}$, which projects a vector onto the $\ell_{\infty}$-ball of radius $r>0$ centered at the origin. For two sequences of non-negative numbers $\{ x_n \}_{n\geq 1}$ and $\{ y_n \}_{n\geq 1}$, $x_n \lesssim y_n$ means that there exists some constant $C>0$ independent of $n$ such that $x_n \leq Cy_n$; $x_n \gtrsim y_n$ is equivalent to $y_n \lesssim x_n$; $x_n \asymp y_n$ is equivalent to $x_n \lesssim y_n$ and $y_n \lesssim x_n$. We use $C, c, c_0, c_1,...$ to denote universal constants whose value may change from line to line.



\section{Preliminaries}\label{sec:Preliminaries}
In this section, we first review some fundamental concepts of differential privacy, then provide a brief introduction to the quantile regression model, and finally present a Newton-type transformation approach to transform high-dimensional quantile regression problems into the least squares alternatives.

\subsection{Differential Privacy}\label{subsec:DP}
Differential privacy is a widely adopted and rigorous framework for privacy protection in data analysis. We begin with the formal definition of differential privacy. For an algorithm with real-valued output, we have the following definition of differential privacy. For a comprehensive and detailed explanation, we refer the readers to \cite{dwork2006calibrating}.
\begin{definition}[Differential Privacy \cite{dwork2006calibrating}] \label{def:DP}
A randomized algorithm $\mathcal{U}: \mathcal{D} \rightarrow \mathbb{R}$ is said to be $(\epsilon, \delta)$-differentially private (abbreviated as $(\epsilon, \delta)$-DP) for some $\epsilon, \delta > 0$, if for every pair of neighboring datasets $D, D^{\prime} \in \mathcal{D}$ differing in a single individual's data, and for every measurable set $\mathcal{A} \subseteq \mathbb{R}$,
\[
\mathbb{P}[\mathcal{U}(D) \in \mathcal{A}] \leq e^\epsilon \mathbb{P}\left[\mathcal{U}(D^{\prime}) \in \mathcal{A}\right] + \delta,
\]
where the probability $\mathbb{P}$ is taken over the randomness of $\mathcal{U}$.
\end{definition}

Definition \ref{def:DP} formalizes the principle that the output of an algorithm should not be significantly
affected when a single individual's data are modified, thereby protecting individual privacy. The parameter $\epsilon$ quantifies the privacy loss, with smaller values indicating stronger privacy guarantees. The parameter $\delta$ allows for a small probability of failure, providing a trade-off between privacy and utility. When $\delta=0$, the algorithm is said to satisfy pure $(\epsilon,0)$-differential privacy,   {ensuring} that the output of the algorithm   {is} indistinguishable for any two neighboring datasets.  
As demonstrated in \cite{dwork2006calibrating}, privacy-preserving algorithms can be designed by adding carefully calibrated noise to their outputs, where the noise distribution is determined by the algorithm's sensitivity.

\begin{definition}[Algorithm Sensitivity \cite{dwork2014algorithmic}] 
\label{def:sen}
For a deterministic vector-valued algorithm $\mathcal{T}(\cdot): \mathcal{D} \rightarrow \mathbb{R}^m$, its $\ell_q$-sensitivity is defined as:
\begin{equation}
\Delta_q(\mathcal{T}) := \sup_{D, D^{\prime} \in \mathcal{D}} \left\|\mathcal{T}(D) - \mathcal{T}\left(D^{\prime}\right)\right\|_q,
\end{equation}
where $D$ and $D^{\prime}$ differ in exactly one entry.
\end{definition}

The $\ell_q$-sensitivity of an algorithm $f$ quantifies the maximum change in the output of the algorithm when a single individual's data are modified. Intuitively, the sensitivity of an algorithm reflects the upper bound on how much we must perturb its output to preserve privacy. For different types of sensitivity measures, we can add noise from different distributions to achieve differential privacy. The most commonly used mechanisms are the Laplace mechanism and the Gaussian mechanism, which are summarized below.

\begin{lemma}[The Laplace and Gaussian Mechanisms \cite{dwork2014algorithmic}]
\label{lemma:laplace gaussian} Two fundamental mechanisms achieve differential privacy:
\begin{enumerate}
    \item \textbf{Laplace Mechanism.}  Let \(\mathcal{T}\) be a deterministic algorithm with \(\ell_{1}\)-sensitivity \(\Delta_{1}(\mathcal{T})\). Define
    \[
      \mathcal{U}(D) \;=\; \mathcal{T}(D) \;+\; \boldsymbol{\zeta}, 
      \quad \boldsymbol{\zeta} = (\zeta_1,\ldots,\zeta_m)^{\top},
      \quad \text{each } \zeta_{i} \overset{\text{ i.i.d }}{\sim} \mathrm{Lap}\bigl(0, \Delta_{1}(\mathcal{T})/\epsilon\bigr)^{m}.
    \]
    Then \(\mathcal{U}\) satisfies \((\epsilon,0)\)-DP.
    \item \textbf{Gaussian Mechanism.}  Let \(\mathcal{T}\) be a deterministic algorithm with \(\ell_{2}\)-sensitivity \(\Delta_{2}(\mathcal{T})\). Define
    \[
      \mathcal{U}(D) \;=\; \mathcal{T}(D) \;+\; \boldsymbol{\zeta},
      \quad \boldsymbol{\zeta}\sim\mathcal{N}\bigl(\boldsymbol{0},\sigma^{2}\mathbf{I}\bigr),
      \quad \sigma \;=\; \sqrt{2\ln(1.25/\delta)}\,\Delta_{2}(\mathcal{T}) / \epsilon.
    \]
    Then \(\mathcal{U}\) satisfies \((\epsilon,\delta)\)-DP.
\end{enumerate}
\end{lemma}

Lemma \ref{lemma:laplace gaussian} presents two fundamental mechanisms for achieving differential privacy, corresponding to the commonly used $\ell_1$- and $\ell_2$-sensitivities. The detailed proof can be referred to Theorems 3.6 and 3.22 in \cite{dwork2014algorithmic}. The following Proposition highlights several fundamental properties of differential privacy. 

\begin{proposition}[Properties of differential privacy \cite{dwork2014algorithmic}]
\label{lemma:DP properties} Differential privacy enjoys the following key properties:
\begin{enumerate}
    \item \textbf{Post-processing Immunity.} 
    If \(\mathcal{U}\) is \((\epsilon,\delta)\)-DP and \(f\) is any (possibly randomized) function, then \(f(\mathcal{U}(D))\) is also \((\epsilon,\delta)\)-DP.
    
    \item \textbf{Basic Composition.} 
    If \(\mathcal{U}_{1}\) is \((\epsilon_{1},\delta_{1})\)-DP and \(\mathcal{U}_{2}\) is \((\epsilon_{2},\delta_{2})\)-DP, then the combined mechanism
    \[
      D \,\mapsto\, \bigl(\mathcal{U}_{1}(D), \,\mathcal{U}_{2}(D)\bigr)
    \]
    satisfies \((\epsilon_{1} + \epsilon_{2},\,\delta_{1} + \delta_{2})\)-DP.
\end{enumerate}
\end{proposition}
Proposition~\ref{lemma:DP properties} states two fundamental properties of differential privacy. The post-processing property asserts that any function applied to the output of a differentially private mechanism, without access to the original data, cannot degrade the privacy guarantee. The basic composition property establishes that the cumulative privacy loss from multiple applications of differentially private mechanisms is additive, allowing the extension of Definition \ref{def:DP} to vector- and matrix-valued algorithms. The detailed proof can be found in Proposition 2.1, Theorems 3.14 and 3.20 in \cite{dwork2014algorithmic}.

In high-dimensional problems, it is common to assume that the true parameter vector is sparse. However, standard differential privacy mechanisms (see Lemma~\ref{lemma:laplace gaussian}) usually destroy sparsity, as they add noise to all entries. To solve this problem, \cite{Dwork_Su_Zhang_2021} proposed the Noisy Hard Thresholding ($\mathsf{NoisyHT}$), also known as the ``peeling'' algorithm, which is described in Algorithm~\ref{alg:NHT}. This method is now widely used in private high-dimensional data analysis, with successful applications in recent works such as \cite{cai2021cost,xia2023adaptive,cai2023private,zhang2024differentially}.

\begin{algorithm*}
\caption{Noisy Hard Thresholding ($\mathsf{NoisyHT}(\boldsymbol{\zeta}, s, \epsilon, \delta, \lambda)$).} 
\label{alg:NHT}
\begin{algorithmic}[1]
\STATE {\bf   {Input}:} Vector $\boldsymbol{\zeta} \in \mathbb{R}^p$, sparsity $s$, privacy parameters $(\epsilon, \delta)$, sensitivity $\lambda$, operator $\tilde{P}_{\mathcal{J}}(\cdot)$.
\STATE Initialize $\mathcal{J} = \varnothing$.
    \FOR{$i = 1$ \TO $s$}
        \STATE Generate $\boldsymbol{\eta}_i \in \mathbb{R}^p$ with $\eta_{i1}, \ldots, \eta_{ip} \stackrel{\text{i.i.d.}}{\sim} \mathsf{Laplace}\left(\lambda \cdot 2\sqrt{3s\log(1/\delta)}/{\epsilon}\right)$.
        \STATE Update $\mathcal{J} \leftarrow \mathcal{J} \cup \left\{\argmax_{j \in [p]\setminus \mathcal{J}} |\zeta_j| + \eta_{ij}\right\}$.
    \ENDFOR
\STATE Set $\boldsymbol{\zeta}_\mathcal{J} = \tilde{P}_{\mathcal{J}}(\boldsymbol{\zeta})$.
$\tilde{P}_{\mathcal{J}}(\boldsymbol{\zeta})$ retains the selected entries $\boldsymbol{\zeta}_{\mathcal{J}}$ and sets the remaining coordinates in $\boldsymbol{\zeta}_{\mathcal{J}^c}$ to zero
\STATE Generate $\tilde{\boldsymbol{\eta}}$ with $\tilde{\eta}_1, \ldots, \tilde{\eta}_p \stackrel{\text{i.i.d.}}{\sim} \mathsf{Laplace}\left(\lambda \cdot 2\sqrt{3s\log(1/\delta)} / \epsilon \right)$.
\STATE Set $\tilde{\boldsymbol{\eta}}_\mathcal{J} = \tilde{P}_{\mathcal{J}}(\tilde{\boldsymbol{\eta}})$.
\STATE {\bf Output:} $\boldsymbol{\zeta}_\mathcal{J} + \tilde{\boldsymbol{\eta}}_\mathcal{J}$.
\end{algorithmic}
\end{algorithm*}

Algorithm \ref{alg:NHT} generates an $s$-sparse approximation of $\boldsymbol{\zeta}\in\mathbb{R}^p$ under $(\epsilon,\delta)$-DP by iteratively selecting the coordinates with the largest Laplace-perturbed magnitudes, where the noise scale is $2\lambda\sqrt{3s\log(1/\delta)}/\epsilon$.
After $s$ selections, the operator $\tilde{P}_{\mathcal{J}}(\boldsymbol{\zeta})$ retains the selected entries and sets the remaining coordinates in $\boldsymbol{\zeta}_{\mathcal{J}^c}$ to zero. Additional Laplace noise of the same scale is then added to the selected entries. This private top-$s$ selection mechanism forms the foundation for the more advanced algorithms developed in the following sections.

\subsection{Quantile Regression Model}\label{subsec:qr}

Quantile regression is a powerful tool to model the complete relationship between the covariates and the response variable while exploring heterogeneous effects \cite{koenker1978regression,koenker2005quantile,he2023smoothed}. Given a scalar response variable \(Y  \in \mathbb{R}\) and a \((p+1)\)-dimensional covariate vector \(\boldsymbol{X} = (x_0, x_1, \ldots, x_p)^{\top} \in \mathbb{R}^{p+1}\) with \(x_0 \equiv 1\), the goal of quantile regression is to estimate the conditional quantile function \( Q_\tau(Y|\boldsymbol{X}) \) for a given quantile level \( \tau \in (0,1) \). This function represents the value of \( Y \) such that \( \mathbb{P}(Y \leq Q_\tau(Y|\boldsymbol{X}) | \boldsymbol{X}) = \tau \). We consider a linear quantile regression model, where the \( \tau \)-conditional quantile is
\begin{equation}\label{eq:linear quantile model}
Q_\tau(Y|\boldsymbol{X}) = \bfX^{\top} \boldsymbol{\beta}^*(\tau) = \sum_{j=0}^p x_j \beta_j^*(\tau),
\end{equation}
where \( \boldsymbol{\beta}^*(\tau) = (\beta_0^*(\tau), \beta_1^*(\tau), \ldots, \beta_p^*(\tau))^{\top} \) denotes the true coefficient vector. Actually, $\boldsymbol{\beta}^*(\tau)$ can be obtained by minimizing the following risk function:
\begin{equation}\label{quantile_loss}
\mathcal{Q}(\boldsymbol{\beta}) = \mathbb{E}\left[\rho_\tau(Y - \boldsymbol{X}^{\top}\boldsymbol{\beta})\right],   
\end{equation}
with \( \rho_\tau(u) = u\{\tau - \mathbb{I}(u \leq 0)\} \) being the check loss function \cite{koenker1978regression}. Since we focus on one fixed quantile level, we write \( \boldsymbol{\beta}^* \) in place of \( \boldsymbol{\beta}^*(\tau) \) throughout the paper.

The above quantile regression model~\eqref{eq:linear quantile model} can be equivalently written as a linear model as follows:
\begin{equation}
    Y = \bfX^{\top} \boldsymbol{\beta}^* + \varepsilon~~\text{with}~~\mathbb{P}( \varepsilon\leq 0 |\bfX) = \tau,
\end{equation}
where \( \varepsilon \) is the noise term and we assume the conditional density of \( \varepsilon \) given the covariates \(\bfX\) exists. In this paper, we consider the high-dimensional setting, where the dimensionality \(p\) diverges with the sample size \(N\), allowing \(p \to \infty\) to grow as \(N \to \infty\). The true parameter $\boldsymbol{\beta}^*$ is assumed to be $s^*$-sparse for some finite $s^*$.

\subsection{Newton-type Transformation for Quantile Regression}\label{subsec:newtontype}
Although quantile regression is robust against outliers and skewed or heavy-tailed noise distributions, it poses computational challenges in large sample sizes and high-dimensional settings due to the non-smooth check loss function \( \rho_\tau(\cdot) \). To address this issue, we first employ a smoothing technique that transforms the quantile regression problem into a least squares problem, which is inspired by the works of \cite{chen2020distributed, wang2024distributed}. Here we employ the Newton-Raphson method to minimize the quantile risk function in \eqref{quantile_loss}. Given a reasonable initial estimator \(\boldsymbol{\beta}_0\), the population form of the Newton-Raphson iteration is given by:
\begin{equation} \label{N-R_iter}
\boldsymbol{\beta}_{1} = \boldsymbol{\beta}_{0} - \boldsymbol{H}^{-1}(\boldsymbol{\beta}_{0})\mathbb{E}\left[\partial \mathcal{Q}(\boldsymbol{\beta}_{0})\right],
\end{equation}
where \(\partial \mathcal{Q}(\boldsymbol{\beta}) = \boldsymbol{X}\{\mathbb{I}(Y - \boldsymbol{X}^{\top}\boldsymbol{\beta} \leq 0) - \tau\}\) is the subgradient of the check loss function with respect to \(\boldsymbol{\beta}\), and \(\boldsymbol{H}(\boldsymbol{\beta}) = \partial \mathbb{E}[\partial \mathcal{Q}(\boldsymbol{\beta})]/\partial \boldsymbol{\beta}^{\top} = \mathbb{E}[\boldsymbol{X}\boldsymbol{X}^{\top} f_{\varepsilon|\boldsymbol{X}}(\boldsymbol{X}^{\top}(\boldsymbol{\beta}-\boldsymbol{\beta}^*))]\) represents the population Hessian matrix of \(\mathbb{E}[\mathcal{Q}(\boldsymbol{\beta})]\). Here, \(f_{\varepsilon|\boldsymbol{X}}(\cdot)\) is the conditional density of \(\varepsilon\) given \(\boldsymbol{X}\).

If the initial estimator \(\boldsymbol{\beta}_0\) is sufficiently close to the true parameter \(\boldsymbol{\beta}^*\), then \(\boldsymbol{H}(\boldsymbol{\beta}_0)\) serves as a good approximation to \(\boldsymbol{H}(\boldsymbol{\beta}^*) = \mathbb{E}[\boldsymbol{X}\boldsymbol{X}^{\top} f_{\varepsilon|\boldsymbol{X}}(0)]\).
Motivated by this insight, we proceed to approximate \(\boldsymbol{H}(\boldsymbol{\beta}^*)\) using a kernel-type matrix, denoted by \(\boldsymbol{D}_h(\boldsymbol{\beta}_0)\). With a slight abuse of notation, the intuitive relationship can be expressed as follows:
\[
\boldsymbol{H}(\boldsymbol{\beta}^*) \approx \boldsymbol{H}(\boldsymbol{\beta}_0) \approx  \boldsymbol{D}_h(\boldsymbol{\beta}_0) := \mathbb{E}[\boldsymbol{X} \boldsymbol{X}^{\top} H_h(e_0)],
\]
where \(e_0 = Y - \boldsymbol{X}^{\top} \boldsymbol{\beta}_0\), and \(H_h(\cdot) = H(\cdot/h)/h\), with \(H(\cdot)\) being a symmetric, non-negative kernel function and \(h \) representing the bandwidth.

According to \cite{wang2024distributed}, we transform the Newton-Raphson iteration into a least squares problem by defining the following pseudo covariates and response variable:
\begin{equation}\label{eq:newton-type}
\begin{aligned}
    \widetilde{\boldsymbol{X}}_h^{(1)} &= \sqrt{H_h(e_0)} \boldsymbol{X}, \\
    \widetilde{Y}_h^{(1)} = \widetilde{\boldsymbol{X}}_h^{(1)\top} \boldsymbol{\beta}_0 &- \frac{1}{\sqrt{H_h(e_0)}} \left(\mathbb{I}(e_0 \leq 0) - \tau\right).
\end{aligned}
\end{equation}
  {
Intuitively, the pseudo covariates reweight the original covariates according to the local kernel-smoothed residual density, while the pseudo response incorporates the quantile score correction, so that the Newton-type update can be equivalently written as a least-squares regression problem.
}
Plugging \(\boldsymbol{D}_h(\boldsymbol{\beta}_0), \widetilde{\boldsymbol{X}}_h^{(0)}\) and \(\widetilde{Y}_h^{(0)}\) into the Newton-Raphson iteration \eqref{N-R_iter}, a simple calculation yields that 
$$
\boldsymbol{\beta}_{1} = \boldsymbol{D}_h(\boldsymbol{\beta}_0)^{-1} \mathbb{E}[\widetilde{\boldsymbol{X}}_h^{(1)} \widetilde{Y}_h^{(1)}].
$$
Note that $\boldsymbol{D}_h(\boldsymbol{\beta}_0)=\mathbb{E}[\widetilde{\boldsymbol{X}}_h^{(1)}\widetilde{\boldsymbol{X}}_h^{(1) \top}]$, the equation \eqref{N-R_iter} can be interpreted as a least squares regression of \(\widetilde{Y}_h^{(1)}\) on \(\widetilde{\boldsymbol{X}}_h^{(1)}\):
$$
\boldsymbol{\beta}_1 = \underset{\boldsymbol{\beta} \in \mathbb{R}^{p+1}}{\operatorname{argmin}} \frac{1}{2} \mathbb{E}\left(\widetilde{Y}_h^{(1)} - \widetilde{\boldsymbol{X}}_h^{(1) \top} \boldsymbol{\beta}\right)^2.
$$
In high-dimensional sparse settings, it is natural to impose an $\ell_0$-constraint on the optimization problem. Accordingly, the one-step population-level optimization can be formulated as
\begin{equation}\label{itr:expection}
    \boldsymbol{\beta}_1 = \underset{\boldsymbol{\beta} \in \mathbb{R}^{p+1},\, \|\boldsymbol{\beta}\|_0 \le s^*}{\operatorname{argmin}} \frac{1}{2} \mathbb{E}\left[\left(\widetilde{Y}_h^{(1)} - \widetilde{\boldsymbol{X}}_h^{(1)\top} \boldsymbol{\beta}\right)^2\right].
\end{equation}
Thus, the original quantile regression problem is equivalently reformulated as a sparse least squares optimization. We can further iterate this process to obtain a sequence of population-level estimators \(\boldsymbol{\beta}_t\) for \(t \geq 1\).
 {The following lemma shows that the true parameter $\boldsymbol{\beta}^*$ minimizes the population loss function, where the pseudo-variables are constructed using the true residuals. Although this property is implicit in the Newton-type construction and its proof is elementary, it is not stated as a standalone lemma in \cite{wang2024distributed}. We include it here for completeness and to make the subsequent differentially private extension self-contained.}

\begin{lemma}\label{lemma:fix_point}
  {If the true parameter $\|\boldsymbol{\beta}^*\|_0 \le s^*$ and $\mathbb{E}[H_h(\varepsilon) \boldsymbol{X} \boldsymbol{X}^{\top}]$ is invertible, where $\varepsilon = Y - \boldsymbol{X}^{\top} \boldsymbol{\beta}^*$, then $\boldsymbol{\beta}^*$ is the unique minimizer of the following optimization problem:}
$$
\boldsymbol{\beta}^* = \underset{\boldsymbol{\beta} \in \mathbb{R}^{p+1}, \|\boldsymbol{\beta}\|_0 \le s^*}{\operatorname{argmin}} \frac{1}{2} \mathbb{E} \left( \widetilde{Y}_h - \widetilde{\boldsymbol{X}}_h^{\top} \boldsymbol{\beta} \right)^2,
$$
where $\widetilde{\boldsymbol{X}}_h = \sqrt{H_h(\varepsilon)}\boldsymbol{X}$ and $\widetilde{Y}_h = \widetilde{\boldsymbol{X}}_h^{\top} \boldsymbol{\beta}^* - \frac{1}{\sqrt{H_h(\varepsilon)}}\left(\mathbb{I}(\varepsilon\leq 0)-\tau\right)$.
\end{lemma}
\begin{proof}
We first focus on the unconstrained optimal solution to this least squares problem:
\begin{align*}    
    \underset{\boldsymbol{\beta} \in \mathbb{R}^{p+1}}{\operatorname{argmin}}\frac{1}{2}\mathbb{E}\left(\widetilde{Y}_h-\widetilde{\bfX}_h^{\top}\boldsymbol{\beta}\right)^2
    &=
    \mathbb{E}[\widetilde{\bfX}_h \widetilde{\bfX}_h^{\top}]^{-1} 
    \mathbb{E}[\widetilde{\bfX}_h \widetilde{Y}_h] \\
    &= \mathbb{E} [H_h(\varepsilon)  \bfX \bfX^{\top}]^{-1} 
    \mathbb{E} [H_h(\varepsilon)  \bfX \bfX^{\top} \boldsymbol{\beta}^*
    - \bfX \left(\mathbb{I}(\varepsilon\leq 0)-\tau\right)] \\
    &= \mathbb{E} [H_h(\varepsilon)  \bfX \bfX^{\top}]^{-1} 
    \mathbb{E} [H_h(\varepsilon)  \bfX \bfX^{\top}] \boldsymbol{\beta}^* 
    = \boldsymbol{\beta}^*,
\end{align*}
where we use $\mathbb{E} [\bfX \left(\mathbb{I}(\varepsilon\leq 0)-\tau\right)] = \boldsymbol{0}$.
The lemma is valid since the $\ell_0$-norm of $\boldsymbol{\beta}^{*}$ is less than or equal to $s^*$.
\end{proof}
  {
From the perspective of M-estimation, Lemma~\ref{lemma:fix_point} characterizes an oracle population update induced by the Newton-type transformation. It shows that, when the true residual is used, the transformed least-squares objective uniquely targets the true parameter \(\boldsymbol{\beta}^*\) over the prescribed sparse parameter space. This population-level property links the Newton-type transformation to the oracle estimating equation and provides a theoretical foundation for the empirical algorithm developed below.}

Suppose we observe a random sample \(\mathcal{Z}^N = \{(\bfX_i, Y_{i}) \}_{i=1}^{N}\) and obtain an initial estimator \(\widehat{\boldsymbol{\beta}}_0\), then we can iteratively calculate the equation in the same way as \eqref{eq:newton-type} based on the sample. For the \(t\)-th iteration, let \(\widehat{\boldsymbol{\beta}}_{t-1}\) be the empirical estimate after \((t-1)\) iterations, then we can construct the pseudo covariates and response variable  as follows:
\begin{equation}\label{eq:newton-type-t}
\begin{aligned}
    \widetilde{\boldsymbol{X}}_{i,h}^{(t)} &= \sqrt{H_h(\widehat{e}_{i,t-1})} \bfX_i, \\
    \widetilde{Y}_{i,h}^{(t)} = \widetilde{\boldsymbol{X}}_{i,h}^{(t)\top} \widehat{\boldsymbol{\beta}}_{t-1}&- \frac{1}{\sqrt{H_h(\widehat{e}_{i,t-1})}} \left(\mathbb{I}(\widehat{e}_{i,t-1} \leq 0) - \tau\right),
\end{aligned}
\end{equation}
where $\widehat{e}_{i,t} = Y_i - \bfX_i^{\top} \widehat{\boldsymbol{\beta}}_{t}$.
In parallel with the population problem \eqref{itr:expection}, we now introduce its sample analogue for the \(t\)-th iteration as follows:
\begin{equation}\label{itr:sample}
\widehat{\boldsymbol{\beta}}_{t} = \underset{\boldsymbol{\beta} \in \mathbb{R}^{p+1}, \|\boldsymbol{\beta}\|_0 \le s^*}{\operatorname{argmin}} \frac{1}{2N} \sum_{i=1}^N \left( \widetilde{Y}_{i,h}^{(t)} - \widetilde{\boldsymbol{X}}_{i,h}^{(t) \top} \boldsymbol{\beta} \right)^2.
\end{equation}

 {We emphasize that the Newton-type transformation reviewed in this section is adapted from \cite{wang2024distributed}. Our main contribution begins in the next section, where this transformation is integrated with distributed gradient aggregation, sparse private updates, and formal differential privacy constraints to develop private estimation and inference procedures.}

\section{Differentially Private Estimation}\label{sec:DP-estimation}
In this section, we leverage the Newton-type transformation introduced in the last section to develop a differentially private quantile regression estimation algorithm in a distributed setting. Suppose that the random sample $\mathcal{Z}^N = \{(\bfX_i, Y_{i}) \}_{i=1}^{N}$ is distributed randomly and evenly across $m$ machines, denoted as index set $\mathcal{M}_{1}, \ldots, \mathcal{M}_{m}$, with each machine storing $n = N/m$ samples. Without loss of generality, we treat $\mathcal{M}_{1}$ as the central machine. The data stored on the $k$-th machine is denoted by $\{ (\bfX_{i}, Y_{i}) \}_{i\in \mathcal{M}_k}$, where $|\mathcal{M}_{k}|=n$ for $k=1,\ldots,m$. Define the local and global loss functions at $t$-th iteration as 
\begin{equation}\label{global_local_loss}
    \begin{aligned}
        &\text{Local loss: } \mathcal{L}^{(t)}_k(\boldsymbol{\beta}) = \frac{1}{2n} \sum_{i \in \mathcal{M}_k} \left( \widetilde{Y}_{i,h}^{(t)} - \widetilde{\boldsymbol{X}}_{i,h}^{(t) \top} \boldsymbol{\beta} \right)^2, \quad \text{and} \\
        &  \text{Global loss: } \mathcal{L}^{(t)}_N(\boldsymbol{\beta}) = \frac{1}{m} \sum_{k=1}^m \mathcal{L}^{(t)}_{k}(\boldsymbol{\beta})=\frac{1}{2N} \sum_{i=1}^N \left( \widetilde{Y}_{i,h}^{(t)} - \widetilde{\boldsymbol{X}}_{i,h}^{(t) \top} \boldsymbol{\beta} \right)^2.
    \end{aligned}
\end{equation}

Our objective is to minimize the global loss function \(\mathcal{L}^{(t)}_N(\boldsymbol{\beta})\) within a distributed framework, while ensuring differential privacy. The intuition of our method is the combination of distributed gradient descent and the Noisy Hard Thresholding algorithm (Algorithm~\ref{alg:NHT}). Specifically, at each outer iteration, each local machine first applies the Newton-type transformation to its covariates and responses via \eqref{T_x_t} and \eqref{T_y_t}. Each machine then computes its local gradient via \eqref{local_gradient} and transmits the $(p+1)$-dimensional vector to the central machine. The central machine first aggregates the gradients from all $m$ machines, then performs a gradient descent update with step size $\eta^1/m$. We enforce both sparsity and differential privacy by applying the $\mathsf{NoisyHT}$ operator with a privacy budget of $(\epsilon/(mKT), \delta/(mKT))$. After privatization and truncation, the estimate is projected onto the feasible set $\{\boldsymbol{\beta} \in \mathbb{R}^{p+1}: \| \boldsymbol{\beta}\|_\infty \leq C_1\}$ and subsequently transmitted to all local machines, which then update their local parameters before moving to the next inner iteration. Algorithm~\ref{alg:DP:SparseEstimation} implements the detailed estimation procedure for distributed high-dimensional quantile regression under $(\epsilon,\delta)$-DP by alternating local kernel smoothing with globally aggregated and privatized gradient updates. 

The most related work to our algorithm is the distributed differentially private sparse estimation algorithms in \cite{zhang2024differentially}, which also use the $\mathsf{NoisyHT}$ operator to ensure sparsity and differential privacy. However, they focused on linear regression models with smooth loss functions, while our algorithm is designed for quantile regression with non-smooth loss functions. The key difference lies in the use of the Newton-type transformation to convert the quantile regression problem into a least squares problem, enabling the application of distributed gradient descent methods.  {We further compare Algorithm~\ref{alg:DP:SparseEstimation} with the distributed high-dimensional quantile regression methods in \cite{chen2020distributed,wang2024distributed}, which also employ Newton-type transformations to address the non-smoothness of the quantile loss. Their methods first construct a surrogate loss by replacing the global Hessian with a local Hessian and selecting a local smoothing bandwidth, and then solve the resulting sparse optimization problem using algorithms such as PSSsp \cite{schmidt2010graphical} or coordinate descent \cite{wright2015coordinate}. In contrast, our algorithm does not rely on surrogate loss construction. Each local machine only computes local gradients and sends them to the central machine, which performs gradient aggregation, a distributed gradient descent update, and a Noisy Hard Thresholding step. The $\mathsf{NoisyHT}$ step is embedded in every iteration to enforce sparsity and differential privacy simultaneously. As a result, our method requires sensitivity control and privacy-budget allocation across machines and iterations, and its error analysis explicitly contains a privacy-induced term. 
} 

\begin{remark}
In our algorithm, we assume that the central machine is fully trusted and does not collude with any of the local machines. In each inner iteration, the local machines send the exact gradient to the central machine without privacy protection. However, the central machine applies the $\mathsf{NoisyHT}$ operator to the aggregated gradient, which ensures that the information sent back to the local machines is privatized. Subsequently, the local machines update the gradients based on the privatized output from the central machine. This design is crucial for maintaining differential privacy while allowing the central machine to perform necessary computations without compromising the privacy of individual data points. Similar trusted central machine design has been adopted in distributed high-dimensional linear regression \cite{zhang2024differentially}, where they also proved that accurate estimation is infeasible even in a simple sparse mean estimation problem under the distributed setting without a trusted central machine.
\end{remark}

\begin{remark}
 {As in most Newton-type and smoothing-based quantile regression methods \citep{chen2020distributed,wang2024distributed}, a sufficiently accurate initialization and a well-chosen bandwidth are necessary for reliable local approximation between the kernel-smoothed Hessian and the population Hessian. We also acknowledge that the interaction between Newton-type updates and Noisy Hard Thresholding is nontrivial. When the initialization is less accurate or the empirical Hessian approximation is imperfect, the transformed gradient update may become unstable in finite samples. However, the $\mathsf{NoisyHT}$ step helps mitigate this instability by enforcing sparsity at each iteration, thereby preventing privacy noise from accumulating over all coordinates. In addition, the projection operator $\Pi_{C_1}$ keeps the iterates in a bounded feasible region and avoids excessively large coordinate-wise updates. These mechanisms provide algorithmic stabilization under privacy perturbation, although they do not replace the local initialization condition required by the theory. Our analysis makes these effects explicit through the privacy-induced error term and the iterative approximation term. Additional empirical sensitivity analyses support the robustness of the proposed procedure with respect to initialization, kernel bandwidth, and Hessian conditioning.}
\end{remark}

\begin{algorithm*}[t]
\caption{Distributed Differentially Private High-dimensional Quantile Regression.} 
\hspace*{0.02in} 
\begin{algorithmic}[1]\label{alg:DP:SparseEstimation}
\STATE {\bf Input:} 
Dataset $\{ (\bfX_{i}, Y_{i}) \}_{i\in \mathcal{M}_k}$, for $k=1,\ldots,m$, bandwidth $h$, quantile level $\tau$, sparsity $s \ge s^*$, stepsize $\eta^1$, privacy parameters $(\epsilon, \delta)$, number of iterations $(T, K)$, feasibility parameter $C_1$, initial estimator $\widehat{\boldsymbol{\beta}}_0$, and noise scale $B_0$.
\FOR{$t$ from 1 to $T$}
\STATE{
For each local machine $j=1,2, \ldots, m$, compute the pseudo covariates and response variable based on the previous estimate $\widehat{\boldsymbol{\beta}}_{t-1}$:
\begin{equation}\label{T_x_t}
\widetilde{\bfX}_{i,h}^{(t)}=\sqrt{H_h(Y_i-\bfX_i^{\top} \widehat{\boldsymbol{\beta}}_{t-1})}\bfX_i,   
\end{equation}
\begin{equation}\label{T_y_t}
\widetilde{Y}_{i,h}^{(t)}=(\widetilde{\bfX}_{i,h}^{(t)})^{\top} \widehat{\boldsymbol{\beta}}_{t-1} - \frac{1}{\sqrt{H_h(Y_i-\bfX_i^{\top} \widehat{\boldsymbol{\beta}}_{t-1})}}\left(\mathbb{I}(Y_i - \bfX_i^{\top} \widehat{\boldsymbol{\beta}}_{t-1} \leq 0) - \tau \right).
\end{equation}
}
\STATE{
Let $\boldsymbol{\beta}_t^{1} = \widehat{\boldsymbol{\beta}}_{t-1}$.
}
\quad \FOR{$k$ from $1$ to $K$} 
\STATE{
For each local machine $j=1,2, \ldots, m$, calculate the local gradient, 
\begin{align}\label{local_gradient}
\mathbf{g}_j^{(t)} = \frac{1}{n} \sum_{i \in \mathcal{M}_j} \left( \widetilde{\bfX}_{i,h}^{(t) \top} \boldsymbol{\beta}_t^{k} - \widetilde{Y}_{i,h}^{(t)} \right) \widetilde{\bfX}_{i,h}^{(t)},   
\end{align}
}
\STATE{
and then send the gradient $\mathbf{g}_j^{(t)}$ to the central machine.
}
\STATE{For the central machine, aggregate the local gradients and perform the gradient descent update:
 $$
 \boldsymbol{\beta}_t^{k+0.5} = \boldsymbol{\beta}_t^{k} - \left(\eta^1 / m \right) \sum_{j=1}^m \mathbf{g}_j^{(t)};
 $$
}
\STATE{
then compute $\boldsymbol{\beta}_t^{k+1} = \Pi_{C_1} \left( \mathsf{NoisyHT} \left( \boldsymbol{\beta}_t^{k+0.5}, s, \frac{\epsilon}{mKT}, \frac{\delta}{mKT}, \frac{\eta^1 B_0}{mn} \right) \right)$, and send the output $\boldsymbol{\beta}_t^{k+1}$ back to each local machine from the machine.
}
\ENDFOR
\STATE{
 Let $\widehat{\boldsymbol{\beta}}_{t} = \boldsymbol{\beta}_t^{K}$.
}
\ENDFOR
\STATE {\bf Output:} 
Return $\widehat{\boldsymbol{\beta}}_T$.
\end{algorithmic}
\end{algorithm*}

Now we establish the theoretical guarantees for our proposed differentially private distributed estimation procedure. Before presenting the main results, we introduce several regularity assumptions.
\begin{assumption}\label{ass:parameter}
   The true parameter $\boldsymbol{\beta}^{*}$ satisfies $\|\boldsymbol{\beta}^*\|_2 \leq c_0$ and $\|\boldsymbol{\beta}^*\|_0 \leq s^*$. 
    We consider a high-dimensional regime in which the dimensionality $p$ may grow polynomially with the sample size $N$, that is, $p=\mathcal{O}\left( N^c \right)$ for some $c\in (0,1/2)$. And we further assume the sparsity satisfies $s^* = {o}((\log p)^{1/2})$. 
\end{assumption}
\begin{assumption}\label{ass:bounddesign}
The random covariate $\boldsymbol{X} \in \mathbb{R}^{p+1}$ is sub-Gaussian, i.e., there exists some $c_1>0$ such that 
$$
{\mathbb{P}}\left(\left|\boldsymbol{X}^{\top} \boldsymbol{\Sigma}^{-1/2} \boldsymbol{\nu} \right| \geqslant c_1 t\right) \leq 2 e^{-t^2 / 2}
$$ 
for every unit vector $\boldsymbol{\nu}$ and $t>0$, where $\boldsymbol{\Sigma}=\mathbb{E}(\boldsymbol{X}\boldsymbol{X}^{\top})$, and there exists some $C_x< \infty$ such that $\|\bfX\|_\infty \le C_x$.
Furthermore, $0 \le \lambda_{\min} \le \Lambda_{\min}(\boldsymbol{\Sigma}) \le \Lambda_{\max}(\boldsymbol{\Sigma}) \le  \lambda_{\max} < \infty$ and the precision matrix $\boldsymbol{\Sigma}^{-1}$ satisfies
$\|\boldsymbol{\Sigma}^{-1}\|_{1} \le C$.
Besides, $m_4 = \sup_{\boldsymbol{\nu} \in \mathbb{S}^{p}} \mathbb{E}(|\langle \boldsymbol{\nu}, \boldsymbol{\Sigma}^{-1/2}\bfX \rangle |^4) < \infty$.
\end{assumption}
\begin{assumption}\label{ass:kernel}
Assume that the kernel function $H(\cdot)$ is symmetric, non-negative, bounded, and integrates to one. In addition, the kernel function satisfies
\[
\int_{-\infty}^{\infty} u^2 H(u) \, \mathrm{d} u < \infty,~~ \kappa_u = \max_{u} H(u),~~\text{and}~~\min _{|u| \leq 1} H(u) > 0.
\]
We further assume that $H(\cdot)$ is second-order differentiable, and its derivative $H^{\prime}(\cdot)$ and second derivative $H^{\prime\prime}(\cdot)$ are bounded.
Moreover, denote
\[
\kappa_k = \int_{-\infty}^{\infty} |u|^k H(u) \, \mathrm{d} u \quad \text{for} \quad k \ge 1.
\]
\end{assumption}
\begin{assumption}\label{ass:f_varepsilon}
    There exist constants $f_2 \geq f_1 > 0$ such that
    \[
    f_1 \leq f_{\varepsilon \mid \boldsymbol{X}}(0) \leq f_2 
    \]
    almost surely over $\boldsymbol{X}$. Moreover, there exists some constant $l_0$ such that
    \[
    \left|f_{\varepsilon \mid \boldsymbol{X}}(u) - f_{\varepsilon \mid \boldsymbol{X}}(v)\right| \leq l_0 |u - v|
    \]
    for any $u,v \in \mathbb{R}$ almost surely over $\boldsymbol{X}$.
\end{assumption}

In Assumption 1, the sparsity condition and the growing conditions on $p$ and $s^*$ are standard in high-dimensional sparse quantile regression literature \cite{chen2020distributed,wang2024distributed}.  {The $\ell_2$-boundedness condition $\|\boldsymbol{\beta}^*\|_2\leq c_0$ is needed to ensure that the true parameter lies in a bounded region, which is a common technical condition for the theoretical analysis of differentially
private algorithms \cite{cai2021cost,zhang2024differentially}.}
Assumption \ref{ass:bounddesign} imposes that the covariate $\bfX$ is sub-Gaussian with uniformly bounded covariance eigenvalues and the kurtosis of arbitrary linear projection $\langle \boldsymbol{\nu}, \boldsymbol{\Sigma}^{-1/2}\bfX \rangle$ has finite fourth moments, which are standard conditions in high-dimensional statistical theory \cite{chen2019quantile,chen2020distributed,wang2024distributed,Yibo2023Confidence}.  {Moreover, to guarantee differential privacy, the precision matrix obeys the elementwise $\ell_1$--norm bound $\|\boldsymbol\Sigma^{-1}\|_{1}\le C<\infty$ \cite{wang2021differentially,su2020differentially}. The condition $\|\boldsymbol X\|\infty\leq C_x$ is also important for differential privacy, because it allows us to bound the global $\ell_\infty$-sensitivity of the local gradients, which is needed to calibrate the noise in the $\mathsf{NoisyHT}$ and Gaussian mechanisms. A similar assumption can also be referred to in \cite{cai2021cost,cai2023private,zhang2024differentially}}. Assumption \ref{ass:kernel} is a standard condition on kernel function \cite{wang2024distributed, Yibo2023Confidence}, stipulating that $H(\cdot)$ is a symmetric, nonnegative kernel density integrating to one, twice continuously differentiable with bounded derivatives, and possessing finite moments $\kappa_k$ for $k \ge 1$.  
Assumption \ref{ass:f_varepsilon} ensures that the conditional error density at zero is bounded away from both zero and infinity and satisfies a global Lipschitz condition, which is standard in the context of quantile regression \cite{chen2019quantile,chen2020distributed,tan2022communication,wang2024distributed}.

The following Theorems \ref{thm1} and \ref{thm2} provide upper bounds on the estimation error for the one-step estimator $\widehat{\boldsymbol{\beta}}_1$ and the $T$-step estimator $\widehat{\boldsymbol{\beta}}_T$, respectively.
\begin{theorem}\label{thm1}
Suppose the initial estimator satisfies $\|\widehat{\boldsymbol{\beta}}_{0} - \boldsymbol{\beta}^*\|_2 = \mathcal{O}_{\mathbb{P}}(a_N)$, and $s^*a_N = o(1)$. Exist  
    $K \in \mathbb{N}^+$,
    the bandwidth satisfies $h \asymp a_N$, and the local sample size satisfies
    \[
    N \gtrsim (s^*)^{3/2} \log p \log n \sqrt{\log N \log(1/\delta)} / \epsilon.
    \]
    Then, under Assumptions \ref{ass:parameter}-\ref{ass:f_varepsilon}, there holds
    \begin{equation}\label{bound:beta1}
    \begin{aligned}
        \|\widehat{\boldsymbol{\beta}}_1- \boldsymbol{\beta}^*\|_2 
        &\lesssim \sqrt{\frac{s^* \log p}{N}} + \sqrt{\frac{(s^*\log p)^2 \log(1/\delta)\log^3 N}{N^2\epsilon^2}} + \sqrt{s^*}  a_N^2 
    \end{aligned}
    \end{equation}
    with probability   {$1 - N^{-C}$}. In addition, Algorithm \ref{alg:DP:SparseEstimation} is $(\epsilon, \delta)$-DP.
\end{theorem}

With proper choice of the bandwidth $h$ and inner iteration $K$, we can refine the initial estimator by one iteration of Algorithm \ref{alg:DP:SparseEstimation}. Specifically, Theorem~\ref{thm1} improves the initial rate $\mathcal{O}_{\mathbb{P}}(a_N)$ to the oracle-plus-privacy rate, up to the next-order iterative approximation term, when $\sqrt{s^*}a_N=o(1)$. Recursively applying Theorem \ref{thm1} yields the convergence rate of the $T$-step estimator $\widehat{\boldsymbol{\beta}}_T$.

\begin{theorem}\label{thm2}
    Suppose that the assumptions and conditions in Theorem \ref{thm1} hold. Then, the final estimator of Algorithm \ref{alg:DP:SparseEstimation} satisfies the following error bound
    \begin{align}
    \label{bound:betaT}
    \| \widehat{\boldsymbol{\beta}}_T - \boldsymbol{\beta}^*\|_2
    \lesssim {}&
    \sqrt{\frac{s^* \log p}{N}}
    + \sqrt{\frac{(s^*\log p)^2 \log(1/\delta)\log^3 N}{N^2\epsilon^2}}
    \nonumber\\
    &+(\sqrt{s^{*}})^{T^2-T+1} a_N^{2T}.
    \end{align}
     with probability   {$1 - T N^{-C}$}.
\end{theorem}
The bound in \eqref{bound:betaT} can be decomposed as follows:
\[
\underbrace{\sqrt{\frac{s^* \log p}{N}}}_{\text{Oracle convergence rate} } + \underbrace{\sqrt{\frac{(s^*\log p)^2 \log(1/\delta)\log^3 N}{N^2\epsilon^2}}}_{\text{DP error}} + \underbrace{(\sqrt{s^{*}})^{T^2-T+1} a_N^{2T}}_{\text{Loss setting error} }.
\]
The first term reflects the oracle convergence rate representing the statistical error, the second term quantifies the error due to differential privacy, and the third term captures the error arising from the initialization and iterative procedure. 
When the number of iterations $T$ is sufficiently large,
i.e.,
\begin{align}\label{iter_T}
T >  \frac{\log \left(\sqrt{\frac{s^* \log p}{N}}  + \sqrt{\frac{(s^*\log p)^2 \log(1/\delta)\log^3 N}{N^2\epsilon^2}}/a_N\right)}{\log (C\sqrt{s^*}a_N)} \quad \text{for some constant } \ C>0,
\end{align}
the third term in \eqref{bound:betaT} becomes negligible compared to the first two terms, leading to the convergence of the estimator $\widehat{\boldsymbol{\beta}}_T$ to the true parameter $\boldsymbol{\beta}^*$ at the oracle rate plus the privacy cost,
\begin{equation}\label{convergence_T}
\bigl\|\widehat{\boldsymbol{\beta}}_T - \boldsymbol{\beta}^*\bigr\|_2 
\;\lesssim\;
\sqrt{\frac{s^*\log p}{N}}
\;+\;
\sqrt{\frac{(s^*\log p)^2\,\log(1/\delta)\,\log^3 N}{N^2\epsilon^2}},
\end{equation}
which matches the minimax convergence rate up to some logarithmic factors established in \cite{cai2021cost} under the non-distributed high-dimensional sparse linear regression setting. 

\begin{remark}
 {Compared with existing non-private distributed high-dimensional quantile regression results \cite{chen2020distributed,wang2024distributed}, Theorem~\ref{thm2} explicitly accounts for the effect of differential privacy. In particular, the estimation error is decomposed into the oracle statistical error, the privacy-induced error, and the iterative approximation error. The additional privacy term arises from the Noisy Hard Thresholding step, which is embedded into each distributed update to enforce sparsity and ($\epsilon,\delta$)-differential privacy simultaneously. Furthermore, beyond private sparse estimation, we develop differentially private pseudo precision matrix estimation, debiasing, confidence intervals, and multiplier bootstrap procedures for simultaneous inference in the subsequent parts. Therefore, our framework is not a direct privatization of existing distributed quantile regression algorithms, but a complete privacy-preserving estimation and inference framework for distributed high-dimensional quantile regression.}
\end{remark}

\section{Differentially Private Inference}\label{sec:DP-inference}
In this section, we develop statistical inference procedures for the proposed differentially private estimator. We first introduce the debiasing method for the multi-step estimator, and then extend it to the distributed setting with lower computational and communication costs. To ensure the differential privacy of the debiased estimator, we use a differentially private precision matrix estimation method and apply it to the debiasing procedure. Finally, we construct the differentially private coordinate-wise confidence intervals for the parameters.

It is noteworthy that the multi-step estimator $\widehat{\boldsymbol{\beta}}_T$ is biased due to the hard-thresholding operation in the $\mathsf{NoisyHT}$ operator. To eliminate the bias and enable valid inference, we apply a debiasing technique commonly used in high-dimensional statistics \cite{Geer2013OnAO, zhang2014confidence, Yibo2023Confidence}. Specifically, the debiased estimator is defined as
\begin{equation}\label{eq:debias}
\widehat{\boldsymbol{\beta}}^{de}_T = \widehat{\boldsymbol{\beta}}_T - \widehat{\mathbf{W}}  \frac{1}{N}\sum_{i=1}^N (\widetilde{Y}^{(T)}_{i,h}-\widetilde{\boldsymbol{X}}_{i,h}^{{(T)}\top}\widehat{\boldsymbol{\beta}}_T) \widetilde{\boldsymbol{X}}_{i,h}^{(T)},
\end{equation}
where $\widehat{\mathbf{W}}$ denotes an approximate inverse of $\boldsymbol{H}(\boldsymbol{\beta}^*)$. Since $\boldsymbol{H}(\boldsymbol{\beta}^*)$ is not directly observable, we estimate it using the sample covariance matrix $\widehat{\boldsymbol{D}}_h^{(T)} = (1/N)\sum_{i=1}^N \widetilde{\boldsymbol{X}}_{i,h}^{(T)} \widetilde{\boldsymbol{X}}_{i,h}^{(T)\top}$. This serves as a consistent estimator for $\boldsymbol{H}(\boldsymbol{\beta}^*) \approx \boldsymbol{H}(\widehat{\boldsymbol{\beta}}_{T-1}) \approx \mathbb{E}[\boldsymbol{X} \boldsymbol{X}^{\top} H_h(Y-\boldsymbol{X}^{\top}\widehat{\boldsymbol{\beta}}_{T-1})]$, as guaranteed by Theorem \ref{thm2} when $T$ is sufficiently large.

Recall the distributed setting in Section \ref{sec:DP-estimation}, we assume the entire data is randomly and evenly stored in $m$ local machines with sample size $n=N/m$. A naive approach is to calculate the approximate inverse of $\boldsymbol{H}(\boldsymbol{\beta}^*)$ using the local covariance matrix on each local machine, then average them to obtain the global covariance matrix, and finally compute the debiased estimator \eqref{eq:debias}. However, this approach requires each local machine to estimate the $(p+1) \times (p+1)$-dimensional precision matrix, and communicate it to the central machine, which incurs high computational and communication costs. To address these limitations, we propose a one-step debiased estimator after the iterative procedure in Algorithm \ref{alg:DP:SparseEstimation} to reduce computation and communication costs. At $T_0$-th iteration, we further define the one-step debiased estimator as follows:
\begin{equation}\label{debais:dis1}
    \widetilde{\boldsymbol{\beta}}_{T_0}^{de} =
    \widehat{\boldsymbol{\beta}}_{T_0} - \widehat{\mathbf{W}}^{(1)}_b \frac{1}{N}\sum_{i=1}^N\left( \widetilde{Y}_{i,h}^{(T_0)} - \widetilde{\boldsymbol{X}}_{i,h}^{(T_0)\top}\widehat{\boldsymbol{\beta}}_{T_0} \right)\widetilde{\bfX}_{i,h}^{(T_0)} ,
\end{equation}
where $\widehat{\mathbf{W}}^{(1)}_b$ is computed based  on \(\widehat{\boldsymbol{D}}_{1,b}^{(T_0)} = (1/n) \sum_{i \in \mathcal{M}_1} \widetilde{\boldsymbol{X}}_{i,h}^{(T_0)} \widetilde{\boldsymbol{X}}_{i,h}^{(T_0)\top}\) in the first machine with $b$ being the local bandwidth different from $h$, and $\widehat{\boldsymbol{\beta}}_{T_0}$ is the $T_0$-step estimator obtained from Algorithm \ref{alg:DP:SparseEstimation} with $T_0$ satisfying \eqref{iter_T}. To achieve a trade-off between computational efficiency and statistical accuracy, we only use the local precision matrix estimator instead of the global averaged one. Note that $\widehat{\boldsymbol{D}}_{1,b}^{(T_0)}=(1/n)\sum_{i \in \mathcal{M}_1}H_b(Y_i-{\bfX}_{i}^{\top}\widehat{\boldsymbol{\beta}}_{T_0-1}){\bfX}_{i}{\bfX}_{i}^{\top}$, under Assumptions in Theorem \ref{thm2} and with $T_0$ satisfying \eqref{iter_T}, the $(T_0-1)$-th step estimator can achieve the near optimal convergence rate as shown in \eqref{convergence_T}. This is a key condition that helps to derive the non-asymptotic error bound for $\widehat{\mathbf{W}}^{(1)}_b$, which is crucial for the subsequent inference procedure. Here, we also want to emphasize that the global bandwidth $h$ is used to estimate $\widehat{\boldsymbol{\beta}}_{T_0}$ and the gradients in Algorithm \ref{alg:DP:SparseEstimation}, while the local bandwidth $b$ is used to estimate the precision matrix $\widehat{\mathbf{W}}^{(1)}_b$.

\subsection{DP-Constrained $\ell_1$-Minimization for Pseudo Precision Matrix Estimation}\label{subsec:clime}
To estimate the pseudo precision matrix $\widehat{\mathbf{W}}^{(1)}_b$, we consider the $\mathsf{CLIME}$ method proposed by \cite{cai2011constrained}, which is a constrained $\ell_1$-minimization problem that estimates the sparse inverse covariance matrix. The $\mathsf{CLIME}$ method solves the following optimization problem:
\begin{equation}\label{CLIME}
\begin{aligned}
\widehat{\mathbf{W}}_b^{(1)}= & \underset{\mathbf{W} \in \mathbb{R}^{(p+1) \times (p+1)}}{\operatorname{argmin}}\|\mathbf{W}\|_{\infty},  \quad\text {s.t.} \ \|\mathbf{W} \widehat{\boldsymbol{D}}_{1,b}^{(T_0)} - \mathbf{I}\|_{\infty} \leq \gamma_{N,n},
\end{aligned} 
\end{equation}
 {where \(\gamma_{N,n}\) is a pre-specified tolerance parameter used in the CLIME-type feasibility constraint. It is chosen to dominate the statistical fluctuation, smoothing bias, and privacy-induced perturbation arising from the private pseudo covariance estimation. The explicit form of \(\gamma_{N,n}\) will be provided in Theorem~\ref{thm3}.}
Subsequently, we propose a differentially private variant of the $\mathsf{CLIME}$ method, which adds Gaussian noise to the sample covariance matrix $\widehat{\boldsymbol{D}}_{1,b}^{(T)}$ before applying the $\mathsf{CLIME}$ procedure. This approach is inspired by the work of \cite{Wang2019DifferentiallyPH} and \cite{su2020differentially}, who developed differentially private graphical Lasso estimators using noise-addition mechanisms. To ensure the symmetry, we average the output with its transpose, i.e., $(\widehat{\mathbf{W}}_b^{(1)} + \widehat{\mathbf{W}}_b^{(1)\top})/2$. For notational simplicity, we assume $\widehat{\mathbf{W}}_b^{(1)}$ is symmetric throughout the remainder of the paper. We summarize the differentially private precision matrix estimation procedure in Algorithm \ref{alg:DP:PrecisionEstimation}.

\begin{algorithm*}[t]
\caption{Differentially Private Pseudo Precision Matrix Estimation.} 
\hspace*{0.02in} 
\begin{algorithmic}[1]\label{alg:DP:PrecisionEstimation}
\STATE {\bf Input:} 
Dataset $\{ (\bfX_{i}, Y_{i}) \}_{i\in \mathcal{M}_1}$, kernel function $H(\cdot)$, local bandwidth $b$, quantile level $\tau$, and privacy parameters $(\epsilon, \delta$), and noise scale $B_1$.
\STATE {Run Algorithm \ref{alg:DP:SparseEstimation} to obtain $\widehat{\boldsymbol{\beta}}_{T_0-1}$ and $\widehat{\boldsymbol{\beta}}_{T_0}$ with $T_0-1$ satisfying \eqref{iter_T}.}
\STATE{
Compute the pseudo covariates and sample covariance matrix:
$$
\widetilde{\bfX}_{i,b}^{(T_0)} = \sqrt{H_b(Y_i-\boldsymbol{X}_i^{\top} \widehat{\boldsymbol{\beta}}_{T_0-1})}\bfX_i, \quad \widehat{\boldsymbol{D}}_{1,b}^{(T_0)} = \frac{1}{n}\sum_{1 \le i \le n} \widetilde{\bfX}_{i,b}^{(T_0)} \widetilde{\bfX}_{i,b}^{(T_0) \top}.
$$
}
\STATE{
Add the noise to the pseudo sample covariance matrix $\widehat{\boldsymbol{D}}_{1,b}^{(T_0)}$:
$$
\widetilde{\boldsymbol{D}}_{1,b}^{(T_0)} = \widehat{\boldsymbol{D}}_{1,b}^{(T_0)} + \mathbf{G},
$$
where $\mathbf{G} \in \mathbb{R}^{(p+1)\times(p+1)}$ is a symmetric matrix. The entries in the upper triangle of $\mathbf{G}$ are independently drawn from the normal distribution $\mathcal{N}(0, \frac{B_1 \log^2 (2np^2) \kappa_u^2 \log (1.25/\delta)}{n^2 \epsilon^2})$, and the symmetry is enforced by mirroring these values to the lower triangle.
}
\STATE{
Compute privacy estimation $\widehat{\mathbf{W}}_{b}^{(1)}$ by $\mathsf{CLIME}$ based on $\widetilde{\boldsymbol{D}}_{1,b}^{(T_0)}$:
\begin{equation}\label{CLIME_1}
\begin{aligned}
\widehat{\mathbf{W}}_b^{(1)}= & \underset{\mathbf{W} \in \mathbb{R}^{(p+1) \times (p+1)}}{\operatorname{argmin}}\|\mathbf{W}\|_{\infty},  \quad\text {s.t.} \ \|\mathbf{W} \widetilde{\boldsymbol{D}}_{1,b}^{(T_0)} - \mathbf{I}\|_{\infty} \leq \gamma_{N,n},
\end{aligned} 
\end{equation}
}
\STATE {\bf Output:} Return $\widehat{\mathbf{W}}_{b}^{(1)} = (\widehat{\boldsymbol{w}}_0, \widehat{\boldsymbol{w}}_1, \ldots, \widehat{\boldsymbol{w}}_p)^{\top}$.
\end{algorithmic}
\end{algorithm*}

Now, we present the non-asymptotic error bound for the differentially private precision matrix estimator $\widehat{\mathbf{W}}^{(1)}_b$ as an approximation of $\widetilde{\boldsymbol{D}}_{1,b}^{(T_0)}$, and thus $\boldsymbol{H}^{-1}(\boldsymbol{\beta}^*)$. Before that, we impose an additional assumption on $\boldsymbol{H}^{-1}(\boldsymbol{\beta}^*)$.

\begin{assumption}\label{ass:hessian_inv}
    For the $\boldsymbol{H}^{-1}(\boldsymbol{\beta}^*):= \left(\widetilde{\boldsymbol{h}}_0, \ldots, \widetilde{\boldsymbol{h}}_p\right)^{\top} = (\widetilde{h}_{i, j})_{1 \leq i, j \leq p}$, there exists some $L>0$,  such that $\|\boldsymbol{H}^{-1}(\boldsymbol{\beta}^*)\|_{L_1} \le L$. Moreover, $\boldsymbol{H}^{-1}(\boldsymbol{\beta}^*)$ is sparse row-wise, i.e., $\max_{0 \le i \le p}\sum_{j=0}^p \mathbb{I}(\widetilde{h}_{i,j} \neq 0) \le c_{N,p}$, where $c_{N, p}$ is positive and bounded away from 0 and allowed to increase as $N$ and $p$ grow.
\end{assumption}

Assumption \ref{ass:hessian_inv} imposes row-sparsity and matrix $L_1$-norm constraints on 
$\boldsymbol{H}^{-1}(\boldsymbol{\beta}^*)$.  {Intuitively, $\boldsymbol{H}^{-1}(\boldsymbol{\beta}^*)$ determines the debiasing directions used to correct the bias of individual coordinates. The row-sparsity condition means that the debiasing direction for each coefficient depends only on a relatively small subset of other covariates. The bounded $L_1$-norm condition prevents the debiasing weights from becoming too large. Such a structure is plausible in applications where covariates exhibit local, block, or network dependence, such as spatial data, genomics data with pathway structures, or financial data with sector-level dependence.} This assumption is standard in the literature on precision matrix estimation and generalized inverse Hessian estimation \cite{cai2011constrained,Geer2013OnAO}. A related condition is also imposed in \cite{Yibo2023Confidence} for inference in convolution-smoothing quantile regression, where the sparsity is required for the inverse of the population kernel matrix $\mathbb{E}(H_h(\varepsilon) \bfX \bfX^{\top})$, which depends on the bandwidth $h$. In contrast, our assumption concerns the sparsity of the inverse of the population Hessian matrix $\boldsymbol{H}(\boldsymbol{\beta}^{*})$ associated with the quantile loss, which does not depend on the bandwidth. This makes our condition more broadly applicable and reliable across different quantile regression settings.

\begin{theorem}\label{thm3}
Under the Assumptions \ref{ass:parameter}-\ref{ass:hessian_inv}, for the output of Algorithm \ref{alg:DP:PrecisionEstimation}, with probability   {$1-C/p$} we have 
\begin{equation}\label{eq:lemma1}
\begin{aligned}\|\widehat{\mathbf{W}}^{(1)}_b\|_{L_1} & \leq \|\boldsymbol{H}^{-1}(\boldsymbol{\beta}^*)\|_{L_1}, 
\ 
\|\widehat{\mathbf{W}}^{(1)}_b \widetilde{\boldsymbol{D}}_{1,b}^{(T_0)} -\mathbf{I}\|_{\infty} \lesssim  \gamma_{N,n}, \
\text { and } \ 
\|\widehat{\mathbf{W}}^{(1)}_b \boldsymbol{H}^{-1}(\boldsymbol{\beta}^*) -\mathbf{I}\|_{\infty} \lesssim  \gamma_{N,n}, \\
\end{aligned}
\end{equation}
where
\begin{align*}
\gamma_{N,n} =& \frac{4\kappa_u \log (2np^2)  \sqrt{\log{p^3}} }{n} + \sqrt{\frac{\log p}{nb}} + \frac{\log p}{nb} 
+ \frac{s^* (\log p)^2 }{N b^3} (\frac{s^* \log p}{N} +\frac{(s^*\log p)^2\,\log(1/\delta)\,\log^3 N}{N^2\epsilon^2}) \\
&+ s^* \sqrt{\log p} (\frac{1}{b}+\sqrt{\frac{\log p}{n b^3}}+\frac{\log p}{n b^2})( \sqrt{\frac{\log p}{N}} + \sqrt{\frac{s^*(\log p)^2\,\log(1/\delta)\,\log^3 N}{N^2\epsilon^2}}) + b^2.
\end{align*} 
Thus, with probability   {$1-C/p$}, we have 
\begin{equation}\label{eq:lemma2}
\|\widehat{\mathbf{W}}^{(1)}_b - \boldsymbol{H}^{-1}(\boldsymbol{\beta}^*)\|_{L_1} \lesssim \gamma_{N,n}.
\end{equation}
Also, Algorithm \ref{alg:DP:PrecisionEstimation} is $(\epsilon, \delta)$-DP. 
\end{theorem}
Theorem \ref{thm3} provides a non-asymptotic error bound for the differentially private precision matrix estimator $\widehat{\mathbf{W}}^{(1)}_b$ in terms of the tuning parameter $\gamma_{N,n}$, which is a function of the sample size $N$, local sample size $n$, sparsity level $s^*$, the local bandwidth $b$ and privacy parameters $(\epsilon, \delta)$. 
\begin{remark}\label{convergence:W}
By choosing the local bandwidth as $b \asymp (s^*\log p / n)^{1/3}$, the error bound in~\eqref{eq:lemma1} can be simplified as
$$
\gamma_{N,n}  \lesssim \sqrt{\frac{\log^2 p}{n}} + \sqrt{\frac{\log^{10/3} p\,\log(1/\delta)\, n^{2/3}\log^3 N}{N^2 \epsilon^2}}.
$$    
\end{remark}


\subsection{DP Confidence Intervals via Debiased Estimator}
In this part, we develop a differentially private coordinate-wise confidence interval for ${\beta}^*_j$. Before proceeding, we first construct the differentially private debiased estimator. The primary idea is to add Gaussian noise to the debiased estimator $\widetilde{\boldsymbol{\beta}}_{T_0}^{de}$ defined in \eqref{debais:dis1} to ensure differential privacy:
\begin{equation}\label{debais:def1}
    \widetilde{\boldsymbol{\beta}} =
    \widetilde{\boldsymbol{\beta}}_{T_0}^{de} + \boldsymbol{E},
\end{equation}
where $\boldsymbol{E}\in \mathbb{R}^{p+1}$ is generated from
\[
\boldsymbol{E}\sim
\mathcal{N}\left(
\boldsymbol{0},
\frac{2B_2^2\log(1.25/\delta)}{n^2m^2\epsilon^2}\,\mathbf{I}
\right),
\]
and $B_2$ is the noise scale. Now, we establish the Bahadur representation for $\widetilde{\boldsymbol{\beta}}$. Denote $\widehat{\boldsymbol{\Delta}} = \widehat{\boldsymbol{\beta}}_{T_0} - \boldsymbol{\beta}^*$ and consider the coordinate-wise estimator $\widetilde{\beta}_j = \mathbf{e}_j^{\top} \widetilde{\boldsymbol{\beta}}$, for \(j=0,\dots,p\),
\begin{equation}\label{Bahadur1}
\begin{aligned}
\sqrt{N}(\widetilde{\beta}_j-\beta^*_j)
={}&-\widetilde{\boldsymbol{h}}_j^{\top}
\frac{1}{\sqrt{N}}\sum_{i=1}^N(\mathbb{I}(\varepsilon_i \le 0)-\tau)\bfX_i
+\sqrt{N}E_j\\
&-\underbrace{(\widehat{\boldsymbol{w}}_j-\widetilde{\boldsymbol{h}}_j)^{\top}
\frac{1}{\sqrt{N}}\sum_{i=1}^N(\mathbb{I}(\varepsilon_i \le 0)-\tau)\bfX_i}_{:=\Gamma_1}\\
&-\underbrace{\sqrt{N}\widehat{\boldsymbol{w}}_j^{\top}
G_N(\widehat{\boldsymbol{\Delta}})}_{:=\Gamma_2}\\
&-\underbrace{\sqrt{N}\widehat{\boldsymbol{w}}_j^{\top}
\left[\frac{1}{N}\sum_{i=1}^N f_{\varepsilon|\bfX}(0)\bfX_i\bfX_i^\top
-\boldsymbol{H}(\boldsymbol{\beta}^*)\right]\widehat{\boldsymbol{\Delta}}}_{:=\Gamma_3}\\
&-\underbrace{\sqrt{N}
(\widehat{\boldsymbol{w}}_j\boldsymbol{H}(\boldsymbol{\beta}^*)-\mathbf{e}_j)^\top
\widehat{\boldsymbol{\Delta}}}_{:=\Gamma_4}\\
&-\underbrace{\widehat{\boldsymbol{w}}_j^\top
\frac{1}{2\sqrt{N}}\sum_{i=1}^N
f'_{\varepsilon|\bfX}(\theta)
(\bfX_i^\top\widehat{\boldsymbol{\Delta}})^2\bfX_i}_{:=\Gamma_5},
\end{aligned}
\end{equation}
where $G_N(\widehat{\boldsymbol{\Delta}}) = (1/N)\sum_{i=1}^N  \{\mathbb{I}(\varepsilon_i \le \bfX_i^{\top} \widehat{\boldsymbol{\Delta}}) - \mathbb{P}(\varepsilon_i \le \bfX_i^{\top} \widehat{\boldsymbol{\Delta}} |\bfX_1,\ldots, \bfX_N) - [\mathbb{I}(\varepsilon_i \le 0) - \mathbb{P}(\varepsilon_i \le 0|\bfX_1,\ldots, \bfX_N)]\}\bfX_i$, $E_j=\mathbf{e}_j^{\top}\boldsymbol{E}$, $\widetilde{\boldsymbol{h}}_j$ and $\widehat{\boldsymbol{w}}_j$ are the $j$-th row of $\boldsymbol{H}^{-1}(\boldsymbol{\beta}^*)$ and $\widehat{\mathbf{W}}^{(1)}_b$, respectively. Here, $\Gamma_1$ to $\Gamma_5$ are the Bahadur remainders, which can be well controlled based on the error bounds in Theorems \ref{thm1} and \ref{thm3}. The detailed calculation of the Bahadur representation follows from the decomposition of the remainder terms.

\begin{theorem}\label{thm:bahadur}
    Suppose that the conditions of Theorems \ref{thm1} and \ref{thm3} hold. The local bandwidth $b \asymp (s^*\log p / n)^{1/3}$, then the Bahadur representation satisfies
    \begin{align*}
    &\left| \sqrt{N}(\widetilde{\beta}_j - \beta^*_j) + \widetilde{\boldsymbol{h}}_j^{\top} \frac{1}{\sqrt{N}}\sum_{i=1}^N(\mathbb{I}(\varepsilon_i \le 0)- \tau)\bfX_{i} - \sqrt{N} E_j \right| \\
    \lesssim&   \sqrt{\frac{ \log^{3} p}{n}}  + \frac{ \log^{5/2} p}{N^{1/4}} + \sqrt{\frac{\log^{13/3} p \log(1/\delta) n^{2/3} \log^3 N}{N^2 \epsilon^2}}
    \end{align*}
      {with probability at least $1-C/p$.}
\end{theorem}

Theorem \ref{thm:bahadur} provides the non-asymptotic Bahadur representation for the debiased coordinate-wise estimator $\widetilde{\beta}_j$. With proper choice of the local bandwidth $b$, when the sample size $N,n \rightarrow \infty$, the Bahadur remainder converges to zero. Note that $(1/\sqrt{N})\sum_{i=1}^N(\mathbb{I}(\varepsilon_i \le 0)- \tau)$ is a zero-mean random variable, and by the de Moivre-Laplace central limit theorem, it converges to a Gaussian distribution with variance $\tau(1-\tau)$. Consequently, we can establish the asymptotic normality of the debiased coordinate-wise estimator $\widetilde{\beta}_j$. 

\begin{corollary}
\label{cor:normality}
Suppose the conditions of Theorem~\ref{thm:bahadur} hold. Then, for $0 \le j \le p$ as $N\to\infty$,
\[
\sqrt{N}\,\bigl(\widetilde\beta_j - \beta_j^*\bigr)
\;\xrightarrow{d}\;
\mathcal{N}\left(0,\; \tau(1-\tau)\,\widetilde{\boldsymbol{h}}_j^\top \boldsymbol{\Sigma}\,\widetilde{\boldsymbol{h}}_j  +   {\frac{8B_2^2 \log (1/\delta)}{ \epsilon^2}}\right),
\]
where ``$\xrightarrow{d}$'' denotes convergence in distribution.
\end{corollary}

The Bahadur representation in Theorem \ref{thm:bahadur} and the asymptotic normality in Corollary \ref{cor:normality} enable us to construct confidence intervals and conduct hypothesis tests for the quantile regression parameters.  {The debiased inference results can be directly extended to any finite-dimensional linear functionals of the form $\boldsymbol{\nu}^{\top}\widetilde{\boldsymbol \beta}$, where $\boldsymbol{\nu}$ lies in the $\ell_1$-ball $\mathbb{B}_1(r)$ with fixed radius $r>0$. This covers many common inferential targets, such as testing a pre-specified contrast or a weighted combination of several coefficients. Specifically, the non-asymptotic Bahadur representation and the Gaussian approximation results can be directly applied to $\boldsymbol{\nu}^{\top}\widetilde{\boldsymbol \beta}$, with the variance term in the confidence interval adjusted accordingly:} 
\[
  {
\sqrt{N} \boldsymbol{\nu}^{\top}(\widetilde{\boldsymbol \beta} - \boldsymbol{\beta}^*)
\;\xrightarrow{d}\;
\mathcal{N}\left(0,\; \tau(1-\tau)\, \boldsymbol{\nu}^{\top} \boldsymbol{H}^{-1}(\boldsymbol{\beta}^*) \boldsymbol{\Sigma}\, \boldsymbol{H}^{-1}(\boldsymbol{\beta}^*) \boldsymbol{\nu} +   {\frac{8\boldsymbol{\nu}^{\top}\boldsymbol{\nu}B_2^2 \log (1/\delta)}{ \epsilon^2}}\right).
}
\]

The remaining problem is to estimate the asymptotic variance. In particular, we use the local $\mathsf{CLIME}$ estimator $\widehat{\mathbf{W}}^{(1)}_b$ and sample covariance matrix $\widehat{\mathbf{\Sigma}}^{(k)}$ on each local machine to estimate the variance of the Bahadur representation. Algorithm \ref{alg:DPconfidentinterval} constructs a differentially private $(1-\alpha)$ confidence interval for the debiased coordinate-wise estimator $\widetilde{\beta}_j$ in four main steps. First, the central machine runs Algorithm \ref{alg:DP:PrecisionEstimation} to compute the pseudo precision matrix $\widehat{\mathbf{W}}_b^{(1)}$, and broadcasts the $j$-th column $\widehat{\boldsymbol{w}}_j$ to all $m$ local machines. Second, each local machine $k$ computes its local gradient $\boldsymbol{g}_k$ and variance contribution $\widehat{\sigma}_j^{(k)}$, and sends both to the central machine. Third, the central machine aggregates the gradients, samples a noise vector $E_j \sim \mathcal{N}(0, B_2^2 \log(1.25/\delta)/(N^2 \epsilon^2))$, and forms the differentially private debiased estimator as \eqref{point_de_estimate}. Finally, the central machine combines the local variance estimates into the global variance, and constructs the two-sided $(1-\alpha)$ confidence interval for the $j$-th coordinate as \eqref{interval}. We also note that $8B_2^2 \log (1/\delta)/(N\epsilon^2) \to 0$ as $N\to\infty$ in \eqref{interval}, this term is retained in the asymptotic variance for precise finite-sample confidence interval construction, as it captures the differential privacy noise contribution. The next theorem establishes the validity of the confidence interval.

\begin{algorithm*}[t]
\caption{Distributed Differentially Private Confidence Interval for $\beta_j^*$.} 
\hspace*{0.02in} 
\begin{algorithmic}[1]\label{alg:DPconfidentinterval}
\STATE {\bf Input:} 
Dataset $\{ (\bfX_{i}, Y_{i}) \}_{i\in \mathcal{M}_k}$, for $k=1,\ldots,m$, quantile level $\tau$, the level of significance $\alpha$, privacy parameters $(\epsilon, \delta)$, and noise scale $B_2$, and $\widehat{\boldsymbol{\beta}}_{T_0} $.
\STATE{
Run Algorithm \ref{alg:DP:PrecisionEstimation} to get $\widehat{\boldsymbol{w}}_j$ on the central machine and then send $\widehat{\boldsymbol{w}}_j$ to all local machines.
}
\FOR{$k$ from 1 to $m$}
\STATE{
On each local machine, calculate the local gradient
$$
   \boldsymbol{g}_k = \frac{1}{n}\sum_{i \in \mathcal{M}_k}(\mathbb{I}(Y_i-\boldsymbol{X}_i^{\top} \widehat{\boldsymbol{\beta}}_{T_0}  \le 0) - \tau)\bfX_{i},
$$
and
\begin{align*}
\widehat{\sigma}_j^{(k)} = \widehat{\boldsymbol{w}}_j^{\top} \widehat{\mathbf{\Sigma}}^{(k)} \widehat{\boldsymbol{w}}_j, \quad
\widehat{\mathbf{\Sigma}}^{(k)} = \frac{1}{n}\sum_{i \in \mathcal{M}_k}  \bfX_{i} \bfX_{i}^{\top}.
\end{align*}
Send $(\boldsymbol{g}_k, \widehat{\sigma}_j^{(k)})$ to the central machine.
}
\ENDFOR
\STATE{Generate $E_j$ from the Gaussian distribution $\mathcal{N}(0, \frac{ B_2^2 \log(1.25/\delta)}{n^2m^2\varepsilon^2} )$}.
\STATE{
Calculate DP debiased estimation:
}
\begin{align}\label{point_de_estimate}
    \widetilde{\beta}_j = \widehat{\beta}_{T_0,j}  + \widehat{\boldsymbol{w}}_j^{\top} \frac{1}{m}\sum_{k=1}^m \boldsymbol{g}_k + E_j. 
\end{align}

\STATE{
Calculate the confidence interval $\text{CI}_j(\alpha)$ for $\widetilde{\beta}_j^*$ on central machine.
\begin{equation}\label{interval}
    \begin{aligned}
        \text{CI}_j(\alpha) =& \left[\widetilde{\beta}_j - \Phi^{-1}(1-\alpha / 2) \frac{\sqrt{\tau(1-\tau)}}{\sqrt{N}} \sqrt{\frac{1}{m}\sum_{k=1}^m \widehat{\sigma}_j^{(k)}  + \frac{8B_2^2 \log (1/\delta)}{N\epsilon^2}}, \right. \\
        & \left.~~\widetilde{\beta}_j + \Phi^{-1}(1-\alpha / 2) \frac{\sqrt{\tau(1-\tau)}}{\sqrt{N}} \sqrt{\frac{1}{m}\sum_{k=1}^m \widehat{\sigma}_j^{(k)}  + \frac{8B_2^2 \log (1/\delta)}{N\epsilon^2}} \right].
    \end{aligned}
\end{equation}
}
\STATE {\bf Output:}
Return $\text{CI}_j(\alpha)$.
\end{algorithmic}
\end{algorithm*}

\begin{theorem}\label{thm:interval}
Suppose that the assumptions and conditions in Theorems \ref{thm1} and \ref{thm3} hold and the local bandwidth fulfills $b \asymp (s^*\log p / n)^{1/3}$. For $\forall \alpha \in (0,1)$ and $j = 0,\ldots,p$, there holds
$$
\sup_{ \alpha \in (0,1)}\left| \mathbb{P}(  \beta_j^{*} \in \mathrm{CI}_j(\alpha)) - (1-\alpha) \right| \lesssim \sqrt{\frac{ \log^{3} p}{n}}  + \frac{ \log^{5/2} p}{N^{1/4}} + \sqrt{\frac{\log^{13/3} p \log(1/\delta) n^{2/3} \log^3 N}{N^2 \epsilon^2}}.
$$
Moreover, the $j$-th confidence interval is asymptotically valid, i.e.,  
$$
\lim _{N, n \rightarrow \infty} \mathbb{P}\left(\beta_j^{*} \in \mathrm{CI}_j(\alpha)\right)=1-\alpha.
$$
Also, Algorithm \ref{alg:DPconfidentinterval} is $(\epsilon, \delta)$-DP.
\end{theorem}

Theorem \ref{thm:interval} provides a non-asymptotic Berry--Esseen bound for the debiased coordinate-wise estimator $\widetilde{\beta}_j$, ensuring that the constructed confidence interval achieves the desired coverage probability. Furthermore, Theorem \ref{thm:interval} rigorously   {verifies} that Algorithm \ref{alg:DPconfidentinterval} satisfies $(\epsilon, \delta)$-differential privacy.  {With slight modifications, the same procedure can be applied to construct confidence interval for a linear functional of interest. In particular, the $(1-\alpha)$ confidence interval for $\boldsymbol{\nu}^{\top}\boldsymbol{\beta}^*$ can be constructed as
    \begin{align*}
      \text{CI}(\alpha) =& \left[ \boldsymbol{\nu}^{\top} \widetilde{\boldsymbol\beta} - \Phi^{-1}(1-\alpha / 2) \frac{\sqrt{\tau(1-\tau)}}{\sqrt{N}} \sqrt{\frac{1}{m}\sum_{k=1}^m \widehat{\sigma}^{(k)}  + \frac{8\boldsymbol{\nu}^{\top}\boldsymbol{\nu}B_2^2 \log (1/\delta)}{N\epsilon^2}}, \right. \\
        & \left.~~ \boldsymbol{\nu}^{\top} \widetilde{\boldsymbol\beta} + \Phi^{-1}(1-\alpha / 2) \frac{\sqrt{\tau(1-\tau)}}{\sqrt{N}} \sqrt{\frac{1}{m}\sum_{k=1}^m \widehat{\sigma}^{(k)}  + \frac{8\boldsymbol{\nu}^{\top}\boldsymbol{\nu}B_2^2 \log (1/\delta)}{N\epsilon^2}} \right],
    \end{align*}
where $\widehat{\sigma}^{(k)} = \boldsymbol{\nu}^{\top} \widehat{\mathbf{W}}_b^{(1)} \widehat{\mathbf{\Sigma}}^{(k)} \widehat{\mathbf{W}}_b^{(1)} \boldsymbol{\nu}$.}
Based on the results of confidence intervals, we can also construct hypothesis tests for the quantile regression parameters.

\begin{remark}\label{rmk:reliability}
 {There are two sources of privacy noise in our inference procedure. The first is the Gaussian noise added to the pseudo sample covariance matrix $\widehat{\boldsymbol{D}}_{1,b}^{(T_0)}$ for private pseudo precision matrix estimation. This noise affects the estimated debiasing direction and appears in the Bahadur remainder term through the precision matrix estimation error, as reflected by the privacy-induced terms in Theorems~\ref{thm:bahadur} and~\ref{thm:interval}. Therefore, this component needs to be asymptotically controlled for valid inference. In particular, it is sufficient to require
$$
\frac{\log^{13/3}p\log(1/\delta)n^{2/3}\log^3 N}{N^2\epsilon^2}=o(1).
$$
The second source is the Gaussian perturbation $E_j$ directly added to the debiased estimator. This term is explicitly separated in the Bahadur representation and is not part of the remainder bound. For the usual asymptotic normality without variance inflation, $E_j$ should be asymptotically negligible, for example when $\mathrm{Var}(\sqrt{N}E_j)\to 0$ or $1/(N\epsilon^2)=o(1)$. However, for the confidence interval in Algorithm~\ref{alg:DPconfidentinterval}, this final privacy noise does not need to be negligible relative to the sampling noise, because its variance is explicitly incorporated into the interval. Thus, when $E_j$ is comparable to the sampling noise, the confidence interval becomes wider, and the corresponding test has lower power, but the coverage remains valid after accounting for the privacy-induced variance inflation. Additional empirical characterization further illustrates this threshold behavior.}
\end{remark}

\section{Multiplier Bootstrap Private Inference for Distributed High-dimensional Quantile Regression}\label{sec:DP-bootstrap}
 {We now extend the inference task from individual coordinates to simultaneous inference. While the coordinate-wise results above can be adapted to linear functionals, many applications require inference over a large collection of parameters. Since a full joint Gaussian approximation for the whole debiased vector is difficult in high dimensions, especially under differential privacy, we instead use a maximum-type statistic and approximate its distribution by a differentially private multiplier bootstrap.} 

Simultaneous inference for high-dimensional models has been extensively studied in the literature, including works in single-machine settings \cite{zhang2017simultaneous} or distributed settings \cite{yu2022distributed}. However, these methods can not guarantee differential privacy, which is crucial for protecting sensitive data in distributed environments. To address this issue, we construct differentially private simultaneous confidence intervals using $\|\cdot\|_{\infty}$-norm based on the debiased estimator $\widetilde{\boldsymbol{\beta}}$ in \eqref{debais:def1}. Specifically, the $(1-\alpha)$ simultaneous confidence region for $\boldsymbol{\beta}^*$ is defined by the quantile 
\begin{equation}\label{SCI}
\begin{aligned}
\mathrm{c}(\alpha)
= \inf\Bigl\{t \in \mathbb{R}:\
&\mathbb{P}\bigl(\|\sqrt{N}(\widetilde{\boldsymbol{\beta}}-\boldsymbol{\beta}^*)\|_\infty \le t\bigr)
\ge \alpha
\Bigr\}.
\end{aligned}
\end{equation}
where $\alpha \in (0,1)$ represents the significance level.
However, the exact distribution of the statistic $\|\sqrt{N} (\widetilde{\boldsymbol{\beta}} - \boldsymbol{\beta}^*) \|_\infty$ is analytically intractable, especially in high dimensions. To overcome this, we employ a multiplier bootstrap procedure to approximate the sampling distribution. Theorem~\ref{thm:bahadur} shows that each coordinate of $\sqrt{N} (\widetilde{\boldsymbol{\beta}} - \boldsymbol{\beta}^*)$ admits a Bahadur representation and is asymptotically normal under suitable regularity conditions. This justifies the use of the bootstrap approach for constructing valid simultaneous confidence intervals and hypothesis tests in the high-dimensional setting. Building on the Gaussian approximation and multiplier bootstrap framework in \cite{chetverikov2013gaussian}, the standard (non-distributed) multiplier bootstrap statistic is defined as:
\begin{equation}\label{N-grad}
\boldsymbol{w}^{*} 
= \widehat{\mathbf{W}}  \frac{1}{\sqrt{N}} \sum_{i=1}^N \xi_{i} \left( \mathbb{I}(\widehat{e}_{i, T_0} \le 0) - \tau \right) \boldsymbol{X}_{i},
\end{equation} 
where \(\xi_{1},\dots,\xi_{N}\) are i.i.d.~from \(\mathcal{N}(0,1)\) and independent from data. 

The classical multiplier bootstrap in \eqref{N-grad} requires generating \(N\) Gaussian multipliers per bootstrap replication, which quickly becomes computationally and communicationally prohibitive for large-scale distributed data. To overcome this limitation, we adopt the \(\mathsf{m\text{-}grad}\) or \(\mathsf{(n+m-1)\text{-}grad}\) distributed multiplier bootstrap framework in \cite{yu2022distributed} and using the local CLIME precision matrix $\widehat{\mathbf{W}}_{b}^{(1)}$. This approach is valid for arbitrary numbers of machines \(m\) and requires at most \((n + m - 1)\) multipliers per replication, where \(n = N/m\) is the local sample size. Specifically, the distributed bootstrap comes in two variants:

\medskip  
\noindent (i) \(\mathsf{m\text{-}grad}\) method (for large \(m\)):
\begin{equation}\label{k-grad}
\boldsymbol{w}^{\sharp}
=\widehat{\mathbf{W}}_b^{(1)}\;\frac{1}{\sqrt{m}}\sum_{j=1}^{m}\xi_{j}\,\sqrt{n}\,\bigl(\mathfrak{g}_{j}-\bar{\mathfrak{g}}\bigr)\,,
\end{equation}
where \(\mathfrak{g}_j = (1/n)\sum_{i\in\mathcal M_k} \bigl(\mathbb I(\widehat e_{i,T_0}\le0)-\tau\bigr)\bfX_i\) is the local gradient from machine \(\mathcal{M}_j\) and \(\bar{\mathfrak{g}}=(1/m)\sum_{j=1}^m\mathfrak{g}_j\).

\medskip  
\noindent (ii) \(\mathsf{(n+m-1)\text{-}grad}\) method (for small \(m\)):
\begin{equation}\label{n-k-1-grad}
\boldsymbol{w}^{\flat}
=\widehat{\mathbf{W}}_b^{(1)}\;\frac{1}{\sqrt{\,n+m-1\,}}
\Biggl\{\sum_{i\in\mathcal{M}_{1}}\xi_{i}\,\bigl(\mathfrak{g}_{1i}-\bar{\mathfrak{g}}\bigr)
\;+\;\sum_{j=2}^{m}\xi_{\,n+j-1}\,\sqrt{n}\,\bigl(\mathfrak{g}_{j}-\bar{\mathfrak{g}}\bigr)\Biggr\}\,,
\end{equation}
where \(\mathfrak{g}_{1i} = \bigl(\mathbb I(\widehat e_{i,T_0}\le0)-\tau\bigr)\bfX_i\) is the \(i\)-th sample gradient on the central machine \(\mathcal{M}_1\).

\begin{algorithm*}
\caption{Private Bootstrap Method for Multiple Testing in Distributed Learning.} 
\hspace*{0.02in}
\begin{algorithmic}[1]\label{alg:DPbootstrap}
\STATE {\bf Input:} 
Dataset $\{ (\bfX_{i}, Y_{i}) \}_{i\in \mathcal{M}_k}$, for $k=1,\ldots,m$, quantile level $\tau$, the level of significance $\alpha$, the number of bootstrap replication $n_B$, threshold $m_0$, privacy parameters $(\epsilon, \delta)$, $T_0$-round estimator $\widehat{\boldsymbol{\beta}}_{T_0}$, and noise level $B_3$.
\STATE{
Send $\widehat{\boldsymbol{\beta}}_{T_0}$ to each local machine.
}
\FOR{$k$ from 2 to $m$}
\STATE{
On each local machine, calculate the local gradient
$$
\mathfrak{g}_k = \frac{1}{n}\sum_{i \in \mathcal{M}_k}(\mathbb{I}(Y_i-\boldsymbol{X}_i^{\top}\widehat{\boldsymbol{\beta}}_{T_0} \le 0) - \tau)\bfX_{i}.
$$
Send $\mathfrak{g}_k$ to the central machine.
}
\ENDFOR
\STATE{On the central machine, calculate
$$
\mathfrak{g}_{1i} = \bigl(\mathbb I(Y_i-\boldsymbol{X}_i^{\top}\widehat{\boldsymbol{\beta}}_{T_0} \le0)-\tau\bigr)\bfX_i, \quad \mathfrak{g}_{1} = \frac{1}{n}\sum_{i=1}^n \mathfrak{g}_{1i}. 
$$
}
\STATE{
On the central machine, solve the optimization in \eqref{CLIME} to get a solution $\widehat{\mathbf{W}}_b^{(1)}$.
}
\STATE \textbf{Bootstrap:} Generate $\xi_1,\ldots,\xi_{n+m-1}\overset{\text{i.i.d.}}{\sim}\mathcal{N}(0,1)$.
\IF{$m \ge m_0$}
\STATE Compute the bootstrap statistics $\boldsymbol{w}^{\sharp}$ using the $\mathsf{k\text{-}grad}$ method defined in equation~\eqref{k-grad} with $\xi_1,\ldots,\xi_{m}$ and $\mathfrak{g}_1, \ldots, \mathfrak{g}_m$.
\ELSE
\STATE Compute the bootstrap statistics $\boldsymbol{w}^{\flat}$ using the  $\mathsf{(n + k - 1)\text{-}grad}$  method defined in equation~\eqref{n-k-1-grad} with $\xi_1,\ldots,\xi_{n+m-1}$ and $\mathfrak{g}_{11} ,\ldots, \mathfrak{g}_{1n}, \mathfrak{g}_2, \ldots, \mathfrak{g}_m$.
\ENDIF
\STATE{
Apply Algorithm \ref{alg:NHT} with  \( (1, \varepsilon,\delta, B_3) \)
to $\boldsymbol w^\sharp$ or $\boldsymbol w^\flat$,
yielding the bootstrapped estimate \(\boldsymbol w^{boot}\):
$$
\boldsymbol{w}^{boot} = \mathsf{NoisyHT}\left(\boldsymbol{w}^\sharp \text{ or }\boldsymbol{w}^\flat, 1, \varepsilon, \delta, B_3\right). 
$$
}
\STATE{ 
\textbf{Repeat steps 8--14 for $n_B$ bootstrap replicates}. Calculate corresponding $\alpha/2$ and $(1-\alpha/2)$ quantiles $\mathrm{c}_{M^\prime}(\alpha)$ for $ \|\boldsymbol{w}^{boot}\|_1$.
}
\STATE{
The ensuing bootstrap confidence intervals for $\beta_j^*$ $(j \in \{0,1, \ldots, p\})$ are given by
$$
\widetilde{\beta}_j = \widehat{\beta}_{T_0,j}  + \widehat{\boldsymbol{w}}_j^{\top} \frac{1}{m}\sum_{k=1}^m \boldsymbol{g}_k + E_j, \quad  E_j \sim \mathcal{N}(0, \tfrac{ B_2^2 \log(1.25/\delta)}{n^2m^2\varepsilon^2} ),
$$
\begin{equation}\label{interval:bootstrap}
\begin{aligned}
\mathrm{CI}_j^{boot}(\alpha) = &\left[\widetilde{\beta}_j-\frac{\mathrm{c}_{M^\prime}(1-\alpha / 2)}{\sqrt{N}}, \widetilde{\beta}_j - \frac{ \mathrm{c}_{M^\prime}(\alpha / 2)}{\sqrt{N}}\right].
\end{aligned}
\end{equation}
}
\STATE {\bf Output:}
Return $\mathrm{CI}_0^{boot}(\alpha), \ldots, \mathrm{CI}_p^{boot}(\alpha)$.
\end{algorithmic}
\end{algorithm*}

Algorithm \ref{alg:DPbootstrap} implements a differentially private and distributed bootstrap procedure for simultaneous inference. The procedure operates as follows: First, the debiased estimator $\widetilde{\boldsymbol{\beta}}$ is broadcast to all $m$ machines. Each machine computes its local gradient $\mathfrak{g}_k$ and sends it to the central server. The central server aggregates these gradients to form the average $\bar{\mathfrak{g}}$, generates i.i.d.~$\mathcal{N}(0,1)$ multipliers $\xi_1,\dots,\xi_{n+m-1}$, and computes either the $\mathsf{k\text{-}grad}$ statistics \eqref{k-grad} or the $\mathsf{(n + k - 1)\text{-}grad}$ statistics \eqref{n-k-1-grad}, depending on the number of machines. The $\mathsf{NoisyHT}$ algorithm (Algorithm \ref{alg:NHT}) is then applied with sparsity set to 1, ensuring that the resulting vector $\boldsymbol{w}^{boot} = \mathsf{NoisyHT}\left(\boldsymbol{w}^\sharp \text{ or }\boldsymbol{w}^\flat, 1, \varepsilon, \delta, B_3\right)$ has exactly one nonzero coordinate. Thus, $\|\boldsymbol{w}^{boot}\|_\infty = \|\boldsymbol{w}^{boot}\|_1$. For each coordinate $j$, the empirical $\alpha/2$ and $(1-\alpha/2)$ quantiles of $\|\boldsymbol{w}^{boot}\|_1$ are computed across bootstrap replicates, forming the two-sided bootstrap confidence interval as in \eqref{interval:bootstrap}. The following Theorem~\ref{thm:bootstrap} establishes the statistical validity and differential privacy guarantees of Algorithm~\ref{alg:DPbootstrap}.
\begin{theorem}\label{thm:bootstrap}
Suppose that the conditions of Theorems~\ref{thm1} and~\ref{thm3} hold. If
\[
\frac{\log^7 p\,\log(1/\delta)\,\log^3 N}{m N \epsilon^2} \to 0,
\]
and either $\log p / m \to 0$ for the $\mathsf{k\text{-}grad}$ statistic, or $\log p / (n - m + 1) \to 0$ for the $\mathsf{(n + k - 1)\text{-}grad}$ statistic, then for all $\alpha \in (0,1)$, we have
\[
\sup_{\alpha \in (0,1)} \left| \mathbb{P}\left(\sqrt{N} \|\widetilde{\boldsymbol{\beta}} - \boldsymbol{\beta}^*\|_\infty \le \mathrm{c}_{M^{\prime}}(\alpha)\right) - \alpha \right| = o_{\mathbb{P}}(1),
\]
where
\[
\mathrm{c}_{M^{\prime}}(\alpha) = \inf \left\{t \in \mathbb{R} : \mathbb{P}^*\left(\|\boldsymbol{w}^{boot}\|_1 \le t\right) \ge \alpha \right\},
\]
and $\mathbb{P}^*(\cdot) = \mathbb{P}(\cdot \mid \mathcal{Z}^N)$ denotes the conditional probability given the observed data $\mathcal{Z}^N$.

In addition, Algorithm~\ref{alg:DPbootstrap} satisfies $(\epsilon, \delta)$-differential privacy.
\end{theorem}

Theorem \ref{thm:bootstrap} establishes that, under the regularity conditions of Theorems~\ref{thm1} and~\ref{thm3}, the simultaneous confidence regions constructed by Algorithm~\ref{alg:DPbootstrap} using the bootstrap methodology achieve asymptotically exact coverage. Specifically, we rigorously show that the bootstrap quantile $\mathrm{c}_{M^{\prime}}(\alpha)$ consistently approximates the ideal quantile $\mathrm{c}(\alpha) = \inf \big\{t \in \mathbb{R}: \mathbb{P} (\|\sqrt{N} (\widetilde{\boldsymbol{\beta}} - \boldsymbol{\beta}^*) \|_\infty \leq t) $$\geq \alpha\big\}$, thereby demonstrating the statistical validity and efficiency of the proposed approach. Moreover, the algorithm guarantees $(\epsilon,\delta)$-differential privacy.

\section{Simulation Experiments}\label{sec:simulation}
This section presents comprehensive simulation studies to evaluate the performance of the proposed differentially private distributed high-dimensional quantile regression algorithms. Data are generated from two types of linear models: one with homoscedastic errors (Model 1) and one with heteroscedastic errors (Model 2):
\begin{itemize}
    \item Model 1: $Y_i=\boldsymbol{X}_i^{\top} \boldsymbol{\beta}^{*}+\varepsilon_i$,
    \item Model 2: $Y_i=\boldsymbol{X}_i^{\top} \boldsymbol{\beta}^{*}+(1+0.4 x_{i1})\varepsilon_i$,
\end{itemize}
where $\boldsymbol{X}_i = (1, x_{i1}, \ldots, x_{ip})^{\top}$ denotes a $(p+1)$-dimensional covariate vector. The covariates $(x_{i1}, \ldots, x_{ip})^{\top}$ are independently drawn from a multivariate normal distribution $\mathcal{N}(\boldsymbol{0},\boldsymbol{\Sigma})$, where the covariance matrix is specified as $\boldsymbol{\Sigma}_{ij} = 0.5^{|i-j|}$ for $1 \leq i, j \leq p$. The true parameter vector is set as $\boldsymbol{\beta}^{*} = (1,1,2,3,4,5,\boldsymbol{0}_{p-5})^\top$ with dimension $p=500$. The global bandwidth is fixed at $h_{opt} = 0.5 \cdot (\log p/N)^{1/3}$ and the quantile level is $\tau = 0.5$. For differential privacy, we fix $\delta = 1/N$ and vary $\epsilon$, where smaller $\epsilon$ values correspond to stronger privacy protection.

We consider the following three types of noise distributions for $\varepsilon_i$:
\begin{enumerate}
    \item Normal distribution: $\varepsilon_i \sim \mathcal{N}(0, 1)$;
    \item Student's $t$ distribution with 3 degrees of freedom: $\varepsilon_i \sim t(3)$;
    \item Cauchy distribution: $\varepsilon_i \sim \mathrm{Cauchy}(0,1)$.
\end{enumerate}
Initial values $\widehat{\boldsymbol{\beta}}_0$ and $\mathbf{W}_0$ are computed on the central machine $\mathcal{M}_1$ using the local data only. Unless otherwise specified, we use $K=T=10$ and report averages over 100 simulation runs, with standard deviations shown in parentheses. The values of $N$, $n$, $m$, and $\epsilon$ that change across experiments are stated in the surrounding text rather than repeated in the captions.
\subsection{Estimation Simulation}
\begin{figure}[t]
    \centering
    \begin{minipage}{\textwidth}
        \centering
        {\bfseries The case of homoscedastic errors} \\
        \includegraphics[width=0.32\textwidth]{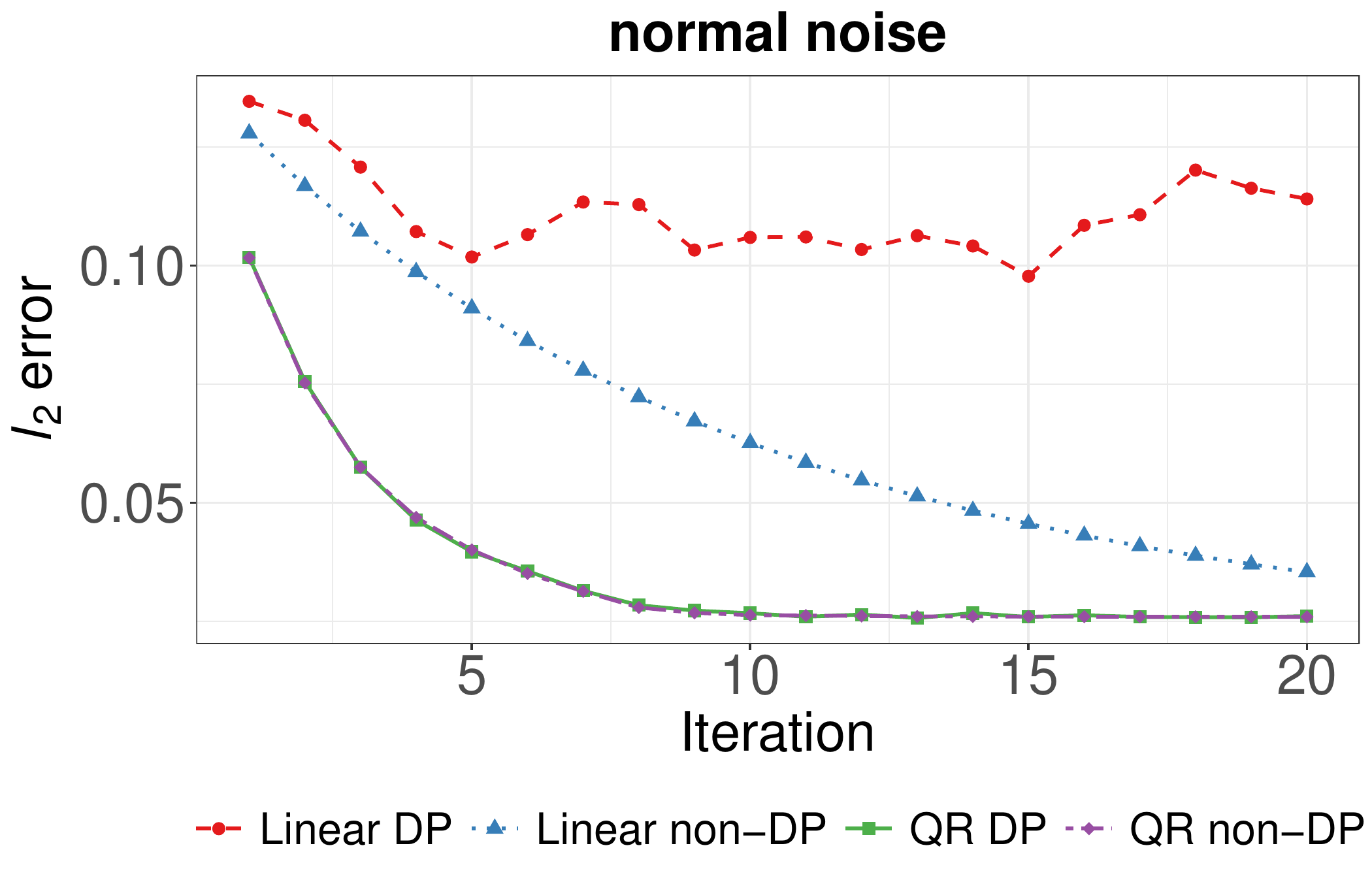}
        \includegraphics[width=0.32\textwidth]{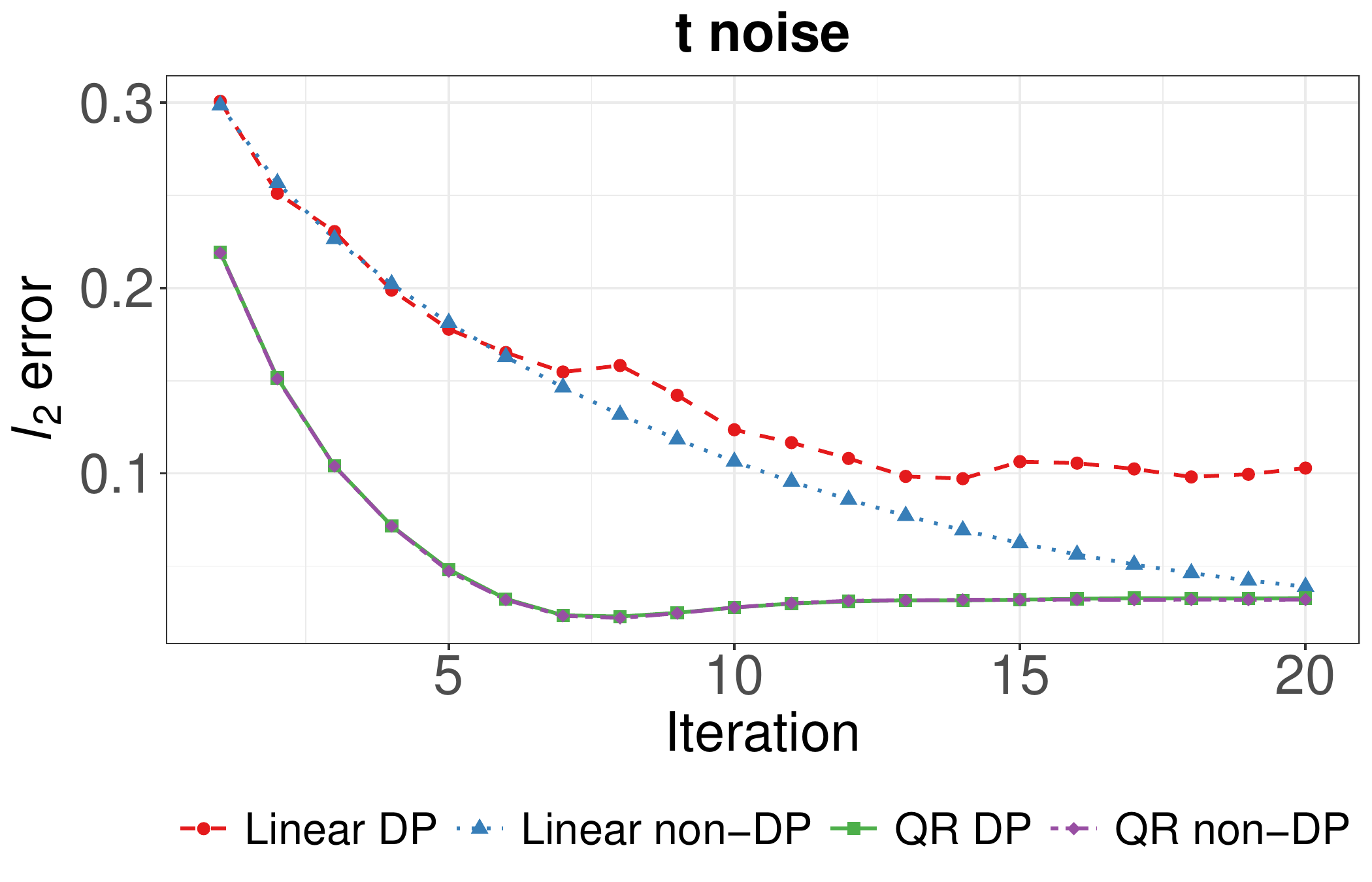}
        \includegraphics[width=0.32\textwidth]{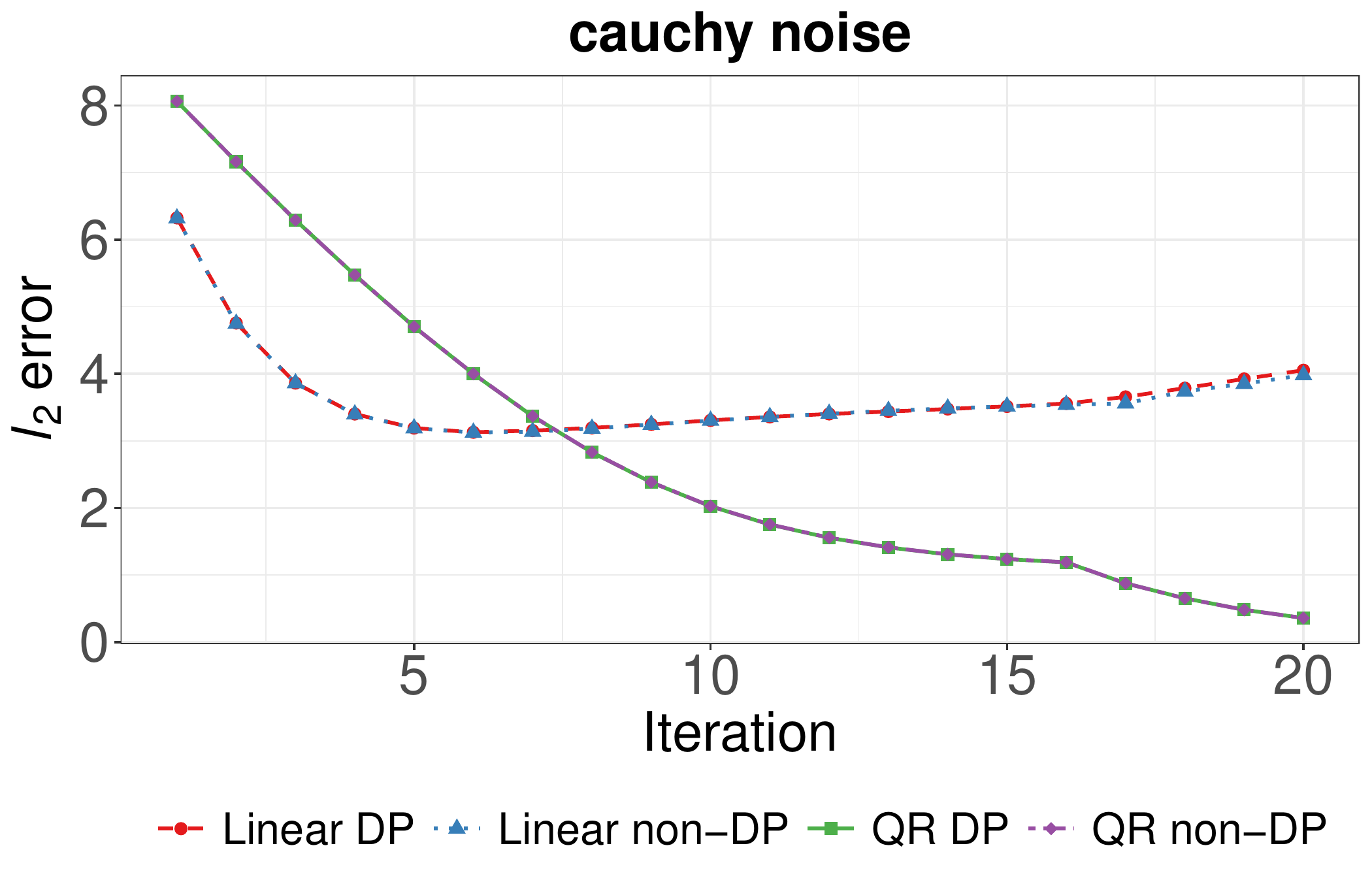}
    \end{minipage}
    \hfill
    \begin{minipage}{\textwidth}
        \centering
        {\bfseries The case of heteroscedastic errors} \\
        \includegraphics[width=0.32\textwidth]{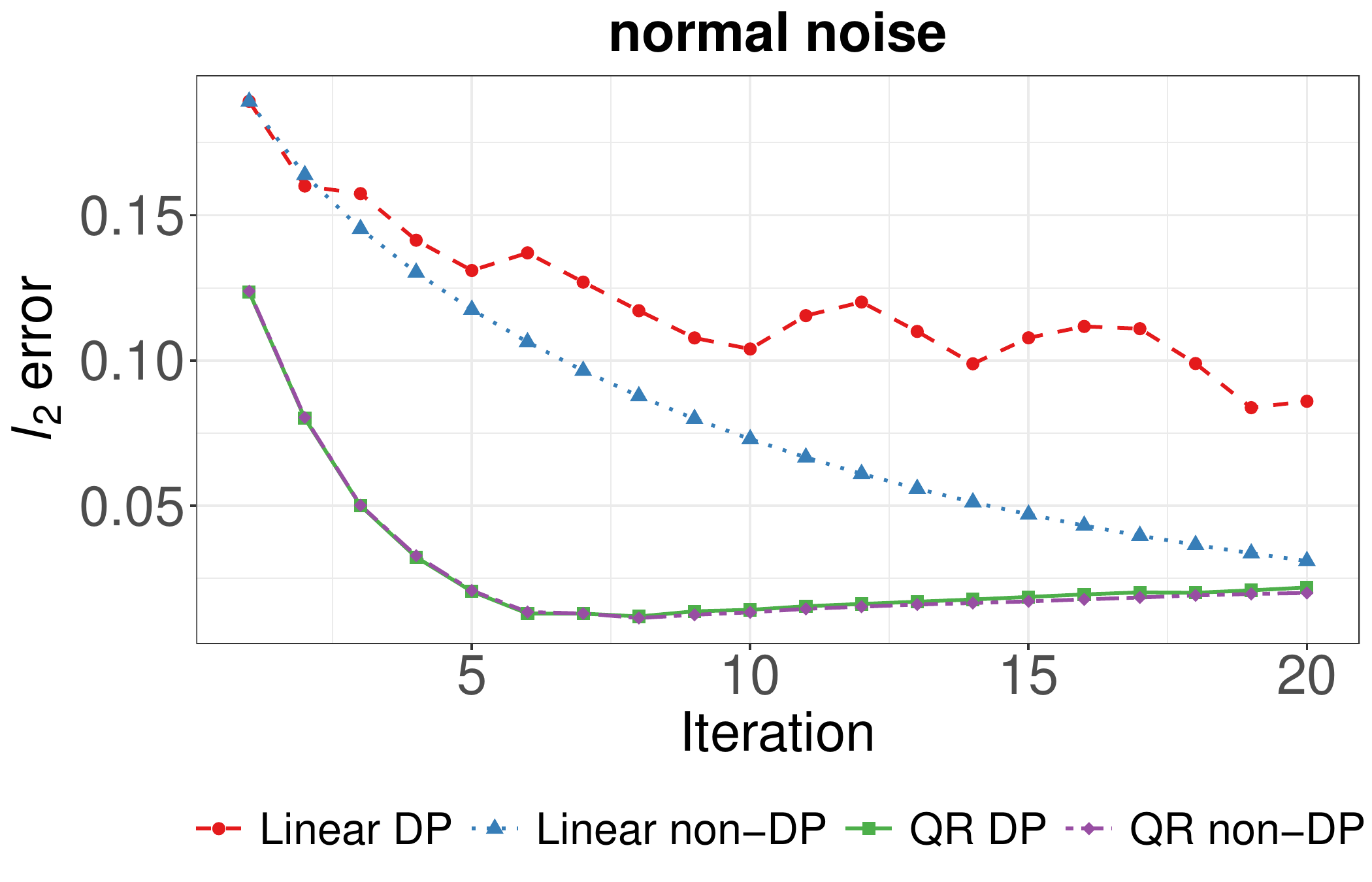}
        \includegraphics[width=0.32\textwidth]{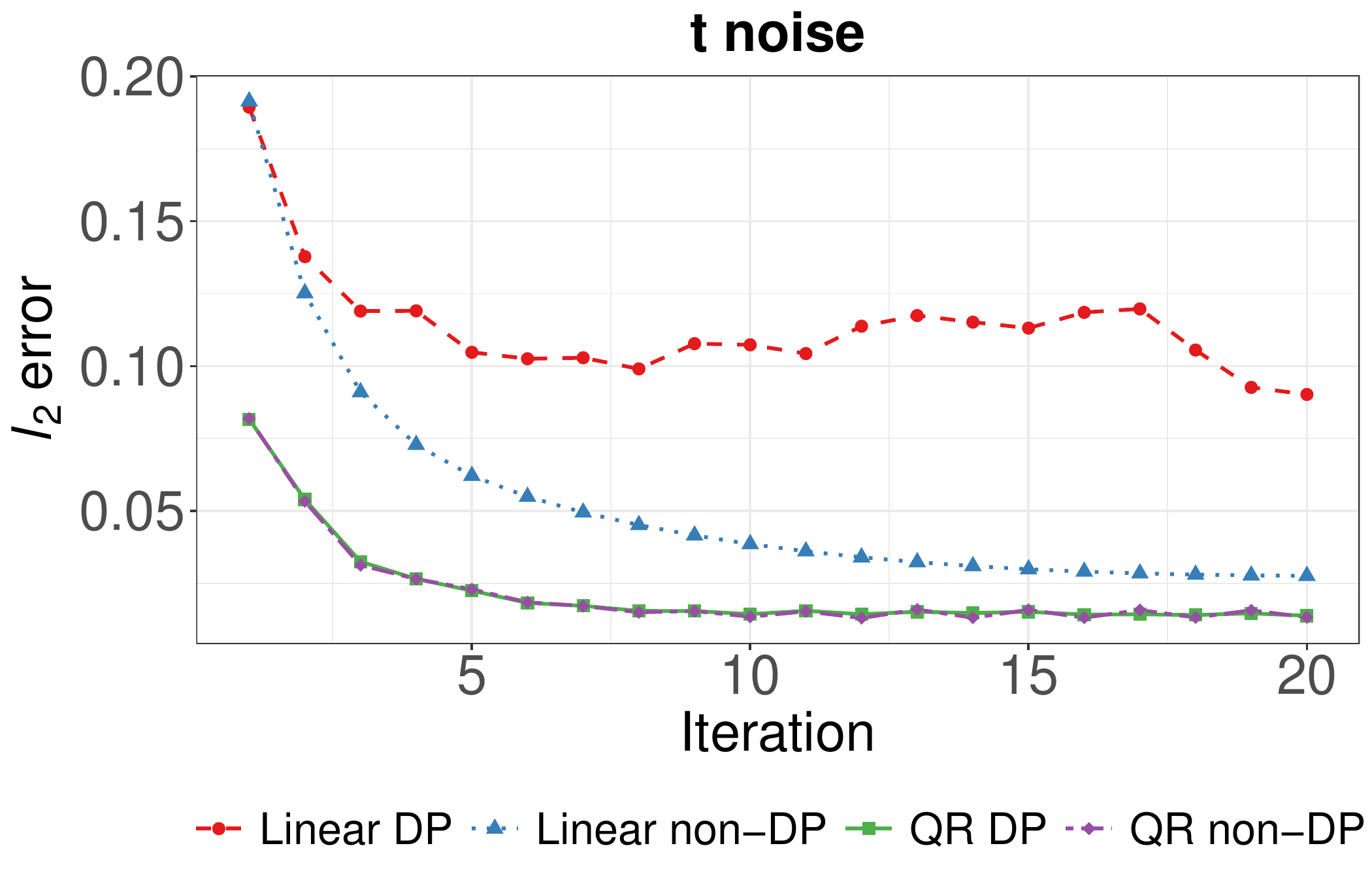}
        \includegraphics[width=0.32\textwidth]{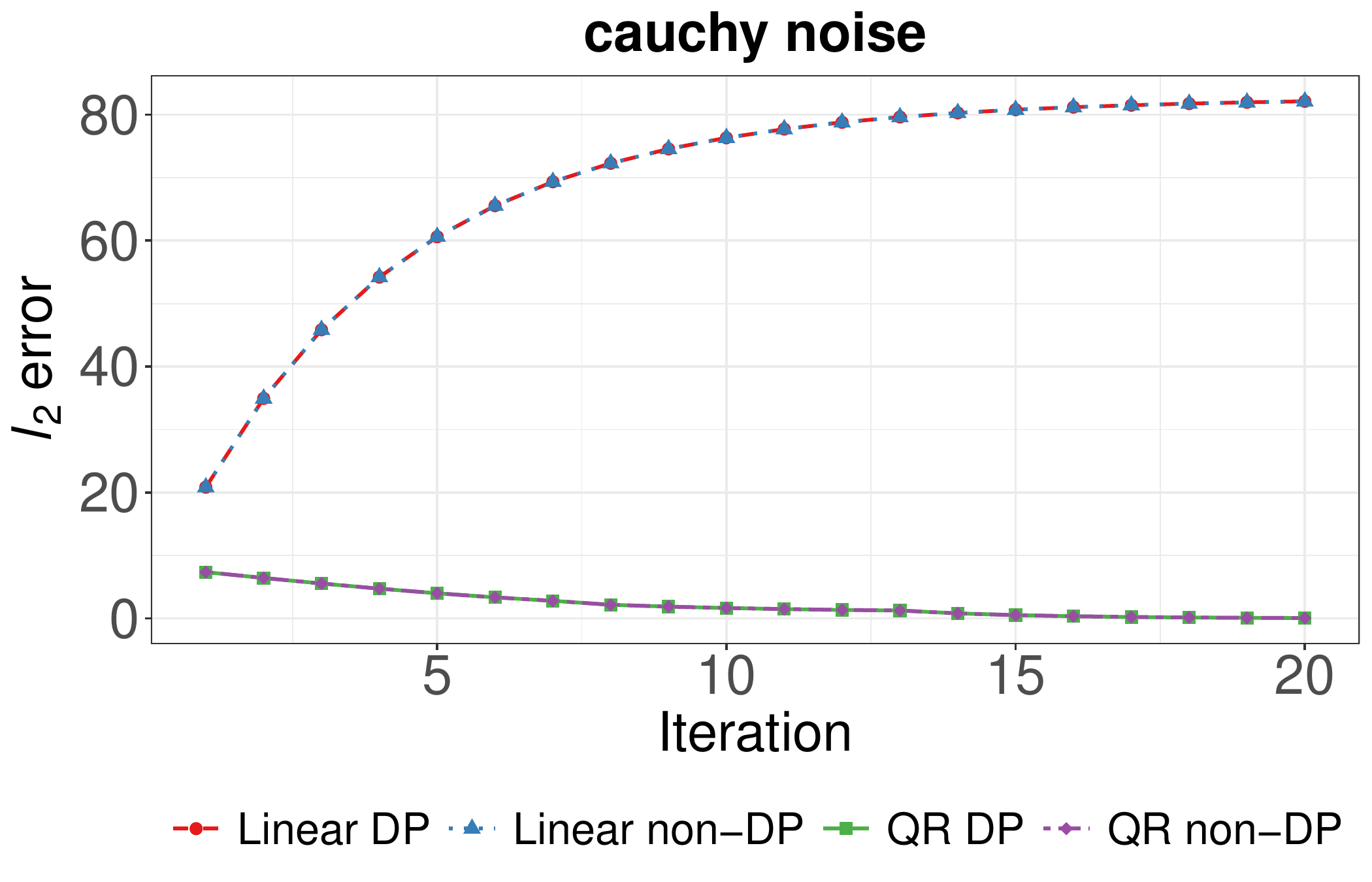}
    \end{minipage}

    \caption{Comparison of $\ell_2$ estimation error over iterations for QR and linear-regression baselines.}
    \label{fig:estimation_iteration}
\end{figure}
Figure~\ref{fig:estimation_iteration} compares our QR method ($T=10, K=10$) with the distributed linear regression approach of \cite{zhang2024differentially}. Each iteration $t$ consists of 10 gradient descent steps. Results are shown for three noise types. In all scenarios, the QR estimator achieves lower $\ell_2$ error than the linear regression baseline, regardless of whether DP is enforced.
Under Cauchy noise, the linear regression estimators fail to converge due to the infinite moments of the Cauchy distribution. Both DP and non-DP versions exhibit diverging errors, highlighting their unreliability in heavy-tailed settings. In contrast, the QR estimators remain stable and continue to converge. A consistent pattern emerges regarding privacy: DP versions of all methods exhibit larger long-term errors than their non-private counterparts. This observation aligns with the theoretical results in Section~\ref{sec:DP-estimation}, where the additional error from DP is characterized as the inherent cost of privacy protection.

\begin{figure}[t]
    \centering
    \begin{minipage}[c]{1\textwidth}
    \centering    
    {\bfseries The case of homoscedastic errors} \\
    \includegraphics[width=0.32\textwidth]{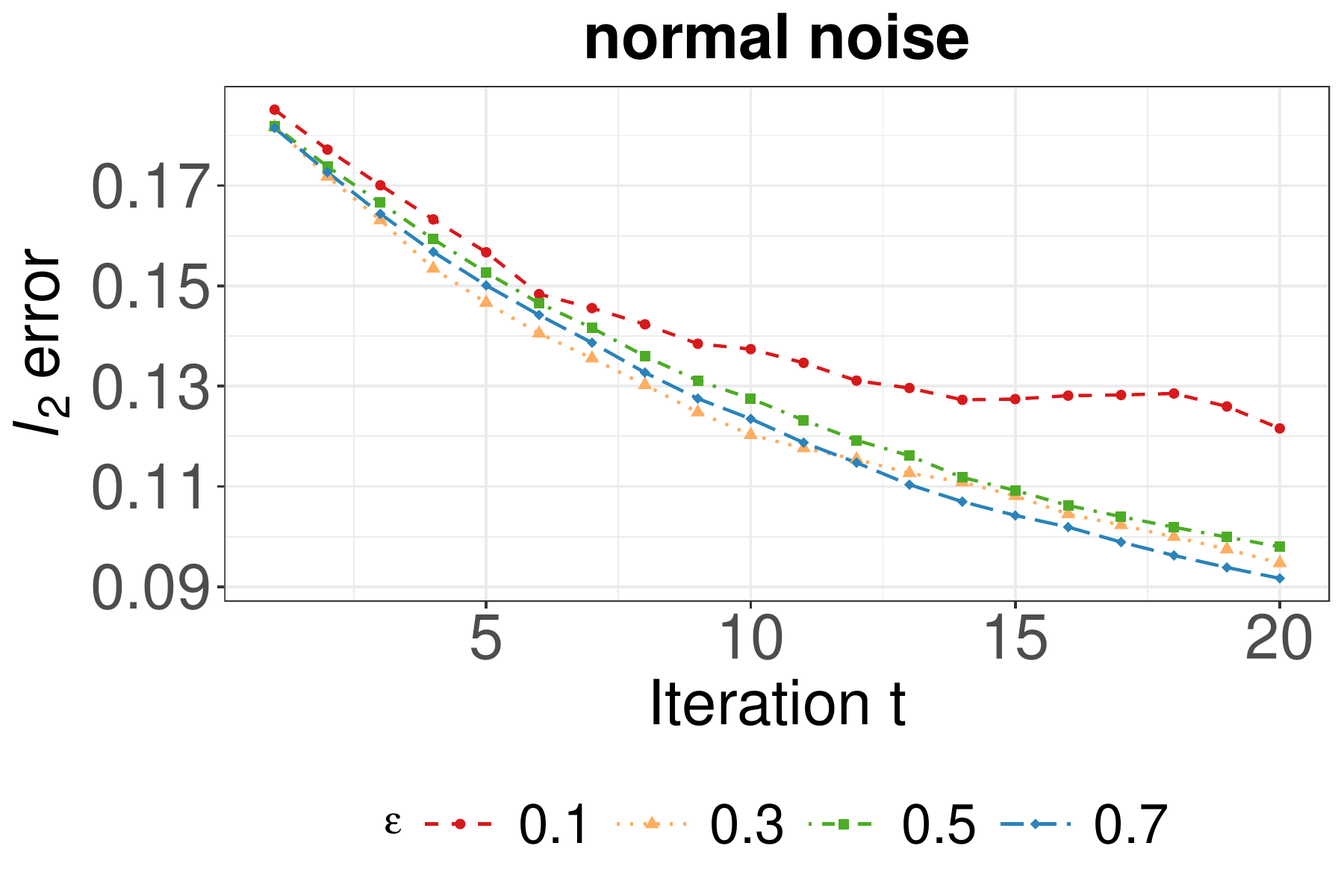}
    \includegraphics[width=0.32\textwidth]{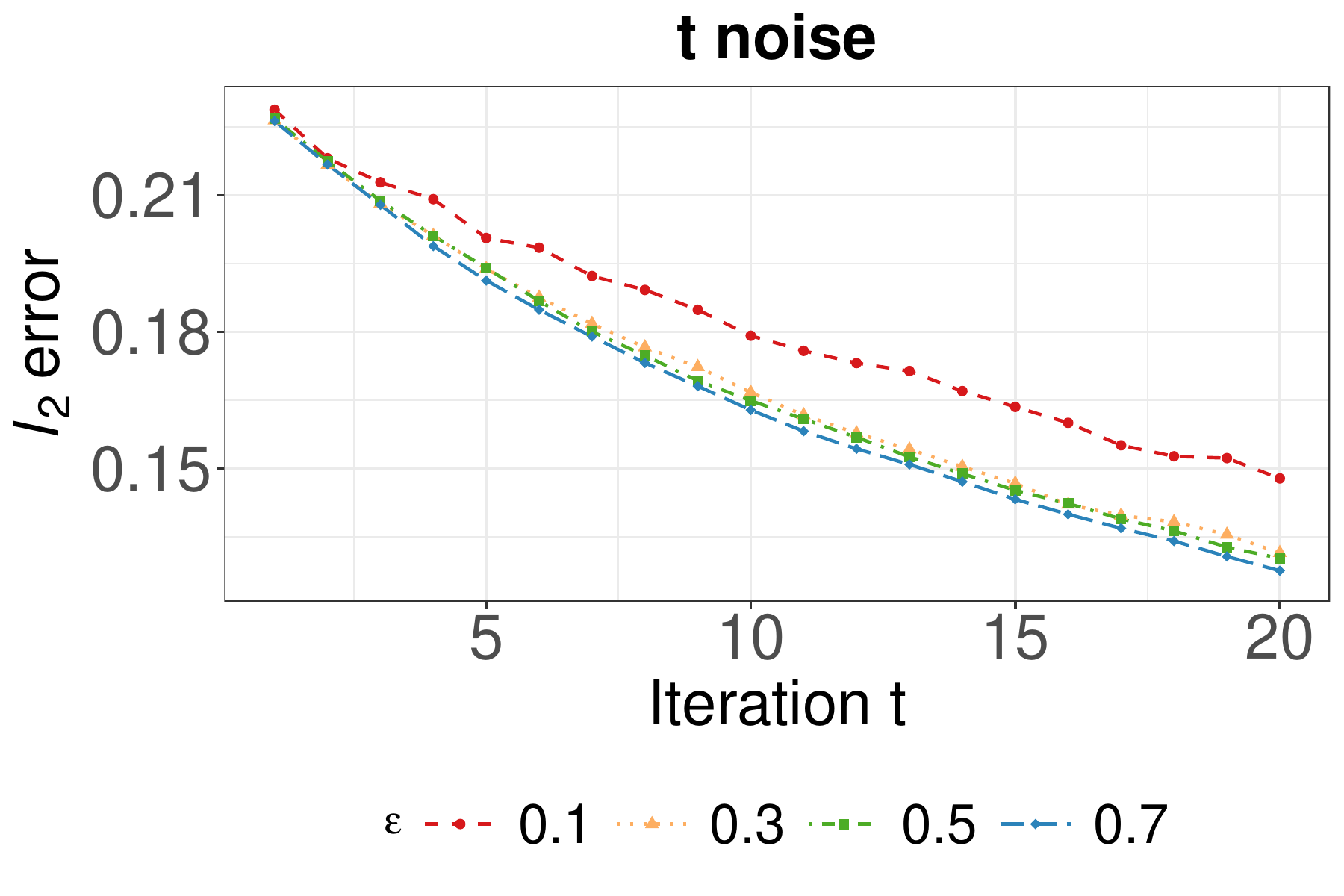}
    \includegraphics[width=0.32\textwidth]{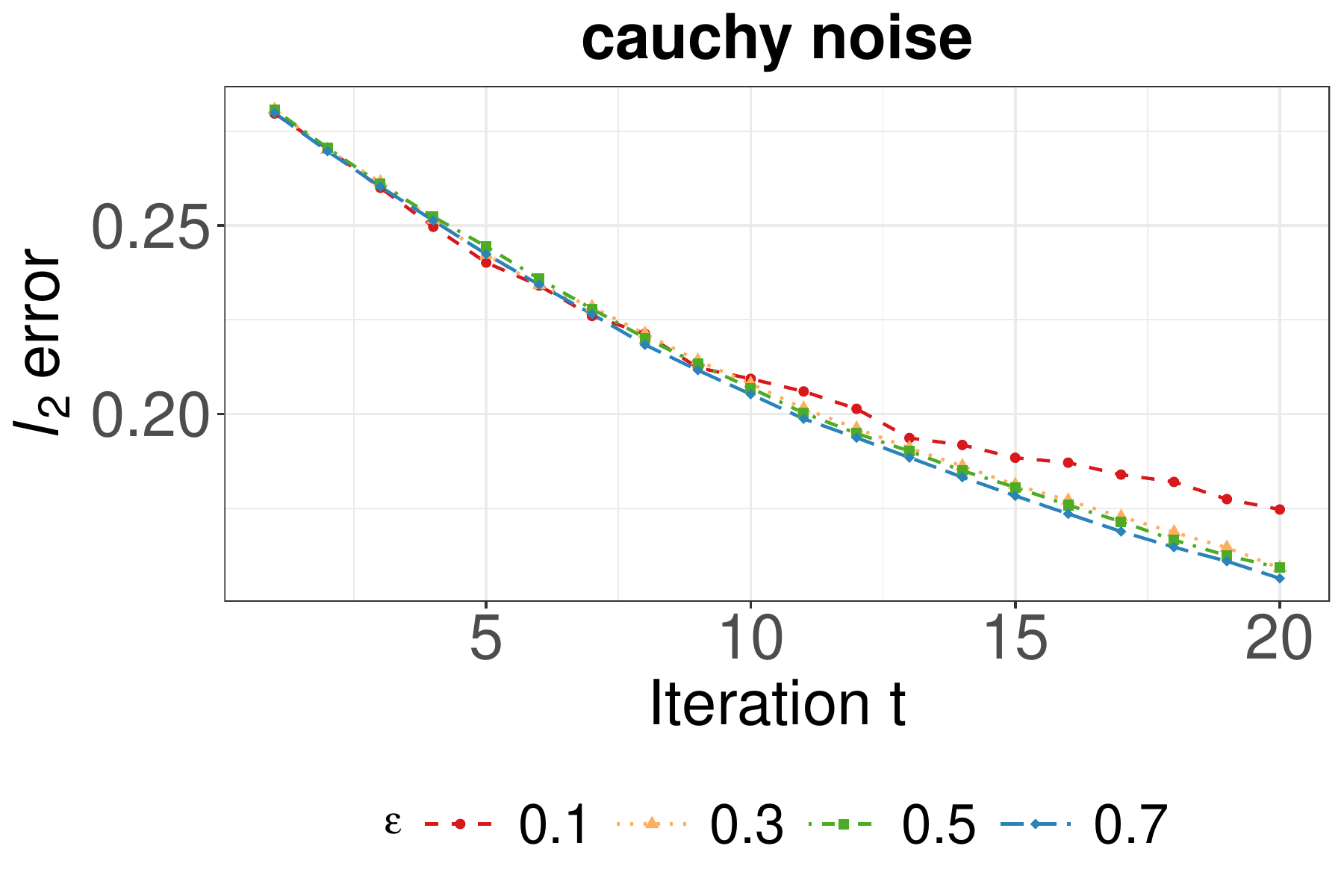}
    \end{minipage}
    \hfill
    \begin{minipage}[c]{1\textwidth}
    \centering
    {\bfseries The case of heteroscedastic errors} \\
    \includegraphics[width=0.32\textwidth]{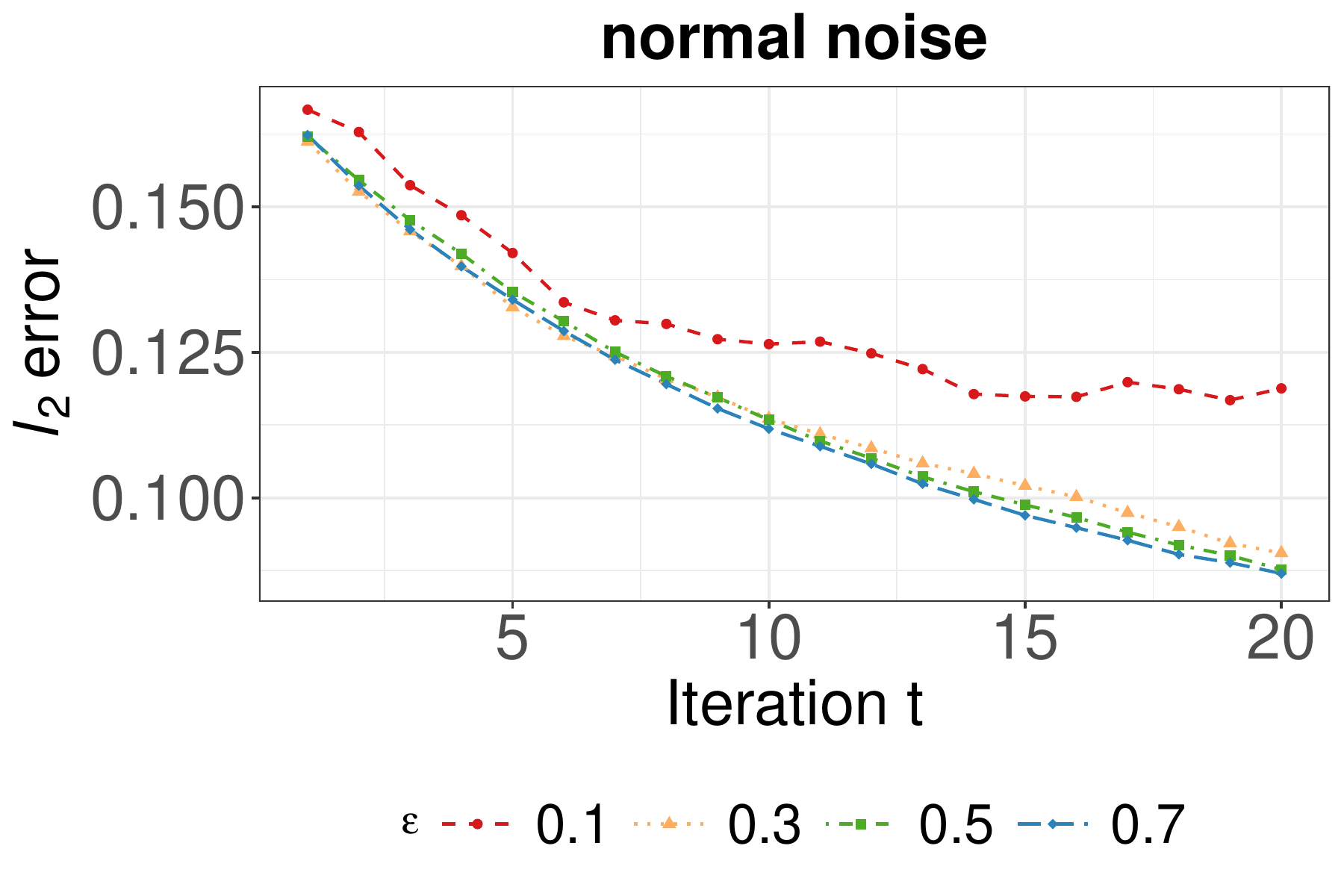}
    \includegraphics[width=0.32\textwidth]{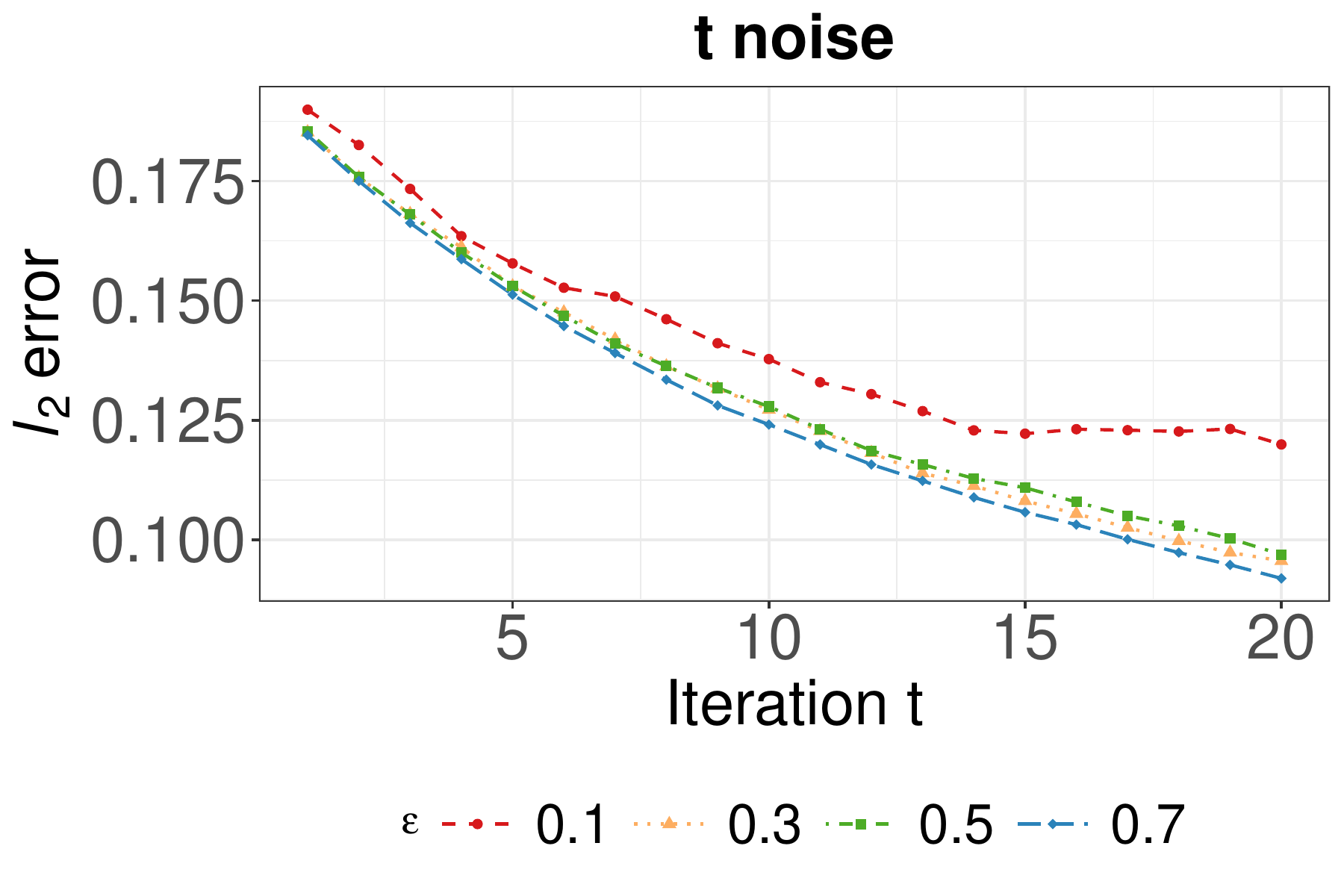}
    \includegraphics[width=0.32\textwidth]{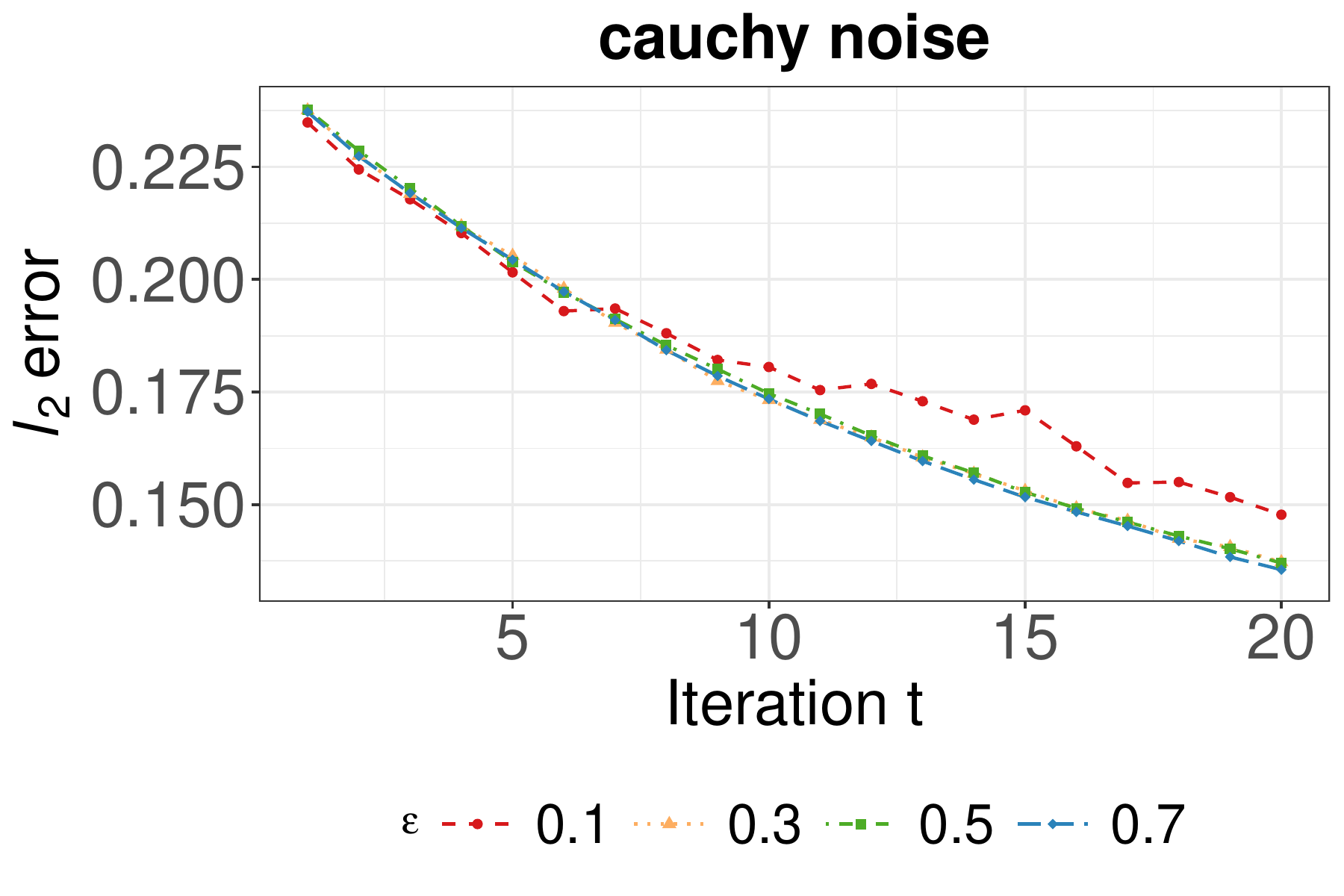}
    \end{minipage}
    \caption{Effect of the privacy budget on the $\ell_2$ estimation error across noise distributions.}
    \label{fig:privacy_budget_error}
\end{figure}

To better understand how varying levels of privacy protection influence convergence, we examine the estimator's behavior as the privacy budget $\epsilon$ changes. Figure~\ref{fig:privacy_budget_error} uses the same model, noise, and iteration settings as the preceding estimation experiment; the line styles and methods are identified in the legend. Smaller values of $\epsilon$ consistently lead to larger long-term estimation errors, matching the privacy term in Theorem~\ref{thm2} and illustrating the privacy--accuracy trade-off.

For the tabulated estimation experiments, Tables~\ref{tab:estimation_N_hom}--\ref{tab:estimation_N_het} fix the local sample size at $n=500$ and vary the total sample size $N$, whereas Tables~\ref{tab:estimation_n_hom}--\ref{tab:estimation_n_het} fix $N=20000$ and vary the local sample size $n$.

\begin{table}[H]
    \scriptsize
    \centering
    \caption{Estimation error as the total sample size varies under homoscedastic errors.}
    \resizebox{\textwidth}{!}{
    \begin{tabular}{c|ccc|ccc|ccc}
    \hline
    \multicolumn{1}{c|}{Noise}                                           
    & \multicolumn{3}{c|}{Normal}                   
    & \multicolumn{3}{c|}{$t(3)$}                   
    & \multicolumn{3}{c}{Cauchy}
    \\ \hline
    \multicolumn{1}{c|}{$N$}                                           
    & 5000         & 10000        & 20000        
    & 5000         & 10000        & 20000        
    & 5000         & 10000        & 20000
    \\ \hline
    $\epsilon = 0.1$ 
    & 0.389(0.183) & 0.344(0.120) & 0.419(0.208)
    & 0.407(0.201) & 0.421(0.221) & 0.447(0.258) 
    & 0.451(0.243) & 0.434(0.208) & 0.520(0.282) \\
    $\epsilon = 0.2$ 
    & 0.193(0.063) & 0.172(0.059) & 0.191(0.071)
    & 0.194(0.057) & 0.186(0.056) & 0.187(0.075) 
    & 0.210(0.059) & 0.206(0.079) & 0.208(0.069) \\
    $\epsilon = 0.5$ 
    & 0.086(0.031)  &0.086(0.033)  &0.074(0.023)   
    & 0.091(0.032)  &0.085(0.024)  &0.083(0.038)   
    & 0.107(0.037)  &0.102(0.037)  &0.102(0.035)   \\
    $\epsilon = 0.7$ 
    &0.071(0.027)  &0.064(0.022)  &0.065(0.026)   
    &0.083(0.029)  &0.070(0.023)  &0.070(0.023)   
    &0.089(0.035)  &0.081(0.025)  &0.077(0.026)   \\
    $\epsilon = 1$ 
    & 0.069(0.002)  &0.053(0.017) &0.049(0.016)  
    & 0.071(0.025)  &0.061(0.025) &0.053(0.020)  
    & 0.079(0.031)  &0.073(0.027) &0.066(0.027) \\\hline
    \end{tabular}
    }
    \label{tab:estimation_N_hom}
\end{table}

\begin{table}[H]
    \scriptsize
    \centering
    \caption{Estimation error as the total sample size varies under heteroscedastic errors.}
    \resizebox{\textwidth}{!}{
    \begin{tabular}{c|ccc|ccc|ccc}
    \hline
    \multicolumn{1}{c|}{Noise}                                           
    & \multicolumn{3}{c|}{Normal}                   
    & \multicolumn{3}{c|}{$t(3)$}                   
    & \multicolumn{3}{c}{Cauchy}                   
    \\ \hline
    \multicolumn{1}{c|}{$N$}                                           
    & 5000         & 10000        & 20000        
    & 5000         & 10000        & 20000        
    & 5000         & 10000        & 20000        
    \\ \hline
    $\epsilon = 0.1$ 
    & 0.333(0.127) & 0.361(0.171) & 0.341(0.167)
    & 0.407(0.268) & 0.407(0.268) & 0.447(0.259)
    & 0.441(0.283) & 0.431(0.290) & 0.468(0.260) \\
    $\epsilon = 0.2$ 
    & 0.166(0.065) & 0.107(0.062) & 0.177(0.071)
    & 0.170(0.060) & 0.164(0.062) & 0.160(0.058) 
    & 0.160(0.055) & 0.168(0.065) & 0.184(0.069) \\
    $\epsilon = 0.5$ 
    &0.078(0.027)   &0.072(0.025) &0.073(0.024)   
    &0.079(0.030)   &0.072(0.028) &0.072(0.021)    
    &0.091(0.038)   &0.085(0.031) &0.075(0.024)  \\
    $\epsilon = 0.7$ 
    &0.059(0.020)   &0.052(0.016) &0.050(0.021)   
    &0.063(0.021)   &0.057(0.021) &0.056(0.018)   
    &0.074(0.029)   &0.063(0.024) &0.066(0.024)   \\
    $\epsilon = 1$ 
    &0.049(0.017) & 0.044(0.015) & 0.043(0.016)  
    &0.059(0.018) & 0.045(0.018) & 0.043(0.013)  
    &0.065(0.027) & 0.057(0.018) & 0.057(0.025) 
    \\ \hline
    \end{tabular}
    }
    \label{tab:estimation_N_het}
\end{table}

\begin{table}[H]
    \scriptsize
    \centering
    \caption{Estimation error as the local sample size varies under homoscedastic errors.}
    \resizebox{\textwidth}{!}{
    \begin{tabular}{c|ccc|ccc|ccc}
    \hline
    \multicolumn{1}{c|}{Noise}
    & \multicolumn{3}{c|}{Normal}                   
    & \multicolumn{3}{c|}{$t(3)$}                   
    & \multicolumn{3}{c}{Cauchy}                   
    \\ \hline
    \multicolumn{1}{c|}{$n$}                                           
    & 500         & 1000        & 2000       
    & 500         & 1000        & 2000        
    & 500         & 1000        & 2000        \\ \hline
    $\epsilon = 0.1$ 
    & 0.475(0.326) & 0.193(0.084) & 0.094(0.031)
    & 0.537(0.293) & 0.178(0.068) & 0.105(0.030)
    & 0.445(0.210) & 0.219(0.073) & 0.096(0.034) \\
    $\epsilon = 0.2$ 
    & 0.192(0.053) & 0.085(0.031) & 0.047(0.014)
    & 0.187(0.061) & 0.113(0.037) & 0.046(0.014)
    & 0.195(0.065) & 0.107(0.045) & 0.061(0.020) \\
    $\epsilon = 0.5$ 
    & 0.072(0.023) & 0.045(0.019) & 0.030(0.011)
    & 0.081(0.042) & 0.051(0.018) & 0.033(0.015)
    & 0.093(0.028) & 0.055(0.015) & 0.044(0.017) \\
    $\epsilon = 0.7$ 
    & 0.054(0.017) & 0.042(0.015) & 0.028(0.012)
    & 0.065(0.022) & 0.046(0.016) & 0.030(0.010)
    & 0.088(0.028) & 0.049(0.016) & 0.037(0.012) \\
    $\epsilon = 1$ 
    & 0.048(0.019) & 0.039(0.014) & 0.026(0.007)
    & 0.050(0.022) & 0.040(0.014) & 0.030(0.009)
    & 0.076(0.028) & 0.048(0.016) & 0.035(0.011) 
    \\ \hline
    \end{tabular}
    }
    \label{tab:estimation_n_hom}
\end{table}

\begin{table}[H]
    \scriptsize
    \centering
    \caption{Estimation error as the local sample size varies under heteroscedastic errors.}
    \resizebox{\textwidth}{!}{
    \begin{tabular}{c|ccc|ccc|ccc}
    \hline
    \multicolumn{1}{c|}{Noise}                                           
    & \multicolumn{3}{c|}{Normal}                   
    & \multicolumn{3}{c|}{$t(3)$}                   
    & \multicolumn{3}{c}{Cauchy}                   
    \\ \hline
    \multicolumn{1}{c|}{$n$}                                           
    & 500         & 1000        & 2000        
    & 500         & 1000        & 2000        
    & 500         & 1000        & 2000        \\ \hline
    $\epsilon = 0.1$ 
    & 0.372(0.192) & 0.177(0.069) & 0.075(0.028)
    & 0.388(0.188) & 0.164(0.047) & 0.087(0.039)
    & 0.471(0.233) & 0.187(0.064) & 0.101(0.046) \\
    $\epsilon = 0.2$ 
    & 0.166(0.055) & 0.077(0.031) & 0.040(0.011)
    & 0.165(0.057) & 0.082(0.035) & 0.048(0.016)
    & 0.198(0.058) & 0.102(0.070) & 0.054(0.019) \\
    $\epsilon = 0.5$ 
    & 0.057(0.015) & 0.036(0.013) & 0.026(0.010)
    & 0.083(0.030) & 0.036(0.012) & 0.029(0.008)
    & 0.079(0.028) & 0.044(0.020) & 0.031(0.009) \\
    $\epsilon = 0.7$ 
    & 0.055(0.017) & 0.032(0.012) & 0.022(0.010)
    & 0.050(0.020) & 0.034(0.068) & 0.026(0.010)
    & 0.060(0.012) & 0.038(0.014) & 0.030(0.010) \\
    $\epsilon = 1$ 
    & 0.040(0.011) & 0.024(0.006) & 0.020(0.005)
    & 0.043(0.016) & 0.028(0.008) & 0.023(0.008)
    & 0.055(0.016) & 0.034(0.011) & 0.026(0.008) 
    \\ \hline
    \end{tabular}
    }
    \label{tab:estimation_n_het}
\end{table}

Tables~\ref{tab:estimation_N_hom}--\ref{tab:estimation_n_het} comprehensively quantify the privacy-accuracy-efficiency trade-offs under fixed computational budgets ($K = T = 10$) and varying privacy budgets $\epsilon$. Three principal empirical patterns emerge from these analyses. 

First, across all noise distributions and data models---including both homoscedastic (Model 1) and heteroscedastic (Model 2) error structures---strengthening privacy protection systematically increases the $\ell_2$ error, with mean error rising as $\epsilon$ decreases from 1.0 to 0.1. This monotonic relationship provides empirical validation for the DP error term in our theoretical framework.

Second, as shown in Tables~\ref{tab:estimation_N_hom}--\ref{tab:estimation_N_het}, the interaction between privacy constraints and sample efficiency reveals a fundamental dichotomy. Under relaxed privacy regimes (e.g., $\epsilon \geq 0.5$), the $\ell_2$ error decreases as the total sample size $N$ increases, indicating that the oracle convergence rate dominates when privacy costs are moderate. Conversely, under strict privacy (e.g., $\epsilon < 0.5$), error reduction plateaus despite increasing $N$, demonstrating that privacy-induced error becomes the limiting factor in highly constrained settings. This pattern persists even in the presence of heteroscedasticity.

Third, as shown in Tables~\ref{tab:estimation_n_hom}--\ref{tab:estimation_n_het}, increasing the local sample size $n$ while fixing the global sample size $N = 20000$ consistently reduces the $\ell_2$ error across all privacy levels. Since we use $K = T = 10$ and take the single-machine estimate as the starting point, this trend reflects the influence of the initial estimator, as captured by the third term in \eqref{bound:betaT}.

\subsection{Inference Simulation}
In this section, we conducted simulation studies to evaluate the inference performance of our proposed differentially private distributed quantile regression methods. Specifically, we focus on two key aspects: (1) assessing the normality of the standardized test statistics, and (2) evaluating the empirical coverage and width of simultaneous confidence intervals constructed via two bootstrap-based procedures. The data generation process and privacy parameter settings are the same as in the estimation experiments. We fix the number of outer iterations at \(T = 10\), set the local bandwidth to $b = 0.5 \cdot (\log p / n)^{1/3}$, and use the median quantile level \(\tau = 0.5\). For the normality diagnostics in Figures~\ref{fig:hom_combined_histograms_tau0.5} and~\ref{fig:het_combined_histograms_tau0.5}, we set $\epsilon=1$, $n=500$, and $N=5000$. For the bootstrap procedures, we use $N=20{,}000$, $n=500$, $n_B=2000$ replications, and significance level \(\alpha = 5\%\).

\setlength{\fboxrule}{0.8pt}  
\setlength{\fboxsep}{1.6pt}   
\begin{figure}[!t]
        \centering
        {\bfseries The case of homoscedastic errors: Histograms} \\
        \includegraphics[width=0.29\textwidth]{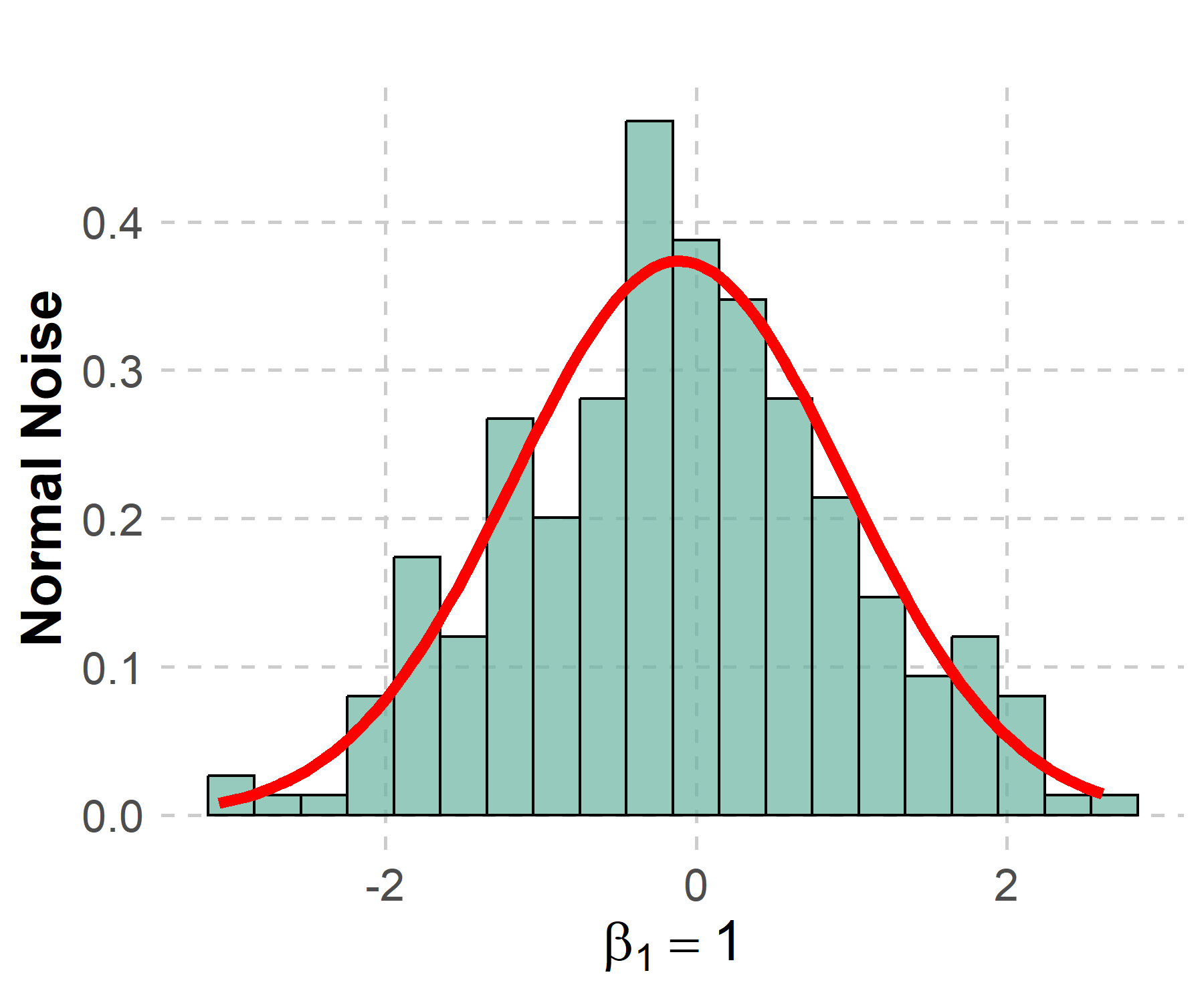}
        \includegraphics[width=0.29\textwidth]{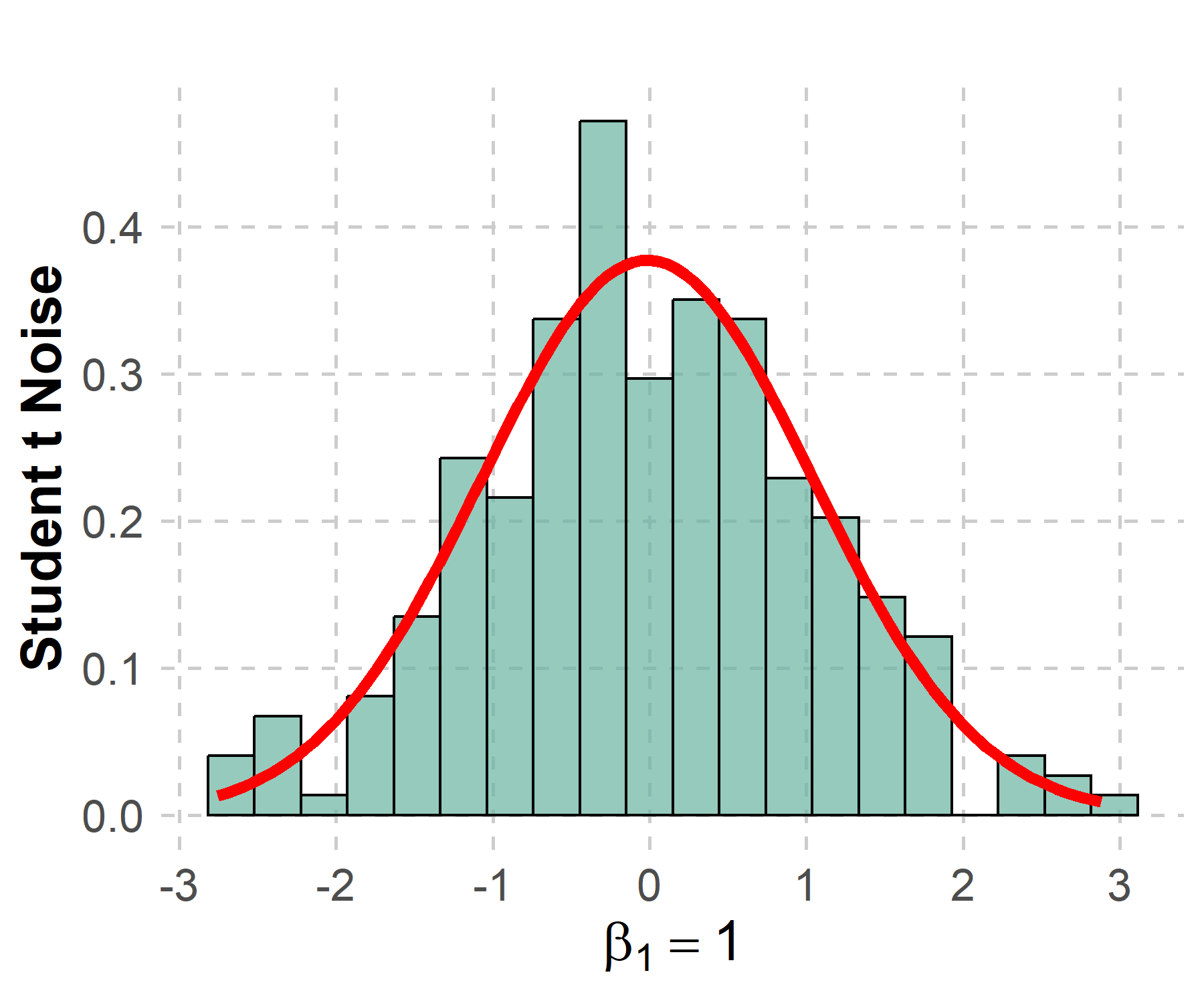}
        \includegraphics[width=0.29\textwidth]{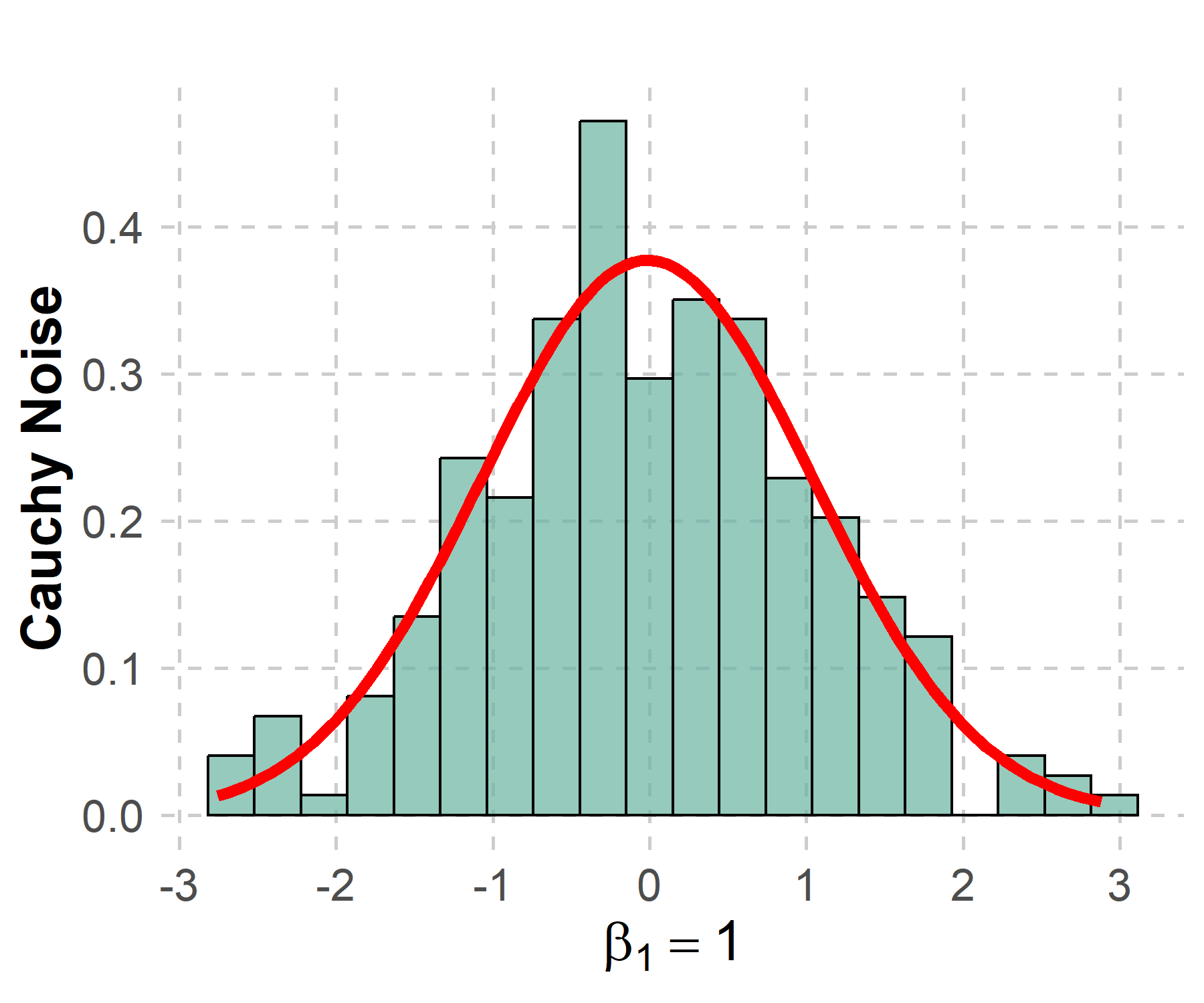}
        \includegraphics[width=0.29\textwidth]{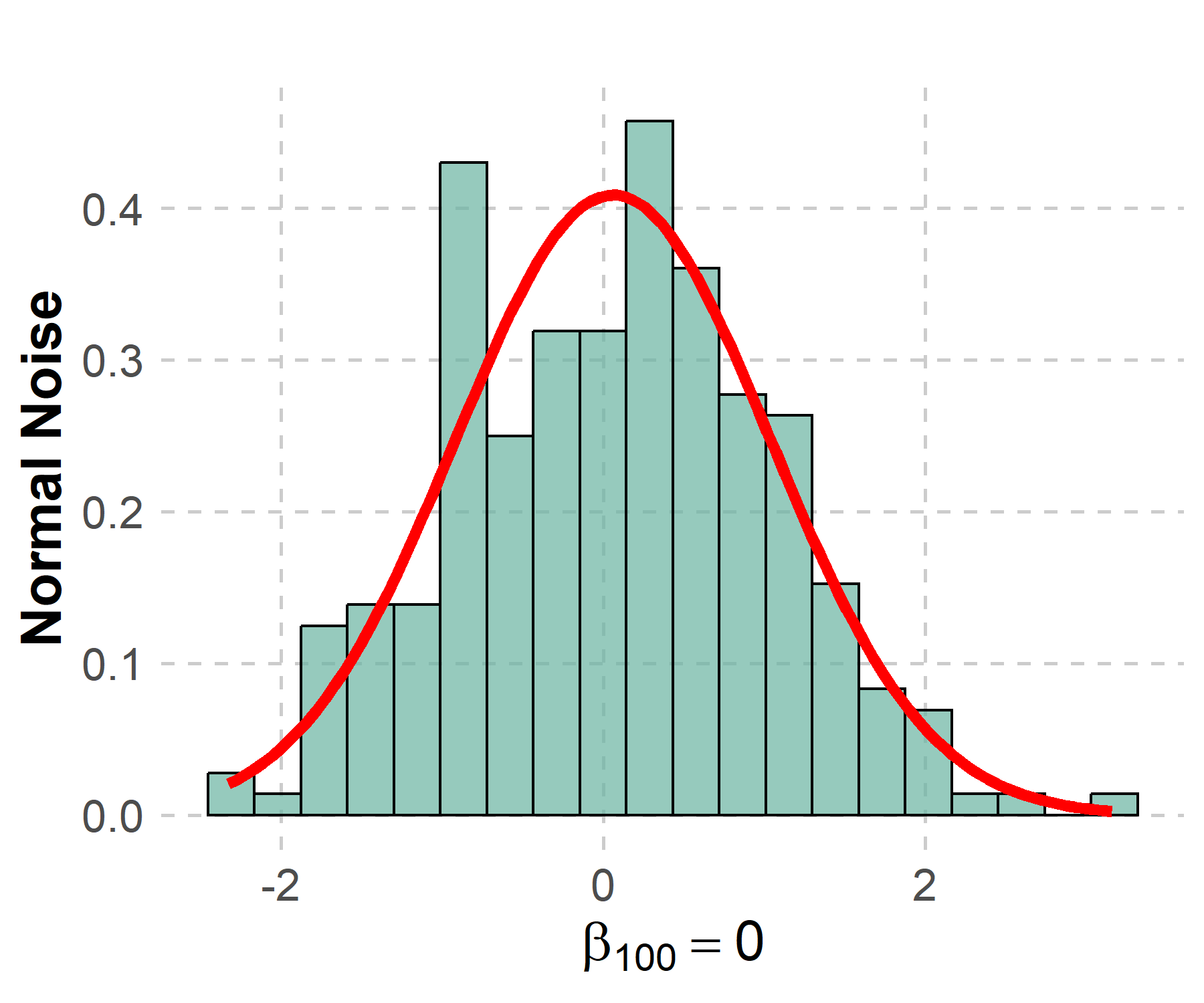}
        \includegraphics[width=0.29\textwidth]{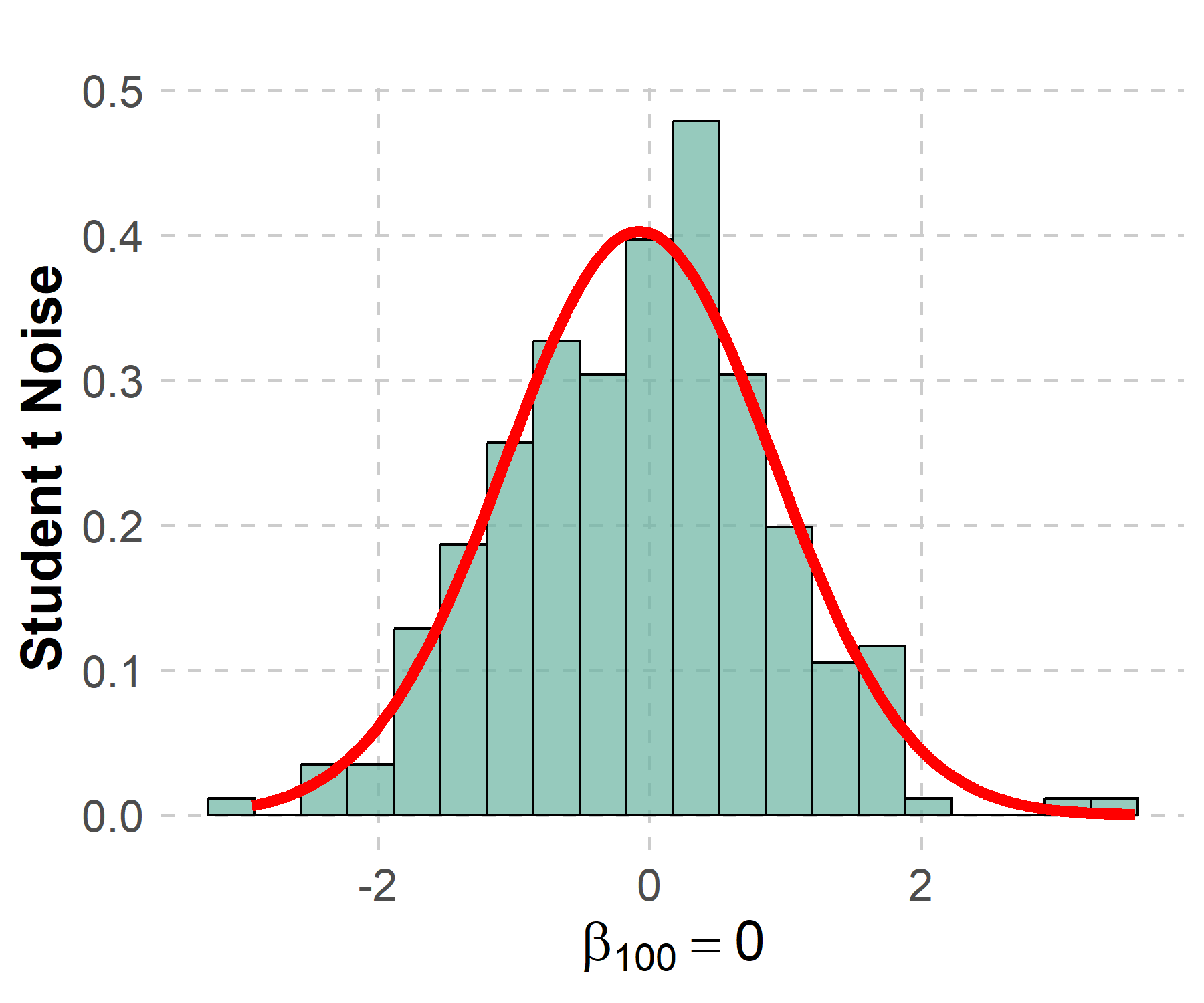}
        \includegraphics[width=0.29\textwidth]{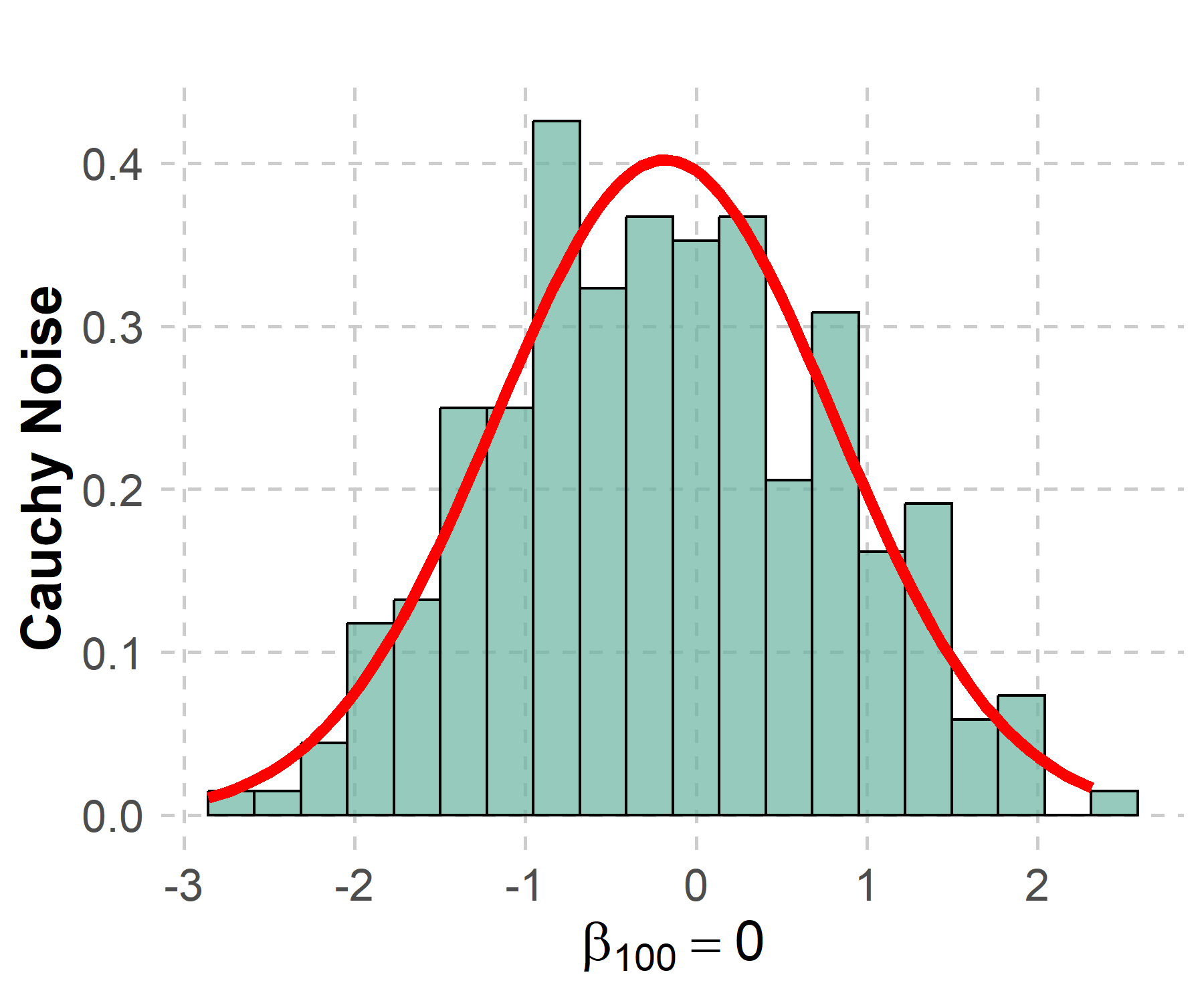}
        \includegraphics[width=0.29\textwidth]{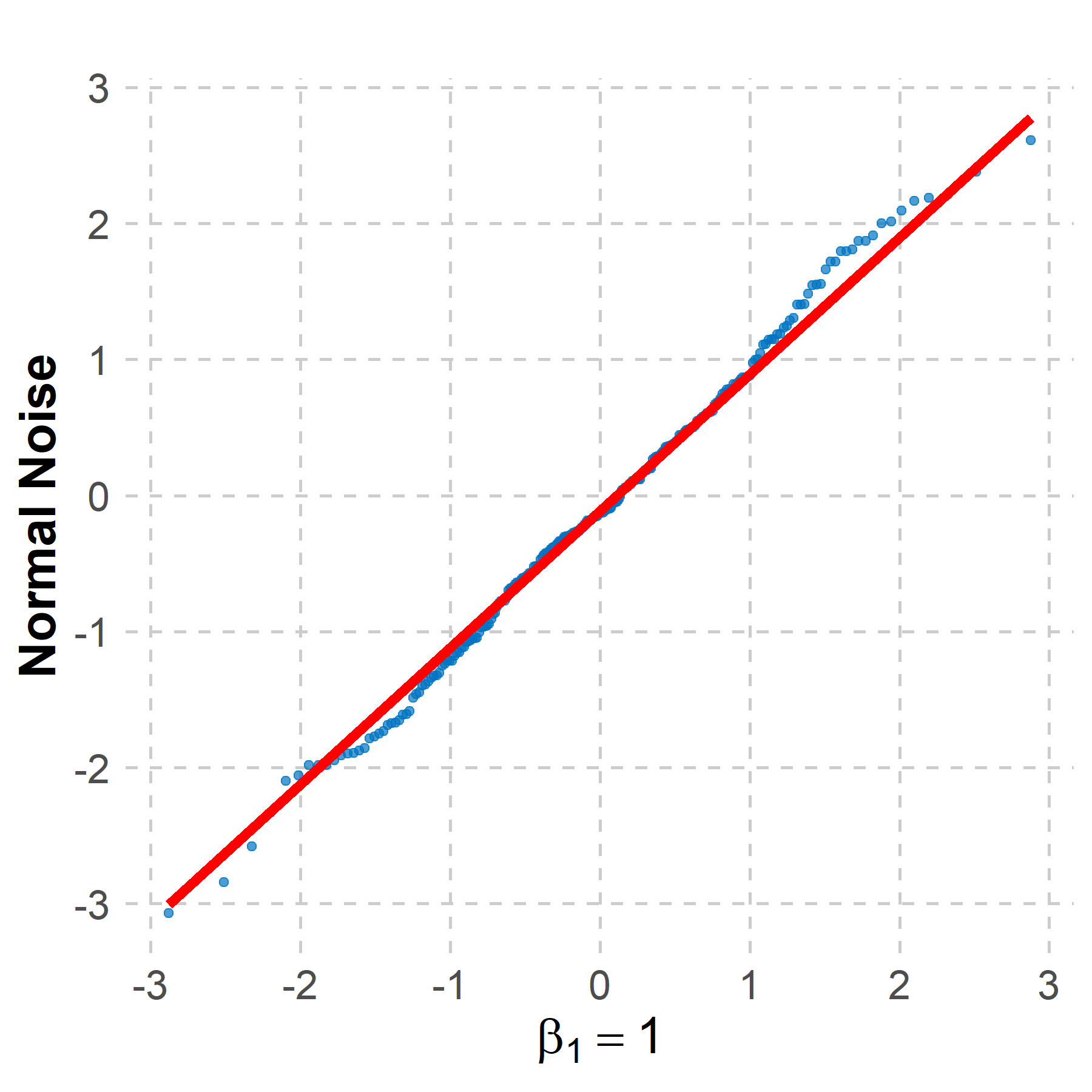}
        \includegraphics[width=0.29\textwidth]{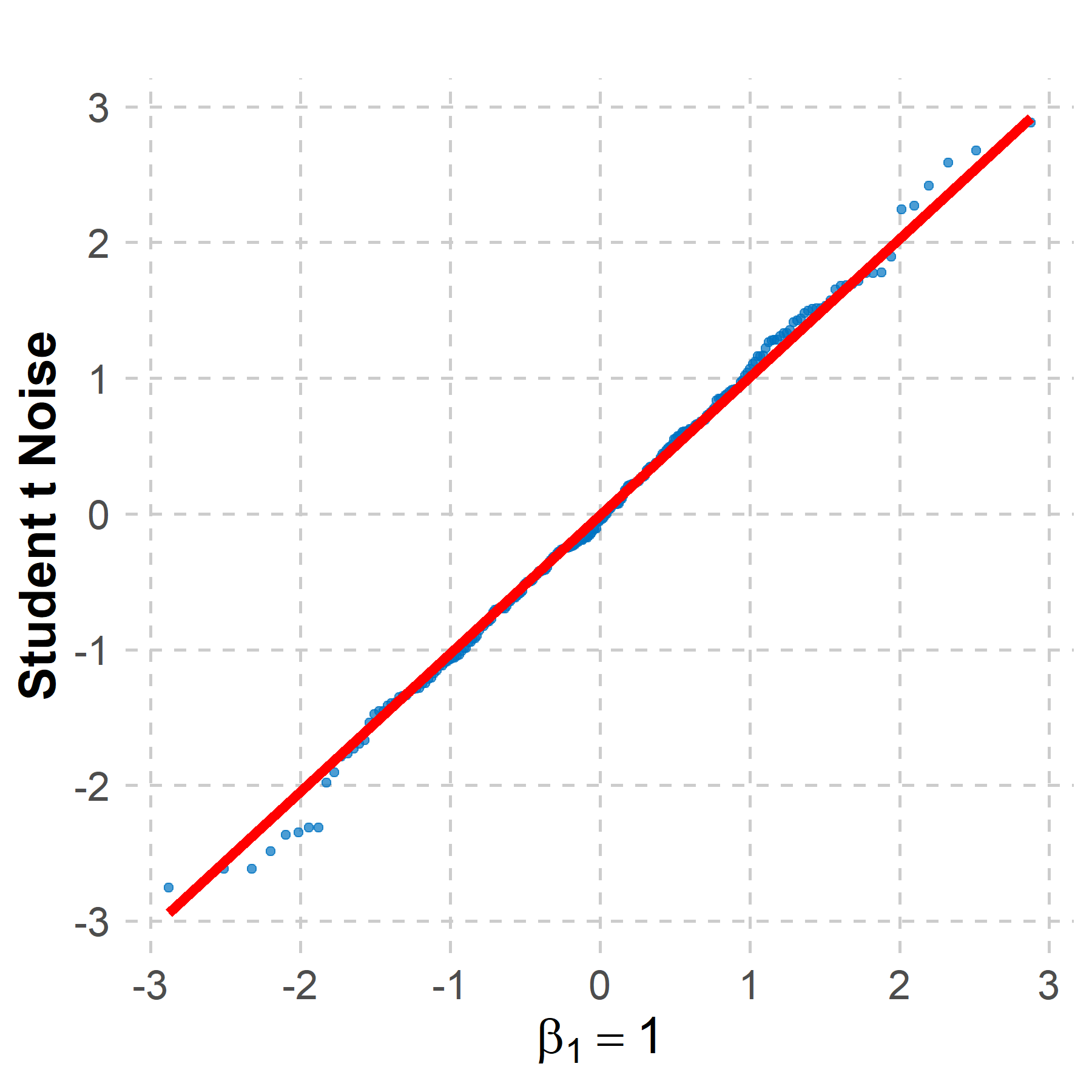}
        \includegraphics[width=0.29\textwidth]{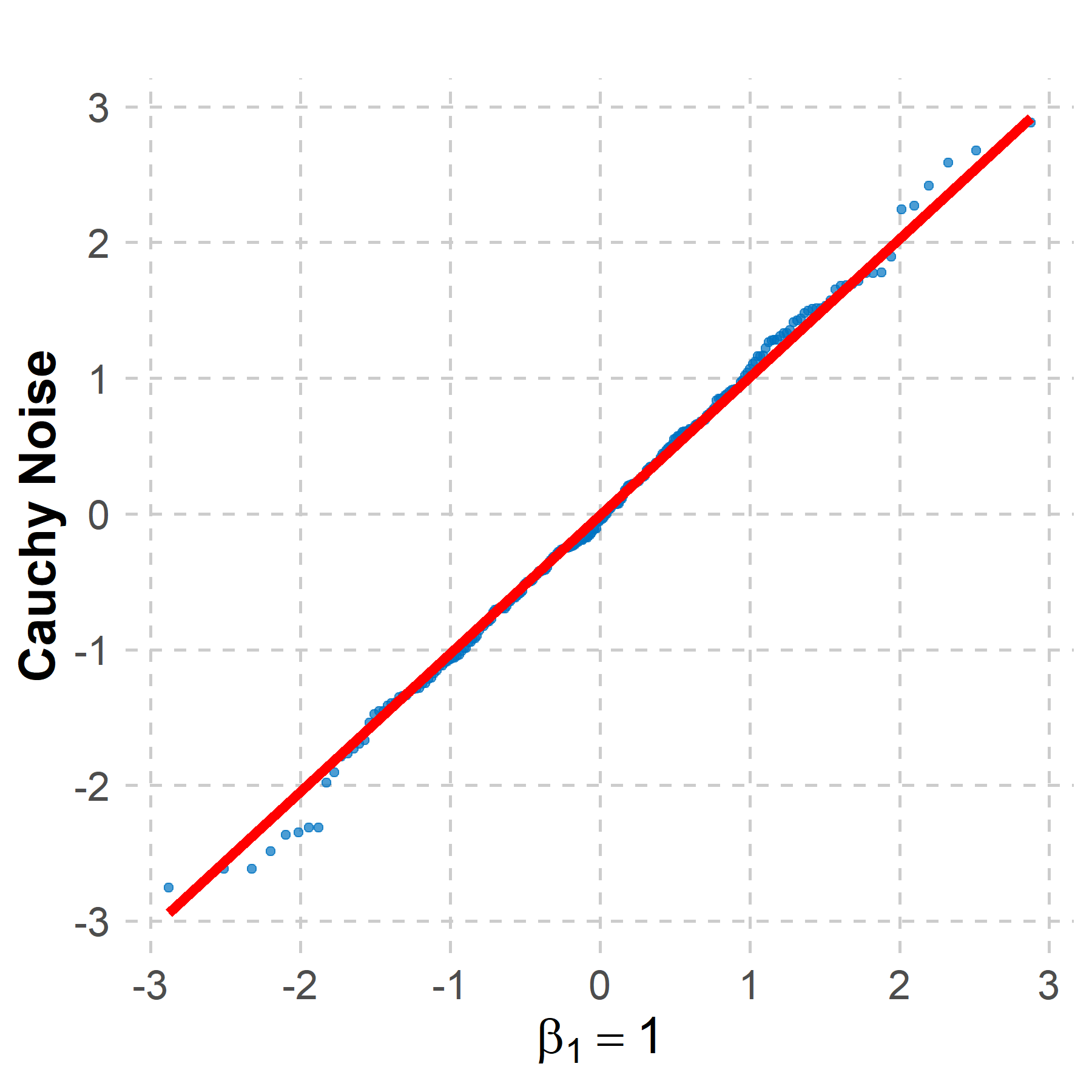}
        \includegraphics[width=0.29\textwidth]{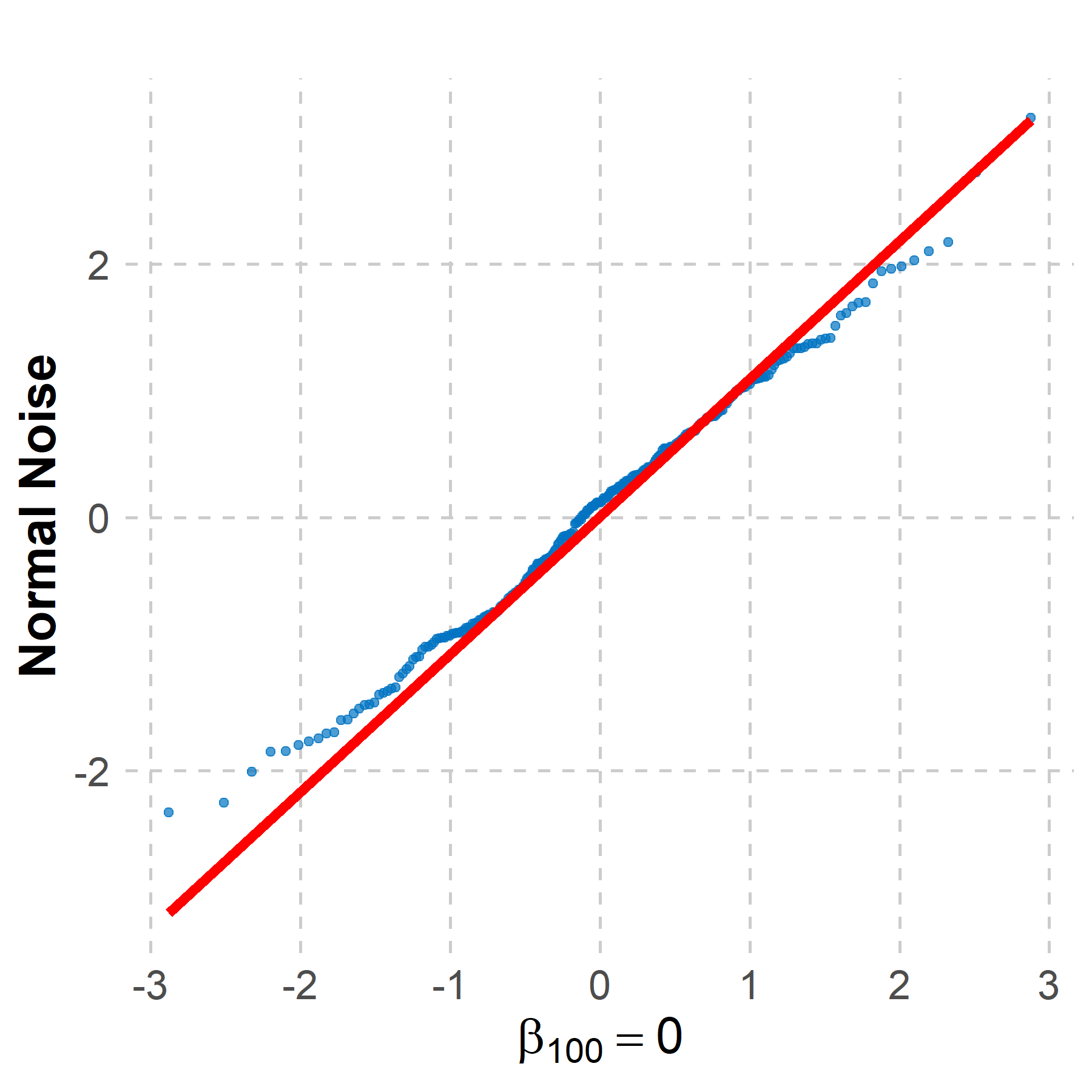}
        \includegraphics[width=0.29\textwidth]{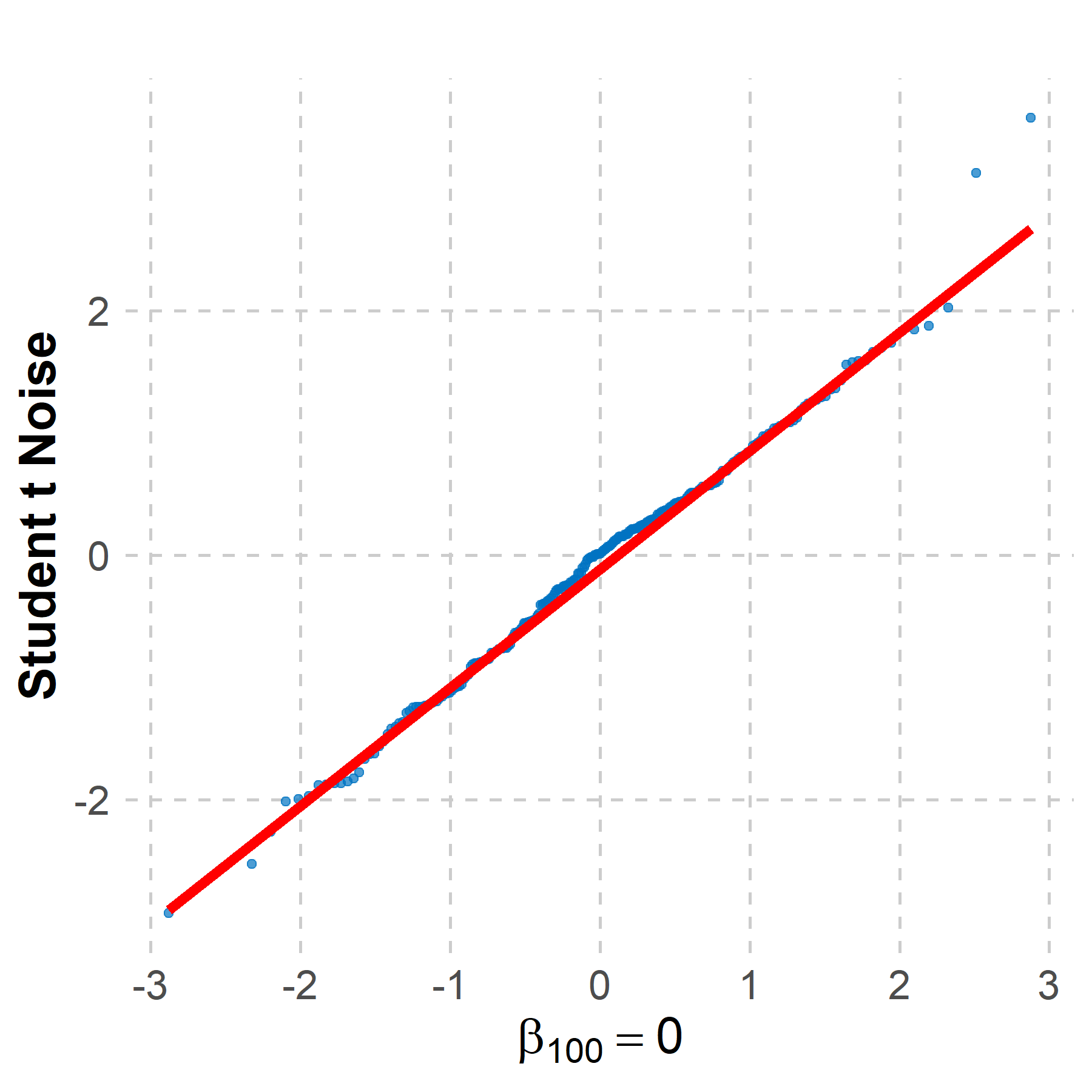}
        \includegraphics[width=0.29\textwidth]{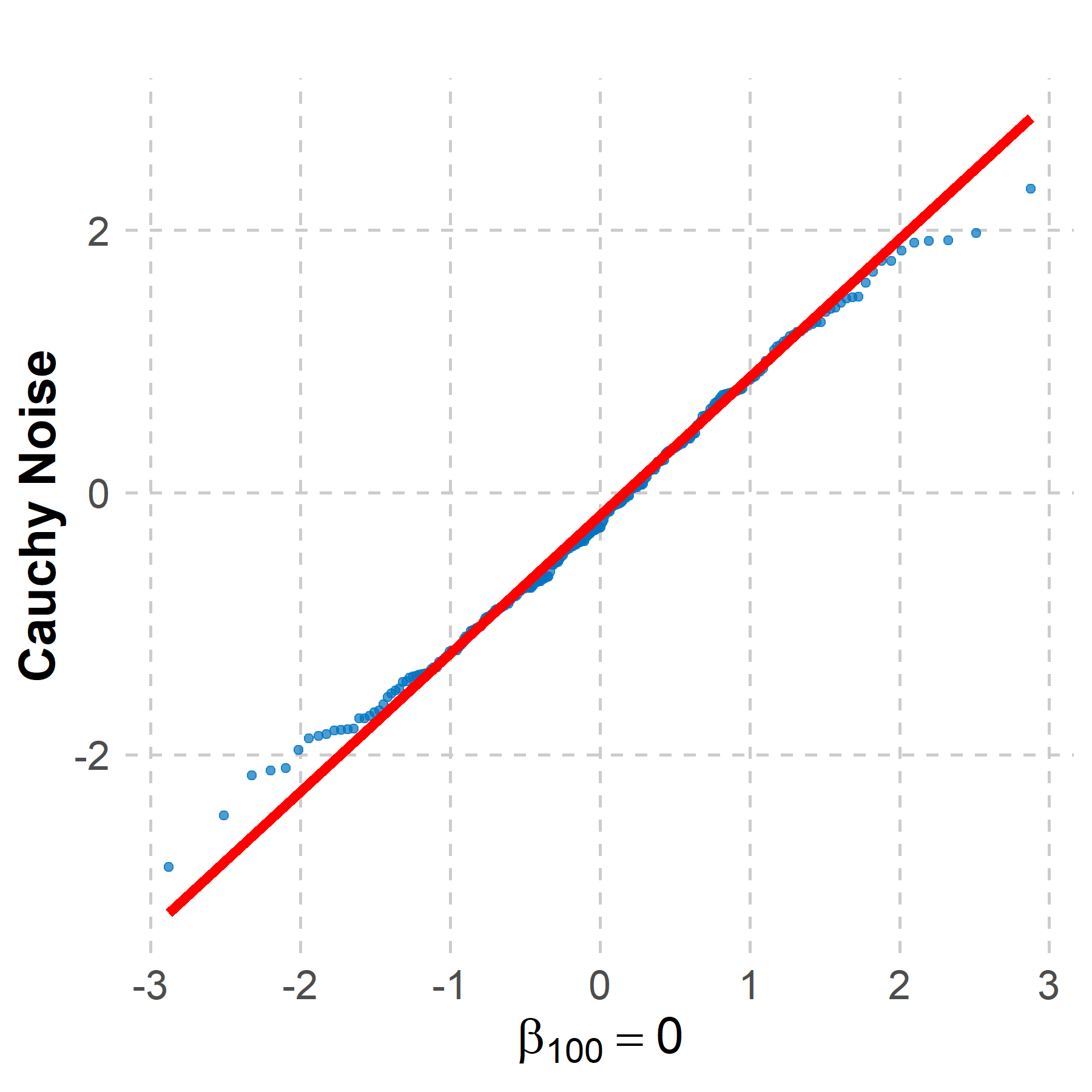} 
        \caption{Normality diagnostics for standardized test statistics under homoscedastic errors.}
    \label{fig:hom_combined_histograms_tau0.5}
\end{figure}

\setlength{\fboxrule}{1pt}  
\setlength{\fboxsep}{2pt}   
\begin{figure}[!t]
        \centering
        {\bfseries The case of heteroscedastic errors: Histograms} \\
        \includegraphics[width=0.29\textwidth]{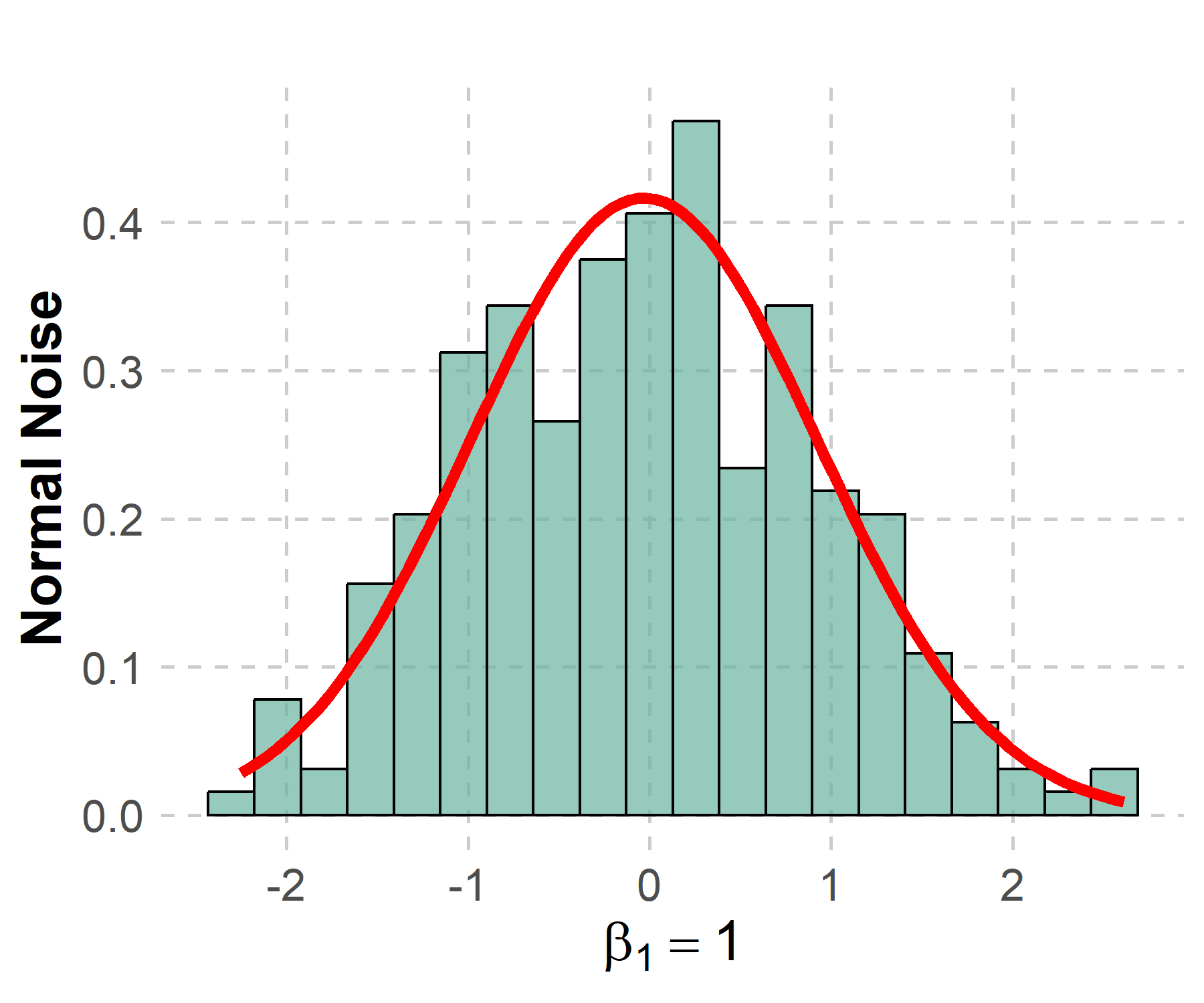}
        \includegraphics[width=0.29\textwidth]{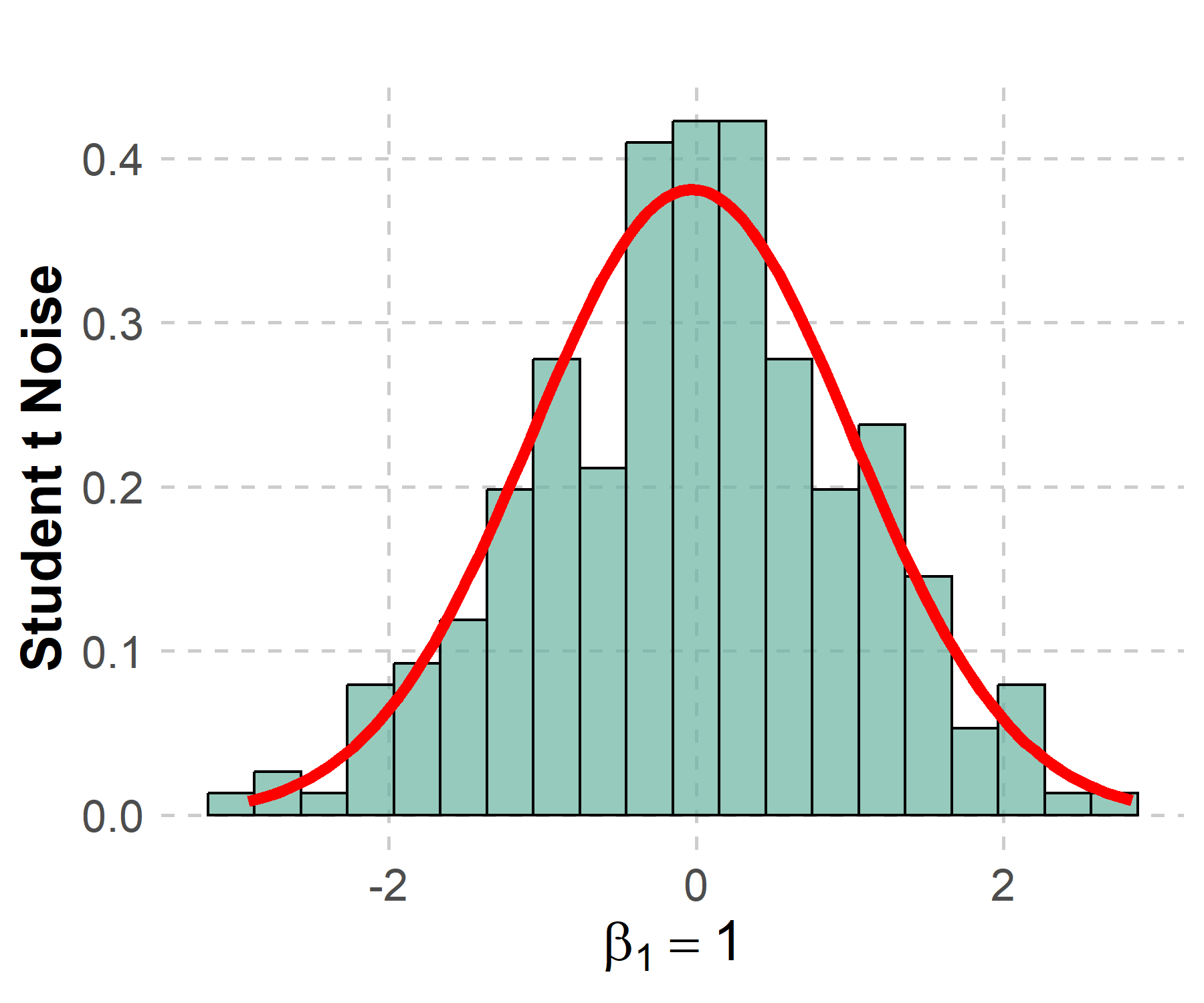}
        \includegraphics[width=0.29\textwidth]{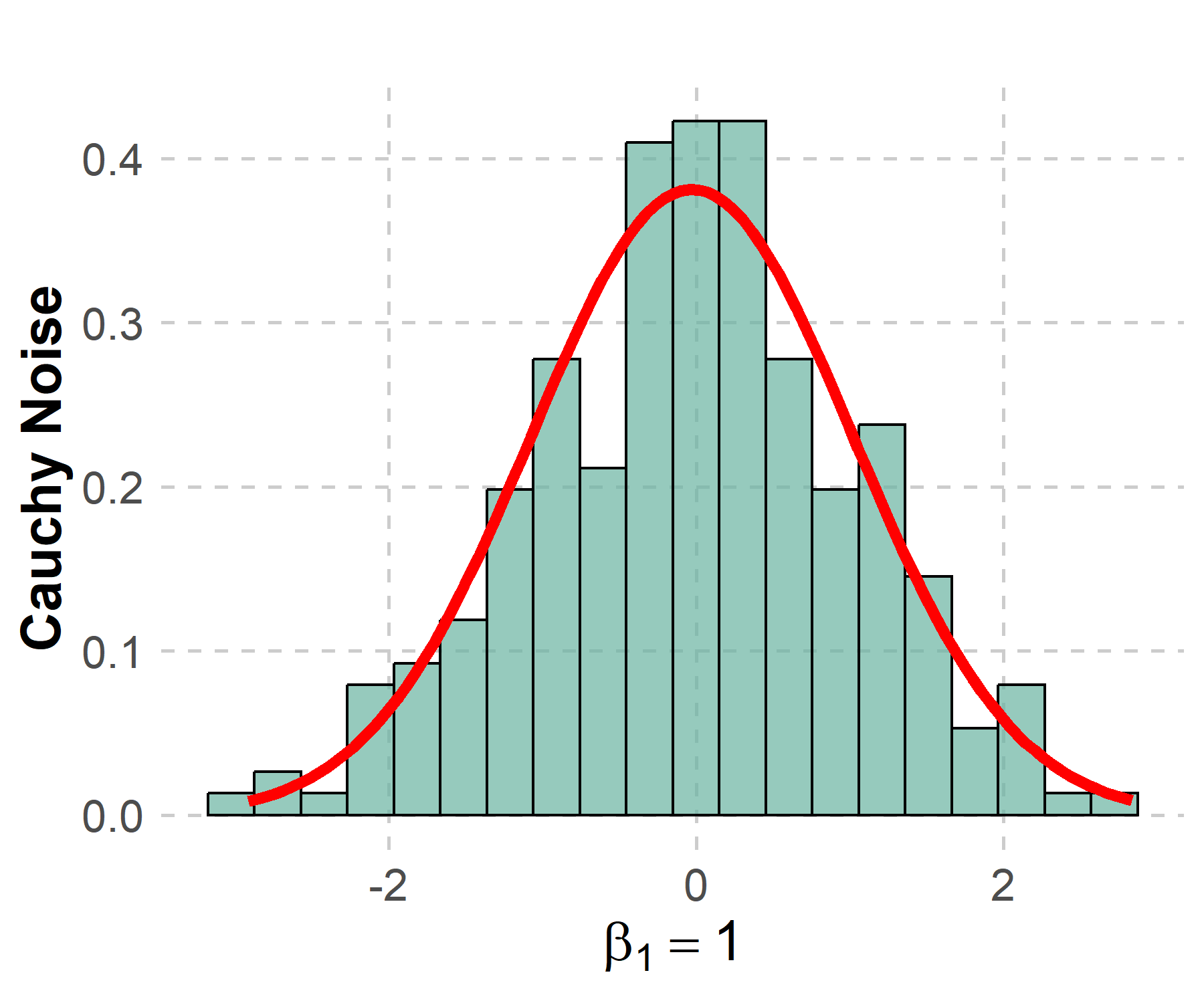}
        \includegraphics[width=0.29\textwidth]{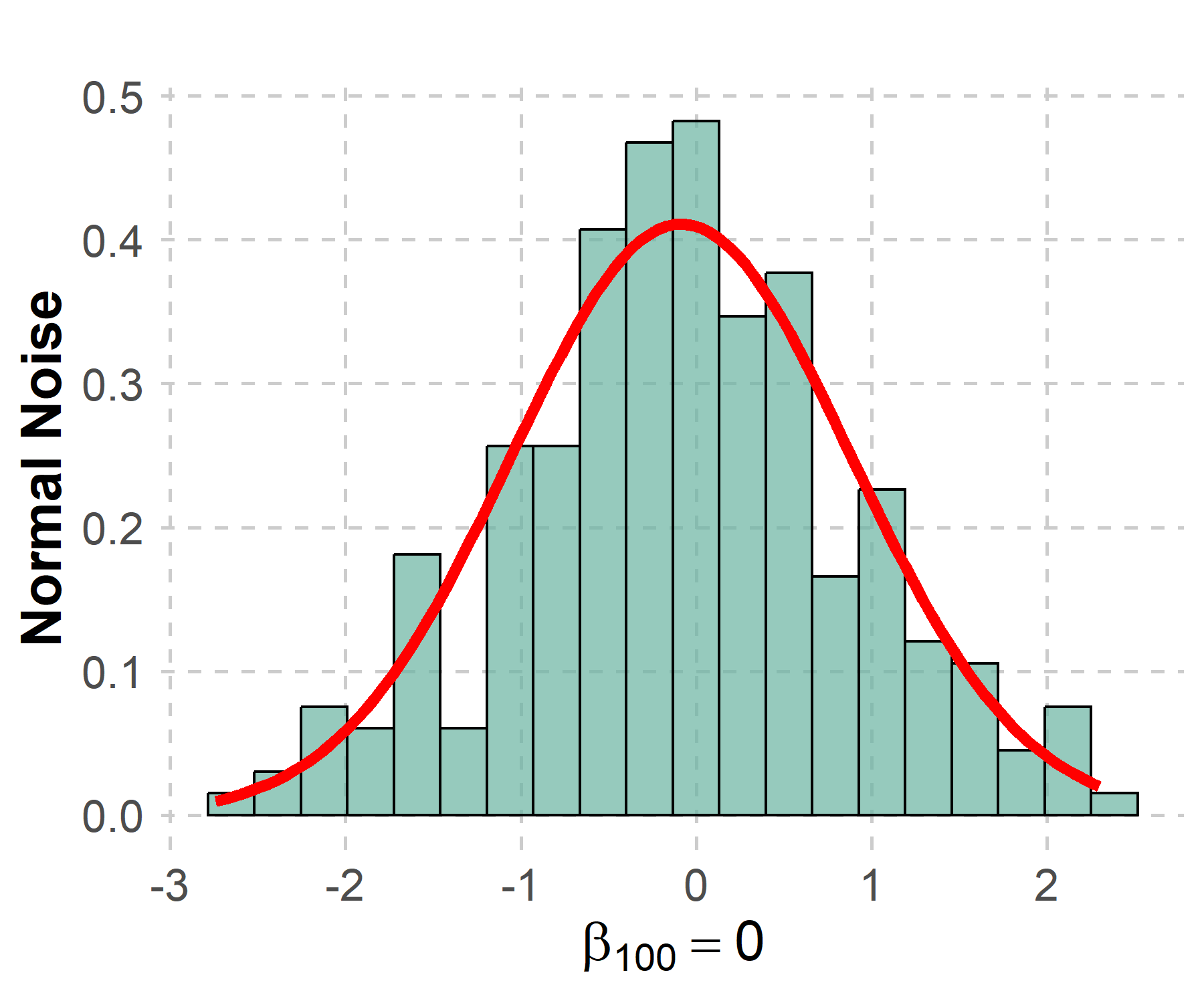}
        \includegraphics[width=0.29\textwidth]{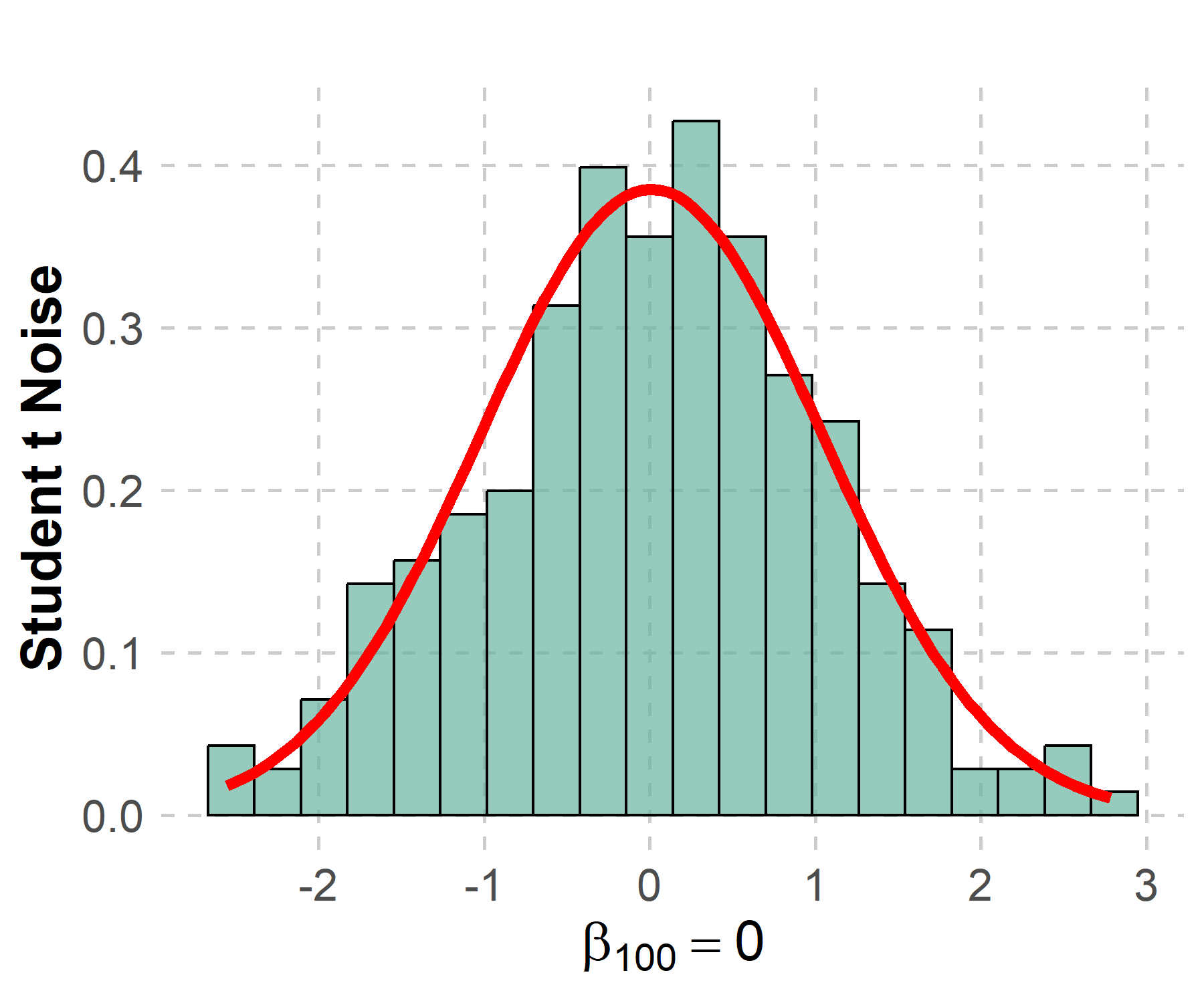}
        \includegraphics[width=0.29\textwidth]{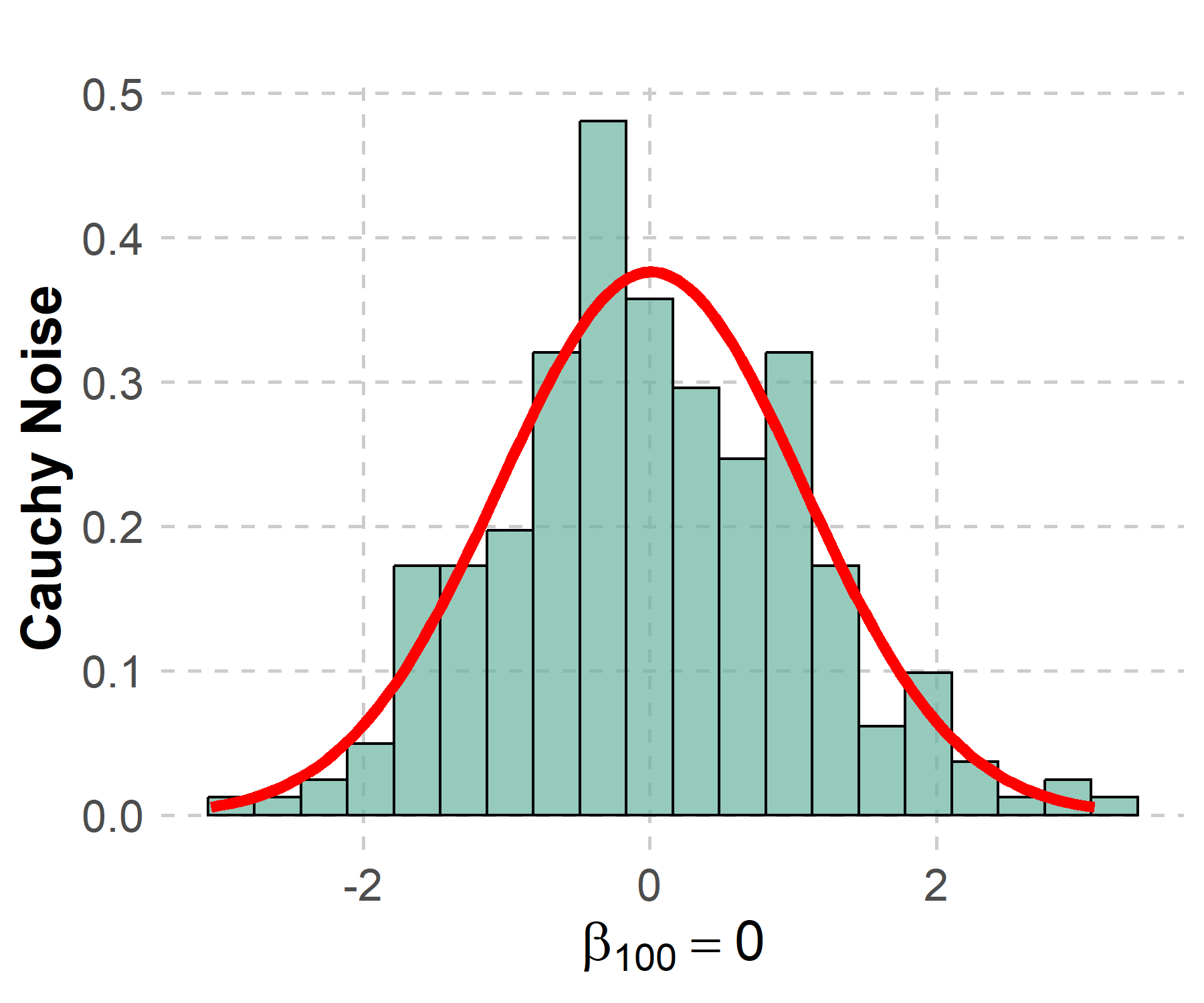}
        \includegraphics[width=0.29\textwidth]{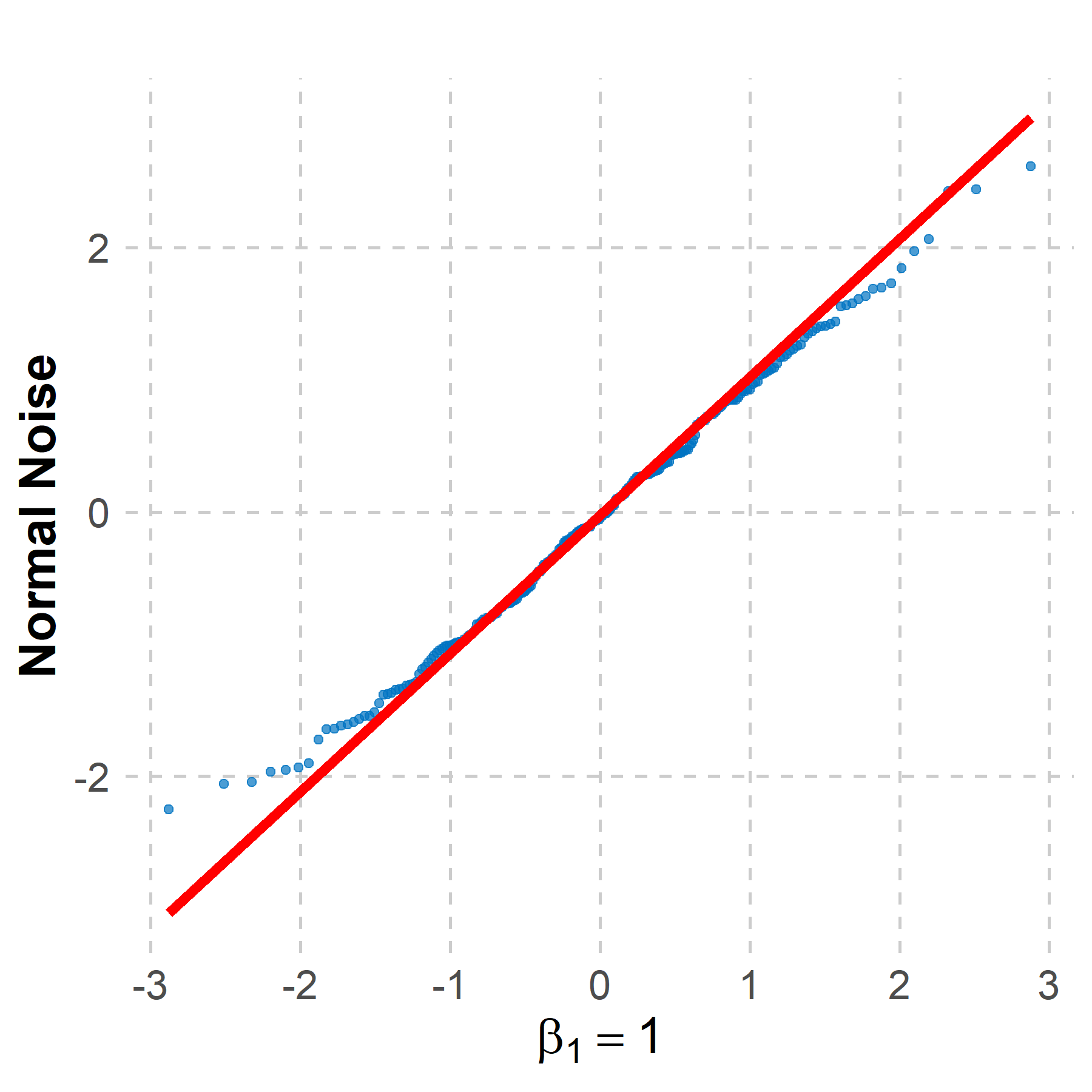}
        \includegraphics[width=0.29\textwidth]{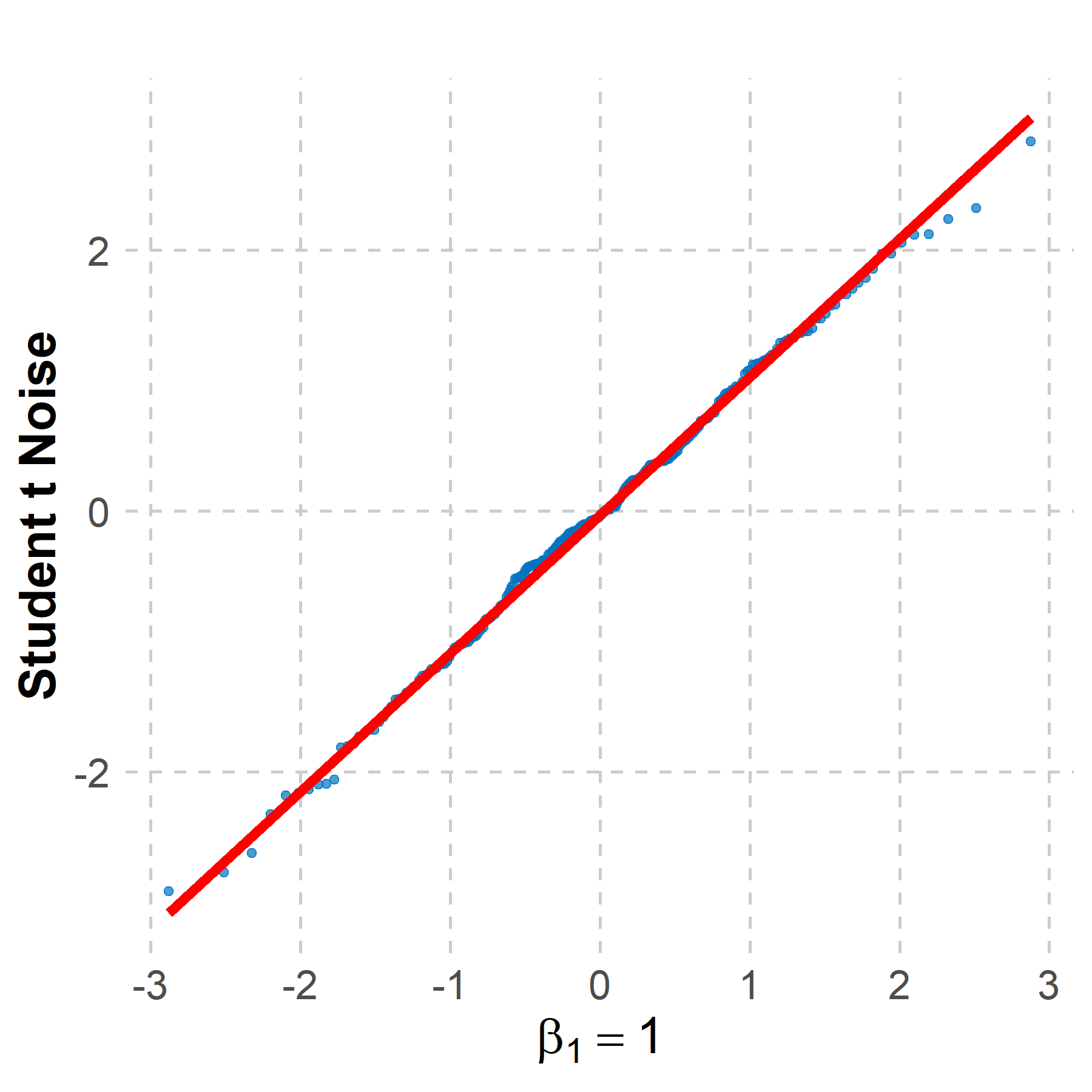}
        \includegraphics[width=0.29\textwidth]{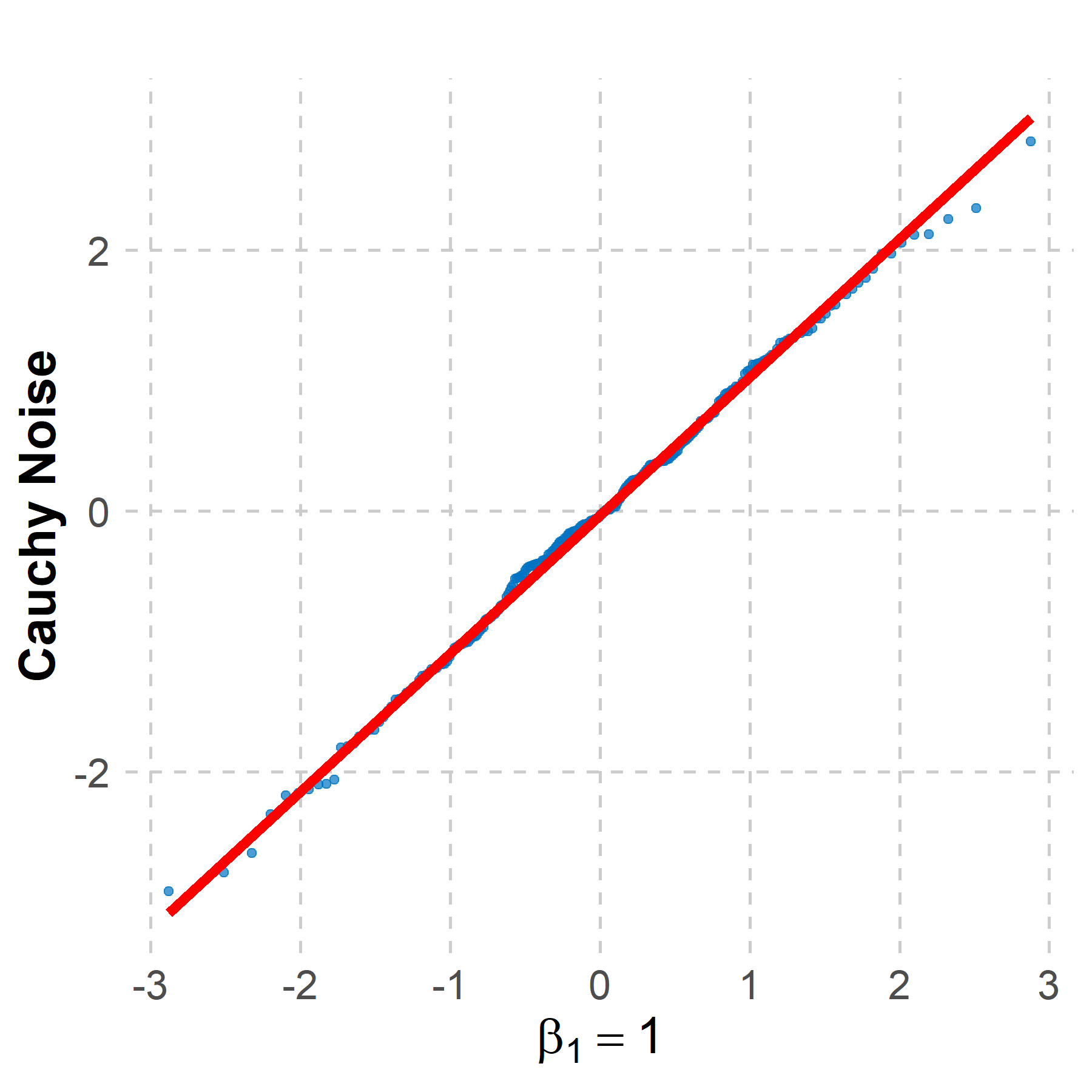}
        \includegraphics[width=0.29\textwidth]{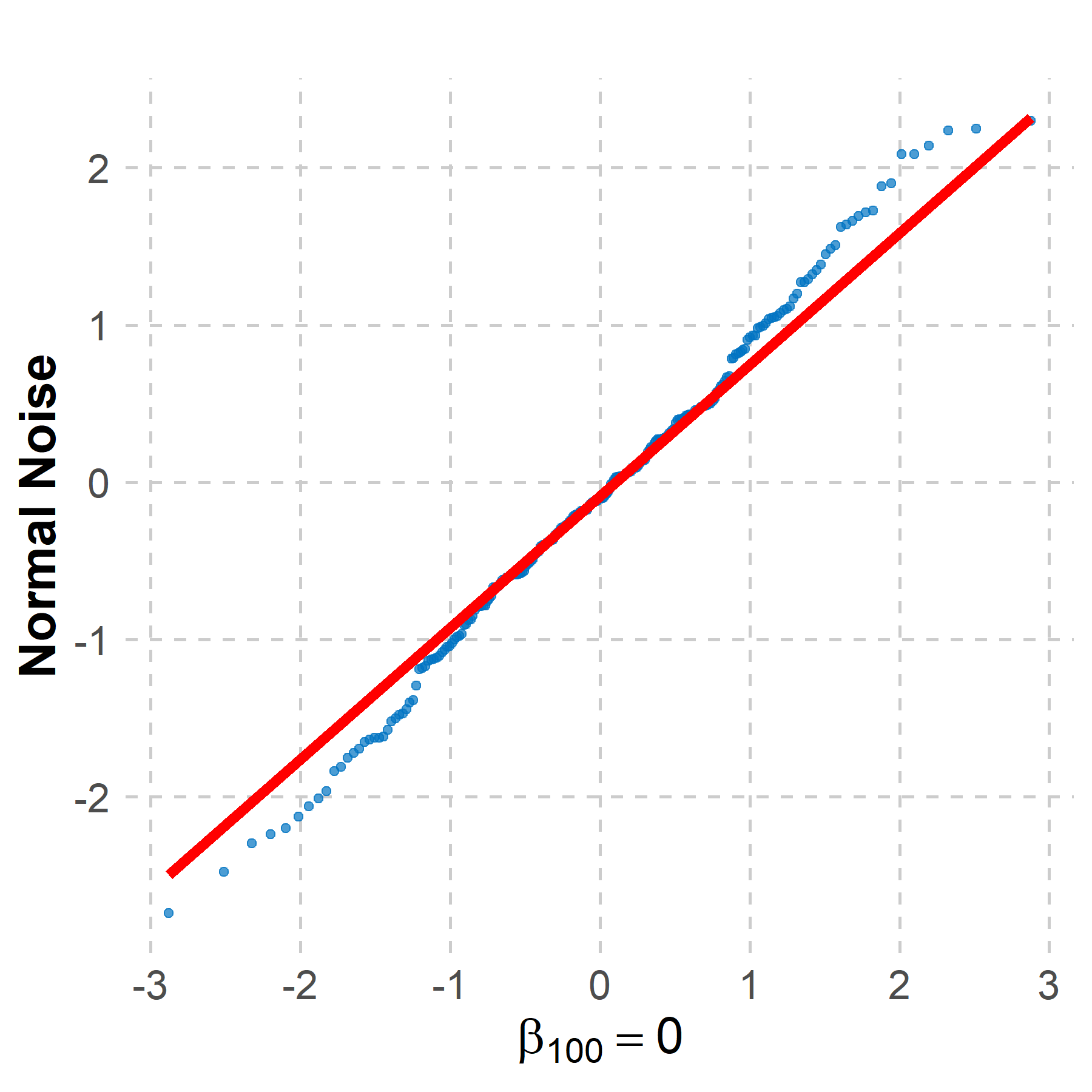}
        \includegraphics[width=0.29\textwidth]{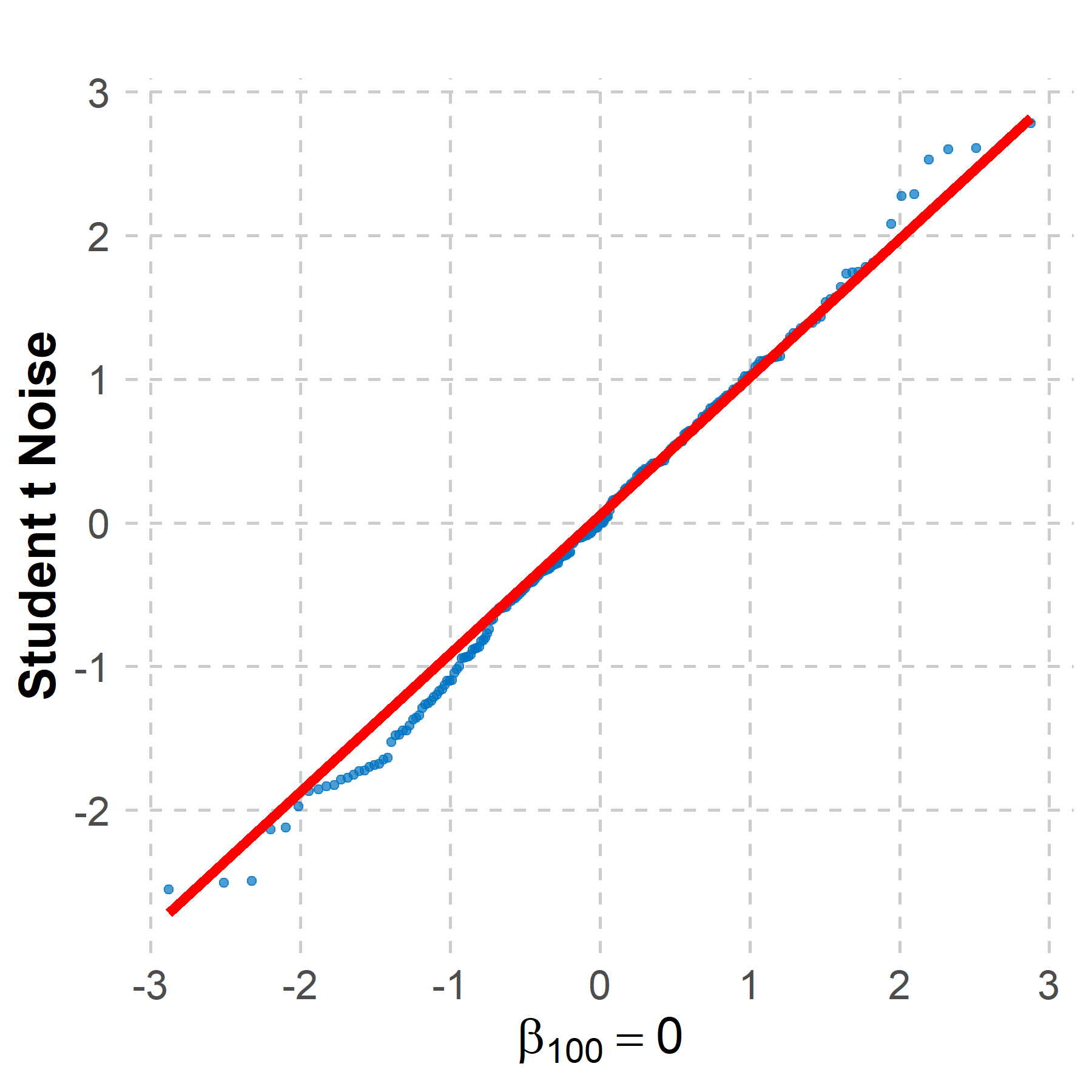}
        \includegraphics[width=0.29\textwidth]{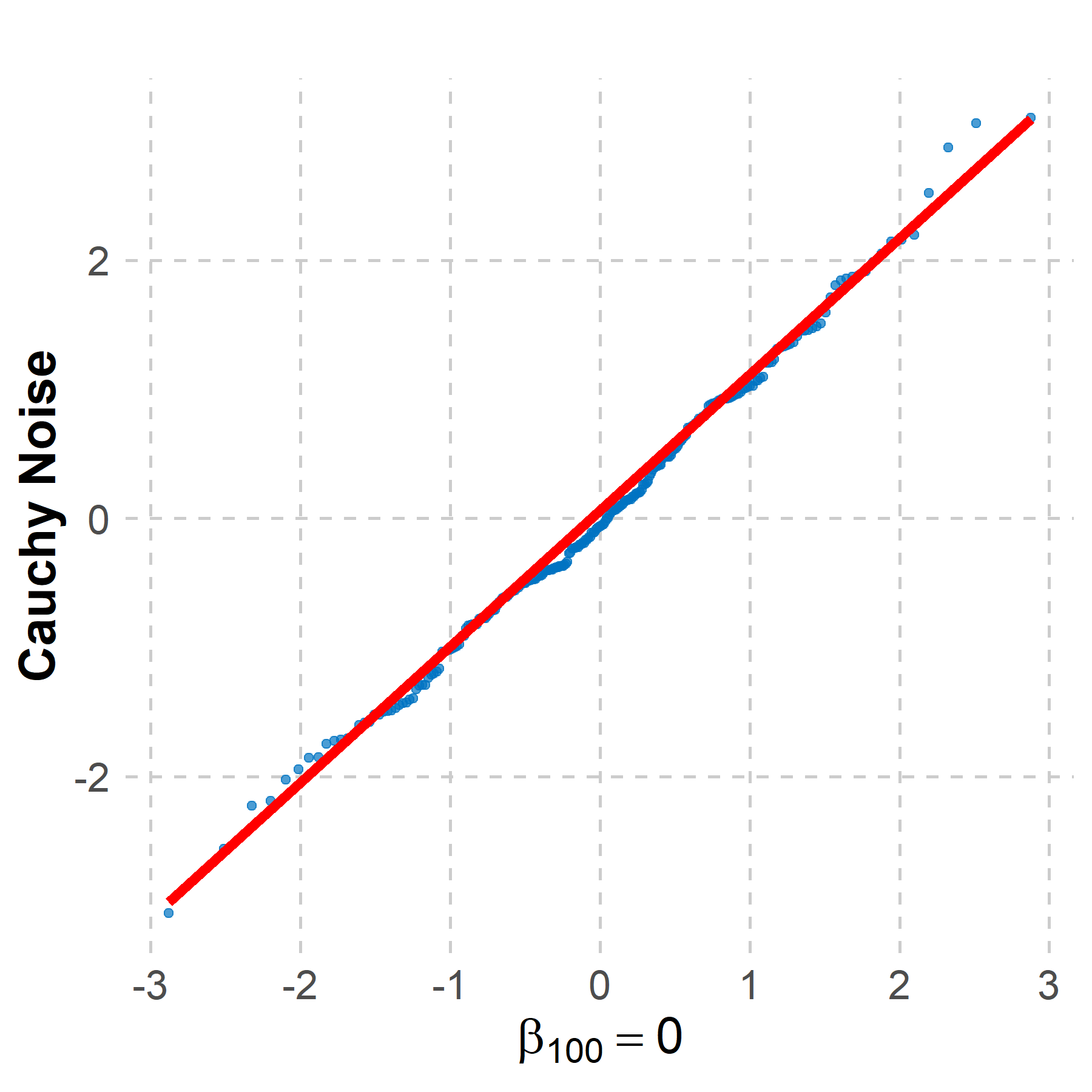} 
        \caption{Normality diagnostics for standardized test statistics under heteroscedastic errors.}
    \label{fig:het_combined_histograms_tau0.5}
\end{figure}

Figures \ref{fig:hom_combined_histograms_tau0.5} and \ref{fig:het_combined_histograms_tau0.5} display the distributions of the standardized test statistic
\[
z_j = \frac{\sqrt{mn}(\widetilde{\beta}_j - \beta^*_j)}{\sqrt{\frac{1}{m}\sum_{k=1}^m \widehat{\sigma}^{(k)} + \frac{8B_2^2 \log (1/\delta)}{mn\epsilon^2}}}
\]
for \(j = 1\) and \(j = 100\) under different null hypothesis settings. Across all noise types, the histograms demonstrate that \(z_j\) closely follows the standard normal distribution, even in the presence of heavy-tailed errors. This indicates that our inference procedure remains robust for both strong and weak signals. The accompanying Q-Q plots compare the empirical quantiles of \(z_j\) with those of the standard normal distribution. The close alignment of points along the $\mathrm{y} = \mathrm{x}$ line further confirms the approximate normality of \(z_j\). These empirical findings are consistent with the theoretical results established in Section~\ref{sec:DP-inference}.

\begin{figure}
    \centering
    \includegraphics[width=0.8\linewidth]{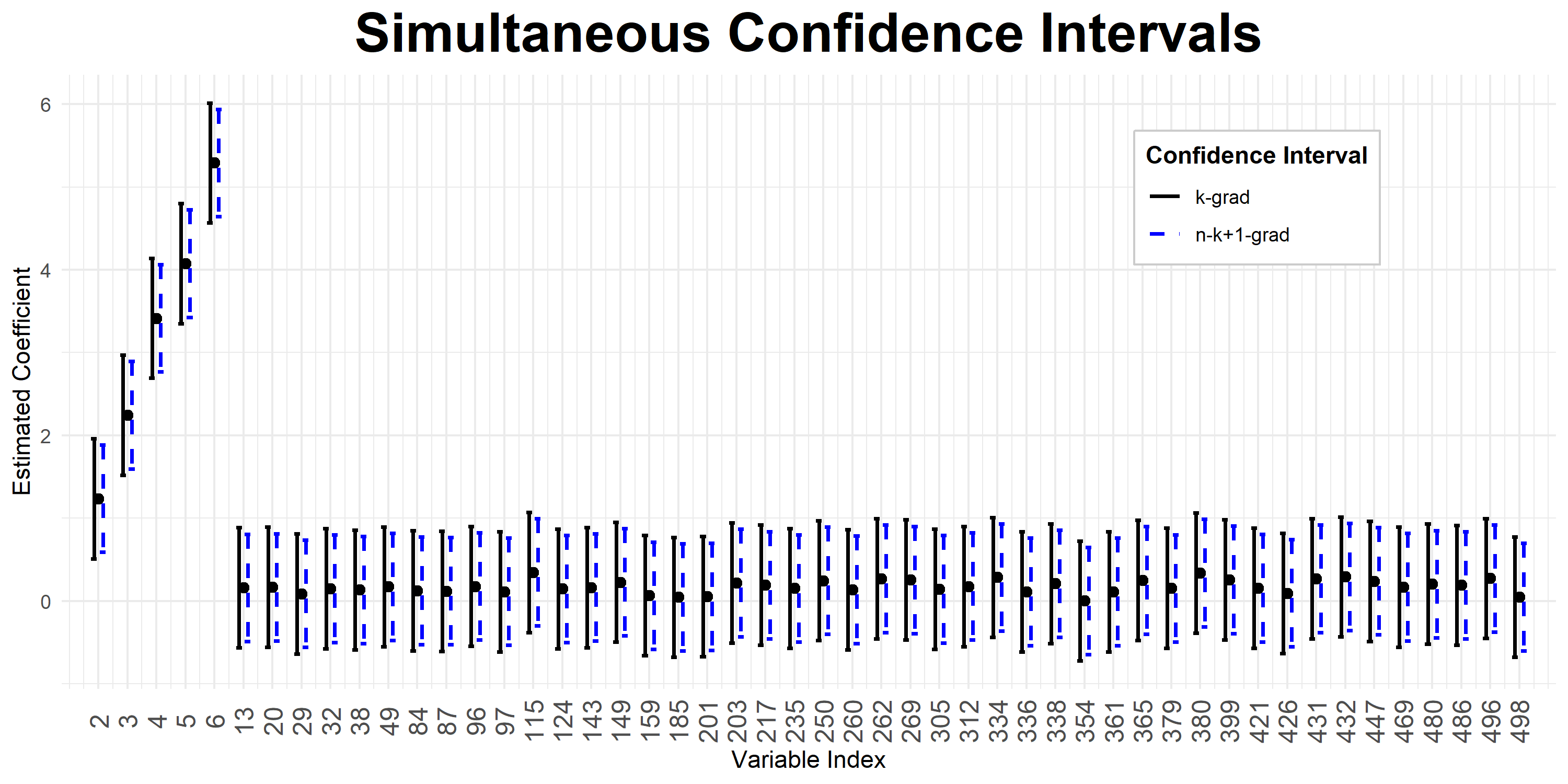}
    \caption{Simultaneous confidence intervals produced by the distributed bootstrap procedure.}
    \label{fig:bootstrap1}
\end{figure}

Figure~\ref{fig:bootstrap1} presents 95\% simultaneous confidence intervals under heteroscedastic Cauchy errors with $m=10$, $\epsilon=1$, and $\delta=1/N$. The figure reports all nonzero coordinates and a representative subset of 45 zero coordinates, with methods and reference values indicated in the legend. The intervals cover the true parameter values with practical widths, and the $\mathsf{(n+k-1)\text{-}grad}$ bootstrap gives narrower intervals than the $\mathsf{k\text{-}grad}$ bootstrap in this small-$m$ setting, consistent with Section~\ref{sec:DP-bootstrap}.




\section{  {Real Data Analysis}}\label{realdata}

 {We use the 2022 Medical Expenditure Panel Survey (MEPS) Full Year Consolidated Data File (HC-243)\footnote{\url{https://meps.ahrq.gov/mepsweb/}} to evaluate the proposed DP-QR framework. The original MEPS HC-243 dataset contains 22431 observations and 1415 non-identifier variables, which contain individual-level demographic, socioeconomic, insurance, healthcare utilization, and clinical information. This makes it a suitable benchmark for distributed privacy-preserving healthcare analytics, since individual-level records are sensitive and often cannot be directly pooled across data holders. 
The response variable is log-transformed annual total healthcare expenditure, $y_i=\log(1+\texttt{TOTEXP22}_i)$. Since the expenditure distribution is highly right-skewed, it motivates quantile regression rather than mean regression. We analyze the transformed response $y_i=\log(1+\texttt{TOTEXP22}_i)$ at quantile levels $\tau\in\{0.50,0.75,0.90\}$. Our task is to estimate conditional quantiles of annual healthcare expenditure from distributed individual-level records and to assess key factors associated with the typical and upper-tail medical spending under different privacy budgets. We focus on the median and upper quantiles because our main interest is to understand typical spending and high-expenditure risk. Lower quantiles mainly correspond to individuals with limited healthcare use, where utilization-related signals are weaker and more easily masked by privacy noise.}

 {For data pre-processing, we first remove expenditure sub-components and payment-source variables to avoid target leakage.  
Negative MEPS response codes are mapped to missing values, and variables with at least 20\% missingness are removed. We then restrict the sample to individuals with positive response \texttt{TOTEXP22} and observed \texttt{INSCOV22} and \texttt{REGION22}. Here, \texttt{INSCOV22} denotes the individual's primary health insurance coverage status in 2022, and \texttt{REGION22} denotes the individual's geographic region of residence. To mimic a federated healthcare environment, we partition the sample by the interaction between \texttt{INSCOV22} and \texttt{REGION22}. Since very small local sample size may lead to unstable estimation and privacy-noise calibration, clients with fewer than 300 observations are removed. The final analysis sample contains $N=18587$ individuals distributed across $m=9$ client nodes. Remaining missing covariates are imputed by median for continuous variables and mode for discrete or categorical variables. We also add second-order interactions among the first 15 core covariates. After removing extremely sparse or nearly constant dummy columns, the final regression design contains $p=286$ covariates excluding the intercept. }

 {To evaluate the predictive performance of the proposed method on real healthcare data, we conduct prediction experiments across different privacy budgets and compare the results with a centralized non-private pooled benchmark. In our experiments, we randomly choose $N_{\mathrm{train}}=14869$ samples as the training data and the rest as the testing data. We fix the gradient step size to $\eta_0=0.05$. The privacy parameter is evaluated at $\epsilon\in\{0.10,0.50,0.75,1.00\}$, and we set $\delta=1/N_{\mathrm{train}}^2=4.523\times 10^{-9}$. Figure~\ref{fig:loss_curve}  reports empirical risk on the test data, which shows a clear privacy--utility trade-off. When the privacy budget is small, especially at $\epsilon=0.10$, the injected DP noise substantially increases the prediction loss for all quantile levels. As $\epsilon$ increases, the check loss decreases sharply, indicating that a larger privacy budget leads to better predictive accuracy.}

\begin{figure}[ht]
    \centering
    \includegraphics[width=0.68\linewidth]{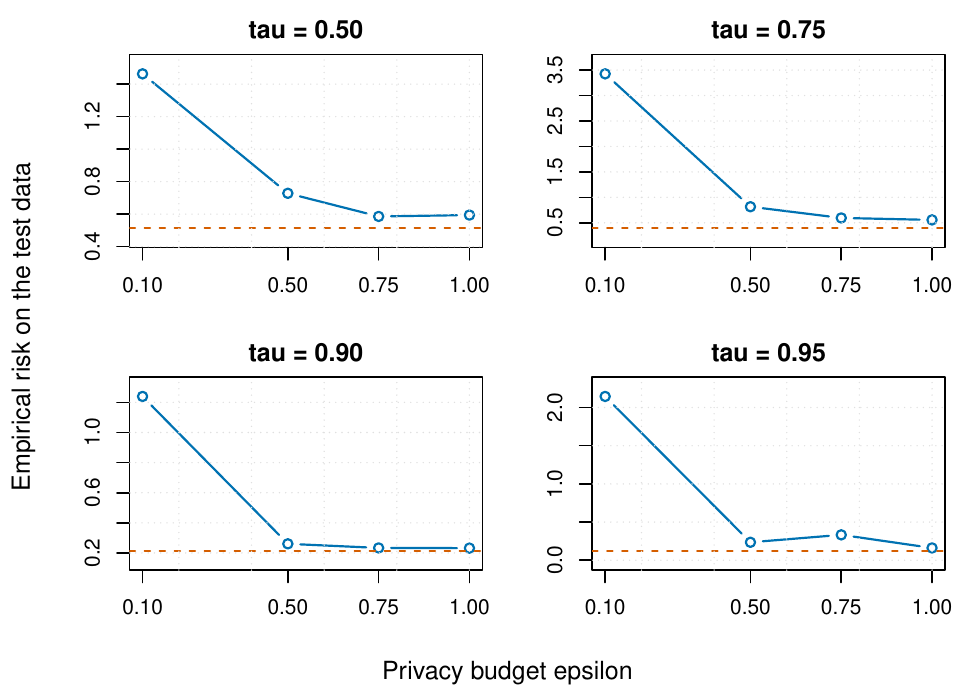}
    \caption{Empirical risk on the test data across privacy budgets and quantile levels. Dashed horizontal lines denote the centralized non-private pooled benchmark.}
    \label{fig:loss_curve}
\end{figure}


\begin{table}[ht]
\centering
\caption{Selected healthcare utilization variables identified by our method and used in the main debiased inference analysis.}
\label{tab:selected_variable_descriptions}
\small
\setlength{\tabcolsep}{5pt}
\begin{tabular}{lll}
\toprule
Variable & Description & Interpretation \\
\midrule
\texttt{hhagd22} & Agency home-health provider days in 2022 & Home-health care intensity \\
\texttt{ipdis22} & Hospital inpatient discharges in 2022 & Inpatient utilization intensity \\
\texttt{obdrv22} & Office-based physician visits in 2022 & Office-based physician utilization \\
\texttt{opdrv22} & Outpatient department physician visits in 2022 & Outpatient physician utilization \\
\bottomrule
\end{tabular}
\end{table}

 {Beyond evaluating predictive performance, our real-data analysis also aims to understand which healthcare utilization factors are associated with annual medical expenditure at different conditional quantile levels. Table~\ref{tab:selected_variable_descriptions} lists the healthcare utilization variables identified by our estimation method and the subsequent debiased inference step as being significantly different from zero in most scenarios. These selected variables correspond to distinct dimensions of medical service use: \texttt{hhagd22} measures agency home-health provider days, \texttt{ipdis22} measures hospital inpatient discharges, \texttt{obdrv22} measures office-based physician visits, and \texttt{opdrv22} measures outpatient department physician visits. The selection suggests that these are important utilization-related factors associated with conditional healthcare expenditure, especially at higher expenditure quantiles.}

\begin{table}[ht]
\centering
\caption{Debiased inference for selected healthcare utilization drivers at $\epsilon=1.00$.}
\label{tab:inference_results_main}
\small
\setlength{\tabcolsep}{4pt}
\begin{tabular}{llccccc}
\toprule
Feature & $\tau$ & Pooled est. & Pooled 95\% CI & DP est. & DP 95\% CI & Sig. \\
\midrule
\texttt{hhagd22} & 0.50 & 0.258 & [0.091, 0.425] & 0.367 & [0.027, 0.707] & * \\
                 & 0.75 & 0.220 & [0.169, 0.271] & 0.203 & [-0.082, 0.488] &  \\
                 & 0.90 & 0.223 & [0.201, 0.245] & 1.374 & [1.119, 1.630] & * \\
\addlinespace
\texttt{ipdis22} & 0.50 & 0.494 & [0.428, 0.560] & 0.121 & [-0.168, 0.410] &  \\
                 & 0.75 & 0.473 & [0.414, 0.533] & 0.349 & [0.081, 0.617] & * \\
                 & 0.90 & 0.438 & [0.356, 0.519] & 1.224 & [1.024, 1.425] & * \\
\addlinespace
\texttt{obdrv22} & 0.50 & 0.667 & [0.601, 0.733] & 0.667 & [0.364, 0.971] & * \\
                 & 0.75 & 0.582 & [0.523, 0.641] & 0.822 & [0.553, 1.091] & * \\
                 & 0.90 & 0.556 & [0.519, 0.592] & 1.021 & [0.888, 1.153] & * \\
\addlinespace
\texttt{opdrv22} & 0.50 & 0.389 & [0.287, 0.492] & 0.690 & [0.420, 0.961] & * \\
                 & 0.75 & 0.404 & [0.352, 0.455] & 0.681 & [0.480, 0.882] & * \\
                 & 0.90 & 0.351 & [0.312, 0.390] & 1.310 & [1.091, 1.529] & * \\
\bottomrule
\end{tabular}
\end{table}

 {Table~\ref{tab:inference_results_main} reports the debiased inference results for selected healthcare utilization variables at $\epsilon=1.00$. Overall, all reported variables have positive pooled estimates across the considered quantile levels, indicating that greater healthcare utilization is associated with higher conditional medical expenditure. The DP estimates are also directionally consistent with the pooled non-private estimates, although the DP confidence intervals are wider because they incorporate privacy noise and distributed estimation uncertainty.}

 {The strength of these associations varies across the expenditure distribution. Office-based physician visits \texttt{obdrv22} and outpatient department physician visits \texttt{opdrv22} show stable positive associations across the median, upper-middle, and high-expenditure quantiles, suggesting that physician-service utilization is relevant not only for typical spending but also for higher expenditure levels. Inpatient discharges \texttt{ipdis22} and agency home-health provider days \texttt{hhagd22} become more pronounced at the upper quantile, especially at $\tau=0.90$, where their DP confidence intervals exclude zero. This pattern suggests that high-cost patients are more strongly associated with intensive healthcare utilization, including inpatient care, outpatient physician visits, and home-health services. Thus, the selected utilization factors are not equally important across the full expenditure distribution. However, their associations tend to be stronger in the upper tail at $\tau=0.9$.}

 {These results demonstrate that our proposed method can recover interpretable and utilization-related signals under a moderate privacy budget in a realistic distributed healthcare setting. 
Throughout this real data analysis, it can be shown that the proposed method can identify healthcare utilization variables that are predictive of different parts of the conditional expenditure distribution under differential privacy.}


\section{Conclusion and Future Work}\label{sec:Conclusion}
In this work, we studied distributed high-dimensional quantile regression under differential privacy, providing both theoretical guarantees and practical algorithms. Our results show that the final estimation error decomposes into two main components: the oracle convergence rate and the additional error due to privacy, consistent with the ``cost of privacy'' established in \cite{cai2021cost}. For inference, we demonstrated that the debiased estimator achieves asymptotic normality, enabling valid confidence intervals and hypothesis testing. Furthermore, we developed a differentially private multiplier bootstrap procedure for simultaneous inference in high dimensions. Extensive simulations confirm that our methods achieve robust estimation and inference with moderate privacy budgets, while excessive privacy protection can impede statistical accuracy, highlighting the fundamental privacy-accuracy trade-off.

Future research directions include extending the proposed framework to other models, such as expected shortfall regression and functional regression, and to more complex setups, such as semi-supervised settings and missing data situations.  {Another important direction is to develop joint inference theory under differential privacy. 
While this paper considers coordinate-wise inference, finite-dimensional linear functionals, and simultaneous inference via private multiplier bootstrap, a full joint asymptotic normality theory for debiased estimators over growing index sets remains challenging. In addition, extending the full theoretical guarantees to non-i.i.d. and non-evenly distributed data is an important future direction, as it requires new techniques to handle client heterogeneity, weighted aggregation, and privacy-noise propagation simultaneously.} Finally, advanced privacy and optimization mechanisms, such as local or Gaussian differential privacy and DP-ADMM, may further improve the trade-off between privacy and statistical efficiency in high-dimensional problems \cite{wang2020sparse,asi2022optimal,li2023robustness,Huang2020DPadmm}.

\bibliographystyle{apalike}
\bibliography{ref}

\end{document}